\documentclass[10pt]{article}
\usepackage[preprint]{tmlr}

\usepackage{microtype}
\usepackage{graphicx}
\usepackage{subcaption}
\usepackage{booktabs}

\usepackage{hyperref}
\hypersetup{
    colorlinks=true,
    citecolor=blue,
    linkcolor=blue,
    urlcolor=blue
}

\usepackage{url}
\usepackage{amsmath,amssymb,mathtools,amsthm}
\usepackage[capitalize,noabbrev]{cleveref}
\usepackage{threeparttable}
\usepackage{algorithm}
\usepackage{algorithmic}
\usepackage{multirow}
\usepackage{wrapfig}
\usepackage[table]{xcolor}
\usepackage{todonotes}

\theoremstyle{plain}
\newtheorem{theorem}{Theorem}[section]

\newtheorem{lemma}[theorem]{Lemma}

\theoremstyle{definition}

\newtheorem{assumption}[theorem]{Assumption}
\theoremstyle{remark}

\title{TabRSI: Recursive Self-Improvement of Tabular Language\\ Models via Stratified Group Contrast Optimization}
\author{%
\textbf{Yunbo Long}$^{1,}$\thanks{Equal contribution.}\quad
\textbf{Tejumade Afonja}$^{2,}$\footnotemark[1]\quad
\textbf{Guangya Hao}$^{1}$\quad
\textbf{Alexandra Brintrup}$^{1}$\quad
\textbf{Mario Fritz}$^{2}$\\[0.5ex]
\normalfont $^{1}$Department of Engineering, University of Cambridge\\
\normalfont $^{2}$CISPA Helmholtz Center for Information Security
}

\begin{document}
\maketitle
\let\AND\undefined 

\begin{abstract}
Recursive self-improvement enables language models to
refine their behavior by learning from feedback on their
own generated outputs. Extending this paradigm to tabular
synthesis is challenging, as individual synthetic records
lack unique correct answers and generation quality depends
on the statistical properties of the synthetic population.
We introduce \textbf{TabRSI}, a recursive post-training
framework in which a tabular language model and its
statistical feedback model adapt together over successive
rounds.
At each round, the current generator produces two
independent synthetic pools: one is used to refit a
real-versus-synthetic discriminator, while the other is
ranked by the refreshed scorer to guide policy updates.
We further propose \textbf{Stratified Group Contrast
Optimization (SGCO)}, which partitions generated records
into high- and low-reward strata and optimizes the contrast
between their mean policy/reference log-likelihood ratios.
Applying nonlinear weighting after group aggregation
yields pairing-invariant updates with a shared group-level
gradient weight, enabling effective learning from
statistical quality rankings without requiring
hand-labeled preference pairs.
Across five benchmark
datasets, TabRSI with SGCO boosts column-density similarity from 86.86 to 95.47 and $\alpha$-precision from
83.40 to 96.73 over standard supervised training, outperforming extended fine-tuning and
four adapted reinforcement-learning and preference-alignment baselines. Five-round trajectories,
scorer-refresh comparisons, and sample-separation
controls further demonstrate the effectiveness of
recursive statistical feedback for tabular generation.
\end{abstract}

\section{Introduction}
\label{sec:introduction}

\begin{figure*}[t]
    \centering
    \includegraphics[width=\linewidth]{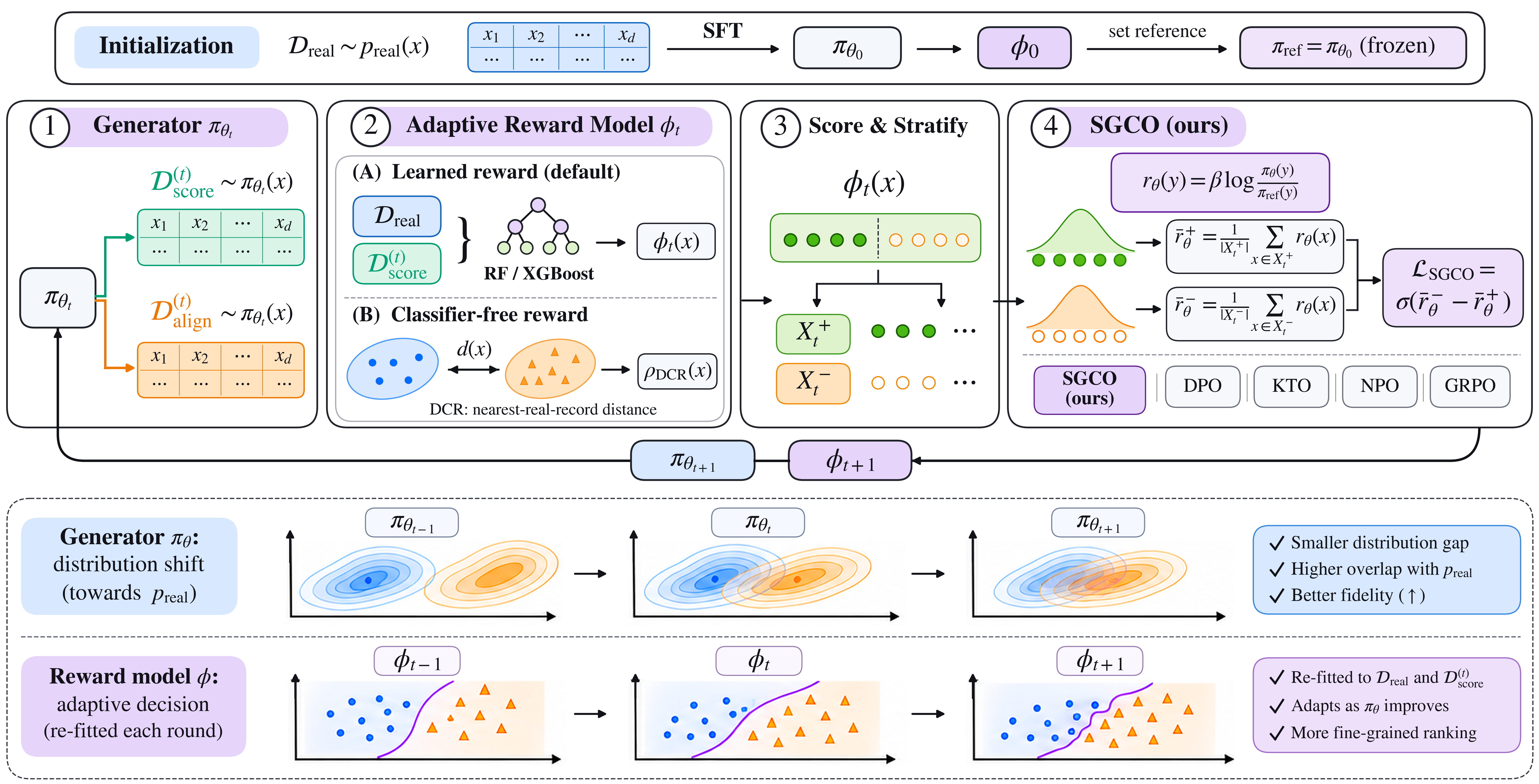}
    \caption{\textbf{TabRSI: recursive improvement of a tabular generator and its statistical feedback.}
    At each round, the generator produces two independent synthetic populations for scorer training and policy alignment. The scorer is refitted using the current generated data and ranks the alignment samples for SGCO. The updated generator then produces new data for the next round. This process allows the generator and its feedback to adapt together. The framework also supports classifier-free rewards.}
    \label{fig:workflow}
\end{figure*}

Recursive self-improvement (RSI) allows language
models to learn from feedback on their own outputs.
Recent studies have used self-play and self-rewarding
methods to improve models through repeated training
\citep{chen2024spin,yuan2024selfrewarding}.
Other approaches, such as Absolute Zero and R-Zero,
extend this idea to reasoning tasks generated by
the models themselves
\citep{zhao2025absolutezero,huang2026rzero}.
Dream-RSI also studies iterative improvement in
agentic code discovery \citep{zheng2026dreamrsi}.
These methods often rely on feedback that can be
checked. For example, mathematical answers can be
verified, and generated programs can be tested by
running them. This makes it possible to identify
better outputs and use them for further training.
However, many generation tasks do not have a single
correct answer. An important question is whether
RSI can also improve models that generate samples
from a distribution.

Tabular data generation is a useful setting for
studying this problem. Tabular records contain
numerical and categorical features, and the
relationships between these features affect whether
a record is realistic. A synthetic dataset must also
preserve the statistical properties of real data,
including feature distributions, correlations,
coverage, and information needed for downstream
prediction. Previous studies have developed GANs,
diffusion models, and language models for this task
\citep{xu2019modeling,kotelnikov2023tabddpm,
zhang2023mixed,shi2024tabdiff,borisov2022language}.
However, a synthetic record does not usually have
a unique correct answer. A record may look
reasonable on its own, while the complete synthetic
dataset may still have incorrect feature
distributions or relationships. For this reason,
tabular self-improvement needs feedback based on
statistical agreement with real data. The feedback
must also remain useful as the generated
distribution changes during training.

Autoregressive tabular language models represent
records as token sequences and learn to generate
them through maximum-likelihood training
\citep{borisov2022language}. These models can
produce realistic synthetic records after
supervised fine-tuning (SFT). However, they are
usually used as fixed generators after training.
Their newly generated records are not used to
improve the model further, even when the records
show differences from the real data distribution.
Additional SFT does not necessarily improve the
statistical properties used to evaluate synthetic
data. A possible solution is to use feedback from
the generated records to guide further training.
Our experiments show that this approach leads to
larger fidelity improvements than simply extending
SFT.

There are two main difficulties when learning from
this feedback. The first is keeping the feedback
useful as the generator changes. A classifier can
learn to distinguish real records from synthetic
ones and identify differences between their
distributions. However, a classifier trained on
earlier synthetic data may no longer provide useful
feedback after the generator has changed. Updating
only the generator may therefore limit further
improvement. The second difficulty is using
statistical feedback to update the language model.
Unlike mathematical or coding tasks, tabular records
do not have natural chosen--rejected pairs. DPO
uses a nonlinear loss for each preference pair
\citep{rafailov2024directpreferenceoptimizationlanguage},
while GRPO normalizes rewards within sampled groups
\citep{shao2024deepseekmath}. For tabular generation,
we can instead compare groups of records with high
and low statistical rewards, without relying on
individual sample pairings.

We introduce \textbf{TabRSI}
(\textbf{Tab}ular \textbf{R}ecursive
\textbf{S}elf-\textbf{I}mprovement), a recursive
post-training framework in which a tabular language
model and its statistical feedback model adapt
together (Figure~\ref{fig:workflow}). At each round,
the current generator produces two independent
synthetic populations. The first population is used
with real training data to fit a discriminator that
distinguishes real and synthetic records. The
updated discriminator then scores the second
population and provides feedback for updating the
generator. After the update, the generator produces
new synthetic records for the next round. The
discriminator is then fitted again using the new
generated distribution. In this way, both the
generator and the feedback model change over
successive rounds. Using separate populations also
ensures that the scorer evaluates alignment records
that were not used to train it. The original SFT
model remains a fixed reference during policy
updates. TabRSI can also use statistical rewards
that do not require a classifier.

To update the generator, we propose
\textbf{Stratified Group Contrast Optimization
(SGCO)}. SGCO divides generated records into
high- and low-reward groups based on their
statistical quality scores. It then compares the
mean policy/reference log-likelihood ratios of
the two groups. SGCO first calculates the group
averages and then applies a nonlinear loss.
This produces a shared group-level gradient
weight and makes the objective independent
of how individual records from the two groups
are paired. The group selection also remains
unchanged under strictly increasing
transformations of the reward values.
SGCO can therefore learn from statistical
reward rankings without requiring hand-labeled
preference pairs or gradients through the
reward model. 
Section~\ref{sec:sgco} explains the objective
and its relationship to other alignment methods.

We evaluate TabRSI on five mixed-type tabular
datasets. With SGCO, the average column-density
similarity (CDE) increases from $86.86$ for
supervised GReaT to $95.47$, while
$\alpha$-precision increases from $83.40$ to
$96.73$. Extending SFT alone reaches $87.42$
and $86.02$, respectively. Average
$\beta$-recall also improves from $47.10$ to
$50.86$. During all five rounds of recursive
training, average CDE and $\alpha$-precision
remain above the original GReaT baseline.
We also compare SGCO with other alignment
methods and test different reward functions.
Experiments with fixed and refreshed scorers
show why the feedback needs to adapt to the
current generator. Further experiments with
random rewards and sample leakage show the
importance of meaningful statistical feedback
and independent scorer-training data.
Our main contributions are:
\begin{itemize}

    \item \textbf{Dual recursive adaptation for
tabular language models.}
We introduce a post-training framework in which
the generator and its statistical feedback model
adapt together over successive rounds. Two
independent synthetic populations are used for
scorer training and policy alignment, allowing
the feedback to evolve with the generator.

    \item \textbf{Stratified group contrast
    optimization.}
    We propose a group-level policy objective
    that compares the mean policy/reference
    log-likelihood ratios of high- and low-reward
    samples. By applying nonlinear weighting
    after group aggregation, SGCO removes the
    need for arbitrary sample pairing and can
    learn directly from statistical reward rankings.

    \item \textbf{Experimental evaluation of
    recursive tabular self-improvement.}
    We evaluate TabRSI on five mixed-type
    datasets using fidelity, utility, and
    empirical attack metrics. Comparisons with
    existing alignment methods, multi-round
    experiments, reward ablations, scorer-refresh
    tests, and sample-separation controls examine
    the effectiveness of the proposed framework.

\end{itemize}

\section{Related work}
\label{sec:related_work}

\paragraph{Recursive self-improvement and verifiable feedback.}
Recursive self-improvement allows language models to learn from their own generated outputs. SPIN compares model-generated responses with target-domain data \citep{chen2024spin}, while Self-Rewarding Language Models use model-generated judgments for repeated preference optimization \citep{yuan2024selfrewarding}. Absolute Zero and R-Zero extend self-improvement to reasoning tasks where generated solutions can be checked \citep{zhao2025absolutezero,huang2026rzero}. Dream-RSI further studies recursive improvement through agentic discovery \citep{zheng2026dreamrsi}. These studies show how models can improve through repeated feedback. However, tabular generation presents a different problem. A synthetic record does not have a unique correct answer, and its quality depends partly on the distribution of other generated records. TabRSI studies recursive self-improvement in this setting, where the generator changes over time and the statistical feedback model is repeatedly fitted to its current outputs.

\paragraph{Tabular generative modeling.}
Several methods have been developed to generate synthetic tabular data. CTGAN and TVAE use adversarial and variational learning, while TabDDPM, TabSyn, and TabDiff use diffusion-based generation \citep{xu2019modeling,kotelnikov2023tabddpm,zhang2023mixed,shi2024tabdiff}. GReaT instead represents tabular records as text and trains an autoregressive language model to generate new rows \citep{borisov2022language}. These methods learn to reproduce statistical properties of real data, but their training objectives differ. GANs already use feedback between a generator and discriminator, whereas tabular language models are commonly trained through supervised sequence prediction. TabRSI builds on a supervised tabular language model and uses its newly generated records for further policy updates. The feedback is recomputed as the generated distribution changes, rather than relying only on the original supervised training objective. We compare different language-model post-training methods and also include GAN, VAE, and diffusion models as tabular generation baselines.

\paragraph{Feedback-guided tabular language models.}
Previous studies have explored how feedback can improve tabular language models. P-TA uses a discriminator to distinguish real and synthetic records and applies PPO to improve the language-model generator \citep{yang2024pta}. This work shows that discriminator feedback can guide tabular generation despite the discrete sampling process. TabRSI studies a different feedback and optimization procedure. In each round, a discriminator is refitted on samples from the current generator and then scores a separate synthetic population for policy alignment. The updated generator produces new data for the next round. We examine this recursive process through scorer-refresh and sample-separation experiments.
ReTabSyn uses a different form of feedback. It constructs preferred and rejected examples by changing target labels or correlated features, and then applies DPO-Positive to improve feature--target relationships \citep{lin2026retabsyn}. This is useful for prediction tasks with limited or imbalanced data. Its preference pairs are constructed from feature and target perturbations before policy fine-tuning. TabRSI instead generates new records at each round and uses statistical rewards to rank them. Its objective does not require target-specific perturbations or fixed preferred--rejected pairs. The studies above provide useful approaches to tabular LM alignment, but they do not jointly examine recursive statistical feedback, independent scorer-training and alignment pools, and pairing-invariant group optimization.

\paragraph{Preference optimization and adaptive feedback.}
Preference optimization provides several ways to update language models using feedback. DPO learns from preferred--rejected pairs \citep{rafailov2024directpreferenceoptimizationlanguage}. KTO uses desirable and undesirable outcomes, while NPO focuses on negative samples \citep{ethayarajh2024ktomodelalignmentprospect,zhang2024negative}. GRPO uses rewards normalized within sampled groups \citep{shao2024deepseekmath}. Our SGCO method follows a different approach. It first divides generated records into high- and low-reward groups, then compares their mean policy/reference log-likelihood ratios. The nonlinear loss is applied after group aggregation. This makes the objective independent of arbitrary pairing between individual records and allows it to learn from statistical reward rankings.
The quality of the feedback is also important. Reward-model overoptimization can cause a model to improve its reward score without improving the intended task \citep{gao2023rewardoveroptimization,karwowski2024goodhart}. For tabular generation, this means that higher rewards do not always imply better statistical fidelity or downstream utility. TabRSI uses feedback from the current generated distribution and evaluates the resulting data with independent fidelity, utility, and empirical attack measures. We also compare refreshed and fixed scorers and test the effect of reusing scorer-training samples during policy alignment.

\section{TabRSI: iterative reward-guided self-improvement}
\label{sec:method}

TabRSI recursively refines a supervised tabular LM by alternating generator alignment and reward-model refitting. At each round, the current generator determines the synthetic distribution against which the scorer is trained; the refreshed scorer then ranks an independent pool of generated records for SGCO. Real-data supervision and a fixed reference policy anchor the updates.

\subsection{From supervised synthesis to iterative refinement}
\label{sec:tabrsi-loop}
Given real records $\mathcal D_{\mathrm{real}}=\{x_i\}_{i=1}^{N}$, let $y=\operatorname{serialize}(x)$ represent a row as a token sequence. To generate a synthetic record, we sample $y\sim\pi_\theta$ and deserialize it as $\tilde x=\operatorname{deserialize}(y)$. We initialize an autoregressive LM $\pi_{\mathrm{base}}$ through maximum-likelihood supervised fine-tuning:
\begin{equation}
\mathcal L_{\mathrm{SFT}}(\theta)=
-\mathbb E_{x\sim\mathcal D_{\mathrm{real}}}
\left[\sum_{k=1}^{|y|}\log\pi_\theta(y_k\mid y_{<k})\right].
\label{eq:tabrsi-sft}
\end{equation}
The resulting $\pi_{\theta_0}$ serves both as the initial generator and as the fixed reference $\pi_{\mathrm{ref}}=\pi_{\theta_0}$ for subsequent updates. Because the reference remains fixed, all rounds share the same policy-likelihood baseline even as the generated distribution changes. At round $t\geq1$, the current policy generates two independent synthetic pools, conditional on $\pi_{\theta_{t-1}}$. A real-versus-synthetic scorer is refitted on one pool; the other is scored, stratified, and used for SGCO alignment:
\begin{equation}
\begin{aligned}
\widetilde{\mathcal D}^{(t)}_{\mathrm{score}},\;\widetilde{\mathcal D}^{(t)}_{\mathrm{align}}
&\stackrel{\mathrm{ind}}{\sim}\pi_{\theta_{t-1}},\\
\phi_t&\leftarrow\operatorname{FitDisc}\!\left(\mathcal D_{\mathrm{real}},\widetilde{\mathcal D}^{(t)}_{\mathrm{score}}\right),\\
\rho_t(\tilde x)&\leftarrow\rho_{\mathrm{cls},t}(\tilde x),\\
(\mathcal X_t^+,\mathcal X_t^-)&\leftarrow\operatorname{Stratify}\!\left(\widetilde{\mathcal D}^{(t)}_{\mathrm{align}},\rho_t\right),\\
\theta_t&\leftarrow\operatorname{Align}_{\mathrm{SGCO}}\!\left(\theta_{t-1};\mathcal X_t^+,\mathcal X_t^-,\pi_{\mathrm{ref}}\right).
\end{aligned}
\label{eq:self_improvement_round}
\end{equation}
Here $\mathcal X_t^+$ and $\mathcal X_t^-$ are the high- and low-reward strata of the alignment pool. The updated policy $\pi_{\theta_t}$ generates the next round's candidates, yielding an alternating generator--scorer feedback loop. The two sampled pools remain distinct throughout scorer fitting and policy alignment. For classifier-free rewards, the $\operatorname{FitDisc}$ step is omitted. Algorithm~\ref{alg:iterative_alignment_framework} gives the complete protocol.

\subsection{Adaptive reward feedback on generated records}
\label{sec:quality_signal}
The external reward $\rho_t:\mathcal X\to\mathbb R$ ranks generated records without requiring differentiation through the scorer. By default, we train a Random Forest discriminator $\phi_t:\mathcal X\to[0,1]$ on real data and the current scorer-training pool. Its output $\phi_t(x)=\widehat P_t(c=1\mid x)$ estimates the real-class probability, where $c=1$ identifies a real record and $c=0$ a synthetic record. The disjoint alignment pool is ranked using discriminator uncertainty:
\begin{equation}
\rho_{\mathrm{cls},t}(\tilde x)
=1-2\bigl|\tfrac12-\phi_t(\tilde x)\bigr|.
\label{eq:tabrsi-classifier-reward}
\end{equation}
The reward lies in $[0,1]$ and peaks at $\phi_t(\tilde x)=1/2$, favoring records whose origin is hardest for the current discriminator to determine. This reward targets classifier ambiguity rather than a high predicted real-class probability: even records confidently identified as real receive low uncertainty reward. Refitting $\phi_t$ after each generator update adapts the ranking to the evolving synthetic distribution; keeping scorer-training and alignment records separate provides out-of-sample feedback. We also study classifier-free distance-to-closest-record (DCR) rewards on appropriately encoded records:
\begin{equation}
\begin{aligned}
d(\tilde x)&=\min_{x\in\mathcal D_{\mathrm{real}}}\|\tilde x-x\|_2,\\
\rho_{\mathrm{dcr},t}(\tilde x)&=
1-\frac{d(\tilde x)-d_{\min}}{d_{\max}-d_{\min}}.
\end{aligned}
\label{eq:tabrsi-dcr-reward}
\end{equation}
Here $d_{\min}$ and $d_{\max}$ are the minimum and maximum of $d(\tilde x)$ over the current alignment pool, with $d_{\max}>d_{\min}$; equal distances yield no DCR ordering. DCR favors proximity to the real-data support without learning a separate classifier. Additional reward forms are described in Appendix~\ref{app:reward_taxonomy}.

\subsection{SGCO: stratified group contrast optimization}
\label{sec:sgco}
Statistical row-level rewards provide rankings rather than natural human preference pairs. SGCO therefore ranks the full alignment pool into $\mathcal X_t^+$ and $\mathcal X_t^-$ and samples $B$ serialized records from each to form $\mathcal B_t^+$ and $\mathcal B_t^-$, with $|\mathcal B_t^+|=|\mathcal B_t^-|=B$. The two strata have no natural one-to-one matching; each instead describes a ranked part of the generated population. The external reward $\rho_t(x)$ determines group membership but is not differentiated through. Instead, policy updates use the implicit policy/reference log-ratio of each serialized row:
\begin{equation}
r_\theta(y)=\beta\log\frac{\pi_\theta(y)}{\pi_{\mathrm{ref}}(y)},
\qquad \beta>0.
\label{eq:tabrsi-implicit-reward}
\end{equation}
The fixed reference allows the two strata to be compared through their mean log-ratios:
\begin{equation}
\bar r_\theta^+=\frac1B\sum_{y\in\mathcal B_t^+}r_\theta(y),
\qquad
\bar r_\theta^-=\frac1B\sum_{y\in\mathcal B_t^-}r_\theta(y),
\quad
\Delta_{\mathrm{group}}(\theta)=\bar r_\theta^+-\bar r_\theta^-.
\label{eq:tabrsi-group-delta}
\end{equation}
Let $\sigma(u)=(1+e^{-u})^{-1}$. A group-level Bradley--Terry model assigns preference probability $\sigma(\Delta_{\mathrm{group}})$ to the high-reward group. We consider two monotone surrogate objectives that apply their nonlinearity \emph{after} calculating the group contrast:
\begin{equation}
\begin{aligned}
\mathcal L_{\mathrm{SGCO}}^{\mathrm{bounded}}(\theta)
&=1-\sigma\!\left(\Delta_{\mathrm{group}}(\theta)\right)
=\sigma\!\left(-\Delta_{\mathrm{group}}(\theta)\right),\\
\mathcal L_{\mathrm{SGCO}}^{\mathrm{log}}(\theta)
&=-\log\sigma\!\left(\Delta_{\mathrm{group}}(\theta)\right).
\end{aligned}
\label{eq:sgco_loss}
\end{equation}
The main experiments use the bounded-sigmoid objective; the log-sigmoid objective is evaluated in the loss-variant ablation (Table~\ref{tbl:ablation_loss_variants}). Both objectives encourage higher policy-relative likelihood for the high-reward stratum, but their gradient coefficients differ:
\begin{equation}
\begin{aligned}
\nabla_\theta\mathcal L_{\mathrm{SGCO}}^{\mathrm{bounded}}
&=-\sigma(-\Delta_{\mathrm{group}})\bigl[1-\sigma(-\Delta_{\mathrm{group}})\bigr]\nabla_\theta\Delta_{\mathrm{group}},\\
\nabla_\theta\mathcal L_{\mathrm{SGCO}}^{\mathrm{log}}
&=-\sigma(-\Delta_{\mathrm{group}})\nabla_\theta\Delta_{\mathrm{group}}.
\end{aligned}
\label{eq:tabrsi-group-gradient}
\end{equation}
In either case, a single scalar multiplies the difference between the two group-average gradients. The bounded variant reduces the gradient magnitude for very large positive or negative group margins, whereas the log-sigmoid variant retains a larger coefficient for strongly negative margins. Appendix~\ref{app:sgco_bt_derivation} derives the group preference model, and the loss ablation evaluates both surrogate forms.

\subsection{Why group aggregation changes preference optimization}
\label{sec:tabrsi-math-comparison}
SGCO differs from pairwise DPO in the order of aggregation and nonlinear weighting. For any pairing of $B$ high- and low-reward sequences, define $z_i=r_\theta(y_i^+)-r_\theta(y_i^-)$ and $\bar z=B^{-1}\sum_i z_i$. Pairwise DPO applies the log-sigmoid loss to each margin before averaging, whereas both SGCO variants first compute the mean contrast:
\begin{equation}
\begin{aligned}
\mathcal L_{\mathrm{DPO}}&=-\frac1B\sum_{i=1}^B\log\sigma(z_i),\\
\mathcal L_{\mathrm{SGCO}}^{\mathrm{bounded}}&=\sigma(-\bar z),\qquad
\mathcal L_{\mathrm{SGCO}}^{\mathrm{log}}=-\log\sigma(\bar z).
\end{aligned}
\label{eq:sgco_dpo_contrast}
\end{equation}
This changes the gradients. DPO uses a pair-specific coefficient, while both SGCO variants apply a shared coefficient to the group-average margin gradient:
\begin{equation}
\begin{aligned}
\nabla_\theta\mathcal L_{\mathrm{DPO}}
&=-\frac1B\sum_{i=1}^B\sigma(-z_i)\nabla_\theta z_i,\\
\nabla_\theta\mathcal L_{\mathrm{SGCO}}^{v}
&=-w_v(\bar z)\,\frac1B\sum_{i=1}^B\nabla_\theta z_i,\\
w_{\mathrm{bounded}}(\bar z)&=\sigma(-\bar z)\bigl[1-\sigma(-\bar z)\bigr],\\
w_{\mathrm{log}}(\bar z)&=\sigma(-\bar z),\qquad
v\in\{\mathrm{bounded},\mathrm{log}\}.
\end{aligned}
\label{eq:sgco_gradient_comparison}
\end{equation}
In particular, DPO and log-sigmoid SGCO share the same scalar nonlinearity, but apply it to different inputs ($z_i$ versus $\bar z$). SGCO uses the group mean instead of assigning separate nonlinear weights to individual chosen--rejected pairs. At the SFT initialization, all margins are zero, giving group and pairwise gradients in the same direction up to scaling. Once individual margins diverge, the aggregation order matters. The group contrast also has a simple pairing-invariance property:
For any permutation $\tau$ of the low-reward group,
\begin{equation}
\frac1B\sum_{i=1}^B\bigl[r_\theta(y_i^+)-r_\theta(y_{\tau(i)}^-)\bigr]
=\bar r_\theta^+-\bar r_\theta^-.
\label{eq:tabrsi-pairing-invariance}
\end{equation}
so SGCO's objective and gradient do not depend on the pairing, unlike pairwise DPO. Similarly, its strata depend on score order alone: for a fixed pool, any strictly increasing transform $h$ preserves group assignments (up to ties) and hence
\begin{equation}
\mathcal L_{\mathrm{SGCO}}(\theta;\rho_t)
=\mathcal L_{\mathrm{SGCO}}(\theta;h\circ\rho_t)
\quad\text{for fixed sampled strata}.
\label{eq:tabrsi-rank-invariance}
\end{equation}
This rank invariance is useful when reward scales differ across rounds. GRPO~\citep{shao2024deepseekmath}, in contrast, normalizes rewards within each sampled response group $\mathcal H(i)$:
\begin{equation}
\widehat A_i
=\frac{\rho_t(x_i)-\operatorname{mean}_{j\in\mathcal H(i)}\rho_t(x_j)}
{\operatorname{std}_{j\in\mathcal H(i)}\rho_t(x_j)+\epsilon},
\label{eq:tabrsi-grpo-advantage}
\end{equation}
where $x_i=\operatorname{deserialize}(y_i)$ and $\epsilon>0$ stabilizes normalization. SGCO forms global high/low strata instead of computing local group advantages. Section~\ref{sec:grpo-comparison} reports the matched comparison.

\subsection{Feedback separation and fixed-round stability}
\label{sec:tabrsi-stability}
Each SGCO update changes the distribution used to fit the next reward scorer, while independent scorer-training and alignment pools prevent in-sample reward feedback. These choices are examined through fixed-versus-refreshed scorer and leakage controls (Section~\ref{sec:alternative_rewards}; Appendix~\ref{app:classifier_variants}). To characterize optimization within a round, fix the scored strata and define $\mathcal J(\theta)=\mathbb E[\mathcal L_{\mathrm{SGCO}}(\theta)]$. Under the smoothness, bounded-gradient, and unbiased stochastic-gradient assumptions in Appendix~\ref{app:theoretical_proofs}, if $\mathbb E\|\hat g_k-\nabla\mathcal J\|^2\leq\sigma_{\mathcal L}^2$ and $\eta\leq1/L$, then
\begin{equation}
\min_{0\leq k<K}\mathbb E\|\nabla\mathcal J(\theta^{(k)})\|^2
\leq\frac{2\bigl[\mathcal J(\theta^{(0)})-\mathcal J^*\bigr]}{\eta K}
+\eta L\sigma_{\mathcal L}^2,
\quad\mathcal J^*=\inf_\theta\mathcal J(\theta).
\label{eq:main_sgco_sgd_bound}
\end{equation}
The gradient of the group-average log-ratio has variance $\mathcal O(1/B)$ under independent within-stratum sampling. This bound describes the SGCO update within one round. We evaluate the complete recursive training process through the five-round trajectories and scorer ablations.

\section{Experimental settings}
\label{sec:setting}

\paragraph{Datasets and baselines.}
We evaluate on five mixed-type tabular benchmarks: Adult~\citep{adultdataset}, Default~\citep{defaultdataset}, Shoppers~\citep{shoppersdataset}, Magic~\citep{magicdataset}, and Beijing~\citep{beijingpm25dataset}. We use GReaT~\citep{borisov2022language} as the generator backbone and compare SGCO with extended supervised fine-tuning (GReaT-FT+) and tabular adaptations of DPO~\citep{rafailov2024directpreferenceoptimizationlanguage}, KTO~\citep{ethayarajh2024ktomodelalignmentprospect}, and NPO~\citep{zhang2024negative}. A separate matched experiment compares SGCO with GRPO~\citep{shao2024deepseekmath}. We report comparisons with GAN-, VAE-, and diffusion-based tabular generators in Appendix~\ref{app:sota}. Dataset characteristics and baseline definitions are detailed in Appendices~\ref{app:dataset} and~\ref{app:baseline}.

\paragraph{Post-training settings.}
All primary post-training methods start from the same supervised GReaT checkpoint and follow a common training protocol over five rounds. At each round, the current generator independently samples two synthetic pools for fitting a real-versus-synthetic Random Forest scorer and for policy alignment. SGCO contrasts high- and low-reward groups, while the adapted baselines use their respective preference or negative-sample objectives. After each round, each method generates new samples from its updated policy. We use $B=4$ as the default SGCO group size and study $B\in\{1,4,8,16,32,64\}$ in a separate sensitivity experiment. We also evaluate classifier-free distance-to-closest-record (DCR) feedback. Further details on the baseline objectives, training procedure, hyperparameters, and GRPO comparison are provided in Appendices~\ref{app:alignment_baselines}, \ref{app:algorithm_implementation}, \ref{app:hyperparams}, and~\ref{app:grpo-matched}.

\paragraph{Evaluation.}
We assess fidelity using column-density similarity (CDE$\uparrow$), pairwise-correlation similarity (PCC$\uparrow$), Wasserstein distance, MMD, and JSD ($\downarrow$), as well as $\alpha$-precision and $\beta$-recall ($\uparrow$)~\citep{alaa2022faithful} and the C2ST quality score ($\uparrow$)~\citep{lopezpaz2017c2st}. Downstream utility (MLE) is evaluated on held-out real data (classification AUC$\uparrow$; Beijing regression RMSE$\downarrow$). Distinguishability-attack (DA) and membership-inference-attack (MIA) AUCs are reported as empirical diagnostics (closer to $0.5$ indicates weaker attacks). Per-dataset results summarize ten generated-output/evaluation repeats, while aggregate standard deviations are calculated across dataset-level means. Metric definitions, statistical reporting, and the separately matched GRPO protocol appear in Appendices~\ref{app:evaluation} and~\ref{app:grpo-matched}.

\section{Experimental results}
\label{sec:results}

\begin{table*}[t!]
\centering

\caption{Comparison of SGCO with supervised and alignment baselines after five training rounds. We compare SGCO with GReaT, extended supervised fine-tuning, and adapted alignment methods, including loss variants. Post-training results are reported after five rounds. Values are mean $\pm$ population standard deviation across five dataset-level means. For MLE, Beijing is excluded because it is a regression dataset. Detailed per-dataset results are provided in Appendix~F.1. Best results are in \textbf{bold}.}
\label{tbl:ablation_loss_variants}
\begin{threeparttable}
\resizebox{\textwidth}{!}{
\begin{tabular}{lccccccc}
\toprule[1pt]
\textbf{Method} & \textbf{CDE$\uparrow$} & \textbf{PCC$\uparrow$} &
\textbf{$\alpha$$\uparrow$} & \textbf{$\beta$$\uparrow$} &
\textbf{C2ST$\uparrow$} & \textbf{DA $\to 0.5$} & \textbf{MLE$\uparrow$} \\
\midrule
\multicolumn{8}{l}{\textit{Baselines}} \\
GReaT
& $86.86{\scriptstyle\pm4.78}$
& $60.76{\scriptstyle\pm24.06}$
& $83.40{\scriptstyle\pm9.28}$
& $47.10{\scriptstyle\pm6.70}$
& $32.75{\scriptstyle\pm19.12}$
& $0.8174{\scriptstyle\pm0.0533}$
& $0.8686{\scriptstyle\pm0.0627}$ \\
GReaT-FT+
& $87.42{\scriptstyle\pm4.88}$
& $60.55{\scriptstyle\pm24.01}$
& $86.02{\scriptstyle\pm8.91}$
& $46.32{\scriptstyle\pm5.88}$
& $33.97{\scriptstyle\pm20.00}$
& $0.8151{\scriptstyle\pm0.0531}$
& $0.8647{\scriptstyle\pm0.0599}$ \\
\midrule

\multicolumn{8}{l}{\textit{DPO variants}} \\
TabDPO (base)
& $94.94{\scriptstyle\pm3.65}$
& $60.34{\scriptstyle\pm24.45}$
& $96.01{\scriptstyle\pm2.03}$
& $50.07{\scriptstyle\pm4.53}$
& $48.20{\scriptstyle\pm34.97}$
& $0.6930{\scriptstyle\pm0.0932}$
& $0.8780{\scriptstyle\pm0.0587}$ \\
+ KL penalty
& $94.69{\scriptstyle\pm3.91}$
& $59.90{\scriptstyle\pm24.41}$
& $96.30{\scriptstyle\pm2.16}$
& $49.59{\scriptstyle\pm4.67}$
& $48.05{\scriptstyle\pm35.11}$
& $0.6936{\scriptstyle\pm0.0948}$
& $0.8778{\scriptstyle\pm0.0578}$ \\
+ Gradient diff.
& $95.00{\scriptstyle\pm3.57}$
& $60.48{\scriptstyle\pm25.29}$
& $96.55{\scriptstyle\pm1.51}$
& $49.85{\scriptstyle\pm4.28}$
& $48.12{\scriptstyle\pm34.78}$
& $0.6912{\scriptstyle\pm0.0891}$
& $\mathbf{0.8795{\scriptstyle\pm0.0568}}$ \\
\midrule

\multicolumn{8}{l}{\textit{NPO variants}} \\
TabNPO (base)
& $88.71{\scriptstyle\pm4.27}$
& $59.38{\scriptstyle\pm23.16}$
& $85.80{\scriptstyle\pm8.54}$
& $49.32{\scriptstyle\pm5.84}$
& $37.84{\scriptstyle\pm23.85}$
& $0.7676{\scriptstyle\pm0.0602}$
& $0.8775{\scriptstyle\pm0.0576}$ \\
+ KL penalty
& $88.78{\scriptstyle\pm4.05}$
& $56.70{\scriptstyle\pm21.22}$
& $85.40{\scriptstyle\pm8.17}$
& $49.17{\scriptstyle\pm5.71}$
& $39.44{\scriptstyle\pm25.59}$
& $0.7651{\scriptstyle\pm0.0616}$
& $0.8754{\scriptstyle\pm0.0590}$ \\
+ Gradient diff.
& $89.06{\scriptstyle\pm3.93}$
& $57.29{\scriptstyle\pm21.61}$
& $85.22{\scriptstyle\pm8.34}$
& $49.52{\scriptstyle\pm5.95}$
& $38.59{\scriptstyle\pm24.31}$
& $0.7619{\scriptstyle\pm0.0612}$
& $0.8751{\scriptstyle\pm0.0577}$ \\
\midrule

\multicolumn{8}{l}{\textit{KTO variants}} \\
TabKTO (base)
& $87.89{\scriptstyle\pm4.31}$
& $58.92{\scriptstyle\pm23.24}$
& $85.14{\scriptstyle\pm8.59}$
& $48.98{\scriptstyle\pm6.20}$
& $35.81{\scriptstyle\pm22.17}$
& $0.7748{\scriptstyle\pm0.0613}$
& $0.8740{\scriptstyle\pm0.0587}$ \\
+ Logsigmoid
& $87.89{\scriptstyle\pm4.35}$
& $58.24{\scriptstyle\pm22.51}$
& $84.93{\scriptstyle\pm8.55}$
& $48.73{\scriptstyle\pm6.23}$
& $36.33{\scriptstyle\pm22.70}$
& $0.7742{\scriptstyle\pm0.0600}$
& $0.8760{\scriptstyle\pm0.0570}$ \\
+ Logs. + Grad. diff
& $87.87{\scriptstyle\pm4.38}$
& $58.47{\scriptstyle\pm22.75}$
& $84.75{\scriptstyle\pm8.66}$
& $48.82{\scriptstyle\pm6.23}$
& $36.10{\scriptstyle\pm22.48}$
& $0.7757{\scriptstyle\pm0.0590}$
& $0.8787{\scriptstyle\pm0.0548}$ \\
\midrule

\multicolumn{8}{l}{\textit{SGCO variants (Ours)}} \\
SGCO (base)
& $\mathbf{95.47{\scriptstyle\pm2.99}}$
& $60.70{\scriptstyle\pm24.24}$
& $96.73{\scriptstyle\pm2.25}$
& $\mathbf{50.86{\scriptstyle\pm4.67}}$
& $\mathbf{48.34{\scriptstyle\pm34.62}}$
& $\mathbf{0.6814{\scriptstyle\pm0.0959}}$
& $0.8783{\scriptstyle\pm0.0598}$ \\
+ Logsigmoid
& $94.86{\scriptstyle\pm4.16}$
& $\mathbf{61.70{\scriptstyle\pm25.46}}$
& $\mathbf{97.70{\scriptstyle\pm1.67}}$
& $49.72{\scriptstyle\pm4.83}$
& $47.48{\scriptstyle\pm34.88}$
& $0.6900{\scriptstyle\pm0.0957}$
& $0.8792{\scriptstyle\pm0.0590}$ \\
+ Logs. + Grad. diff
& $95.04{\scriptstyle\pm3.95}$
& $59.55{\scriptstyle\pm23.32}$
& $96.84{\scriptstyle\pm1.76}$
& $50.53{\scriptstyle\pm4.54}$
& $47.79{\scriptstyle\pm34.57}$
& $0.6871{\scriptstyle\pm0.0895}$
& $0.8790{\scriptstyle\pm0.0561}$ \\

\bottomrule[1pt]
\end{tabular}}
\end{threeparttable}
\end{table*}

\begin{table*}[t]
\centering
\caption{Five-round comparison of SGCO with GReaT and GRPO across five datasets. Values are mean $\pm$ population standard deviation across dataset-level means. Classification AUC uses the four classification datasets (Beijing excluded); all other columns use five datasets. All values are multiplied by 100 and shown with two decimal places. The best mean in each column is bold; MIA AUC is ranked by closeness to 50. Detailed dataset results are provided in Appendix~\ref{app:grpo-matched}.}
\label{tbl:grpo-main}
\begingroup
\small
\setlength{\tabcolsep}{2pt}
\renewcommand{\arraystretch}{1.22}
\resizebox{\linewidth}{!}{%
\begin{tabular}{@{}l*{8}{c}@{}}
\toprule
\textbf{Method}
& \textbf{CDE}$\uparrow$
& $\boldsymbol{\alpha}\uparrow$
& $\boldsymbol{\beta}\uparrow$
& \shortstack[c]{\textbf{Wasserstein}$\downarrow$}
& \textbf{MMD}$\downarrow$
& \textbf{JSD}$\downarrow$
& \shortstack[c]{\textbf{MIA AUC}}
& \shortstack[c]{\textbf{AUC}$\uparrow$} \\
\midrule
GReaT & $86.71\pm4.20$ & $83.40\pm9.28$ & $47.10\pm6.70$ & $22.44\pm15.69$ & $12.26\pm12.95$ & $15.08\pm11.18$ & $50.44\pm0.97$ & $86.86\pm6.27$ \\
GRPO  & $87.44\pm4.89$ & $84.96\pm8.20$ & $48.70\pm6.20$ & $22.39\pm15.72$ & $12.30\pm13.03$ & $14.68\pm11.46$ & $\boldsymbol{50.19}\pm0.94$ & $87.50\pm5.82$ \\
SGCO  & $\boldsymbol{94.41}\pm2.66$ & $\boldsymbol{96.99}\pm1.64$ & $\boldsymbol{50.64}\pm4.49$ & $\boldsymbol{21.60}\pm16.07$ & $\boldsymbol{11.99}\pm13.07$ & $\boldsymbol{14.26}\pm11.57$ & $50.28\pm0.96$ & $\boldsymbol{87.63}\pm5.77$ \\
\bottomrule
\end{tabular}%
}
\endgroup
\end{table*}

\begin{table}[t]
\centering
\caption{Effect of statistical reward choice in TabRSI with SGCO across five datasets.  Classifier uses the refreshed discriminator-uncertainty reward, DCR uses a classifier-free distance-based reward, and Leakage reuses scorer-training samples for policy alignment. GReaT is the supervised baseline. Values are mean $\pm$ population standard deviation across five dataset-level means; classification AUC excludes Beijing. MIA AUC is ranked by proximity to $0.5$. Best results among GReaT, DCR, and Classifier are in bold.}

\label{tbl:reward_comparison}
\footnotesize
\setlength{\tabcolsep}{3pt}
\resizebox{\linewidth}{!}{%
\begin{tabular}{lrrrrr}
\toprule
Reward & Wasserstein$\downarrow$ & MMD$\downarrow$ & JSD$\downarrow$ & MIA AUC$\to0.5$ & Cls. AUC$\uparrow$ \\
\midrule
GReaT & $0.2284\pm0.1521$ & $0.1226\pm0.1295$ & $0.1508\pm0.1118$ & $0.5044\pm0.0097$ & $0.8740\pm0.0545$ \\
DCR $\rho_{\mathrm{dcr}}$ & $0.2227\pm0.1572$ & $0.1220\pm0.1299$ & $0.1469\pm0.1129$ & $\boldsymbol{0.5016\pm0.0093}$ & $\boldsymbol{0.8765\pm0.0582}$ \\
Classifier $\rho_{\mathrm{cls}}$ & $\boldsymbol{0.2224\pm0.1555}$ & $\boldsymbol{0.1214\pm0.1305}$ & $\boldsymbol{0.1444\pm0.1150}$ & $0.5024\pm0.0098$ & $0.8763\pm0.0577$ \\
Leakage $\rho_{\mathrm{leak}}$ & $0.0254\pm0.0068$ & $0.0110\pm0.0073$ & $0.0138\pm0.0064$ & $0.7359\pm0.0580$ & $0.8680\pm0.0561$ \\
\bottomrule
\end{tabular}}
\end{table}

We study recursive improvement as an interaction between a changing tabular generator, the reward feedback fitted to its current outputs, and the SGCO update that turns that feedback into policy learning. We ask four questions: \emph{(i)} How much does the generator improve over successive rounds? \emph{(ii)} How does SGCO compare with DPO, GRPO, and other alignment objectives? \emph{(iii)} What is gained by reward stratification and group contrast? \emph{(iv)} How do reward-model refresh and independent feedback affect sustained fidelity and downstream quality? We use the same five mixed-type benchmarks, with the original loss-variant study and the separately matched GRPO comparison reported as distinct experiments.

\subsection{TabRSI improves a fitted tabular generator across successive rounds}
\label{sec:self_improvement_results}

Table~\ref{tbl:ablation_loss_variants} summarizes the original comparison between GReaT, an extended-SFT baseline (GReaT-FT+), and reward-guided alignment variants in the common iterative protocol. SGCO consistently improves the supervised generator's central fidelity measures: average column-density similarity increases from \textbf{86.86 to 95.47} ($+8.61$ percentage points), while $\alpha$-precision increases from \textbf{83.40 to 96.73} ($+13.33$ points). Coverage also improves, with $\beta$-recall increasing from 47.10 to 50.86; average downstream classification AUC increases from 0.8686 to 0.8783. Extended supervised fine-tuning reaches 87.42 CDE and 86.02 $\alpha$-precision. The larger gains from TabRSI show the value of using feedback from newly generated records to guide post-training.
Figure~\ref{fig:iterative} captures the key temporal behavior of TabRSI. The same generator is updated repeatedly, and its subsequent samples become the candidates for the next iteration. SGCO's improvements accumulate during the early rounds and remain comparatively stable through round five. In contrast, KTO exhibits less stable late-round behavior under the tested protocol. The dataset-specific trajectories in Appendix~\ref{app:further_results} exhibit the same broad pattern of early gains, with different fidelity--utility trade-offs across tasks. Beyond comparisons among tabular LMs, Appendix~\ref{app:sota} places these post-trained generators alongside CTGAN, TVAE, TabDDPM, TabSyn, and TabDiff. These results show how reward-guided refinement narrows several of the quality gaps between a static tabular LM and specialized synthesizers.

\begin{figure*}[t]
    \centering
    \includegraphics[width=\linewidth]{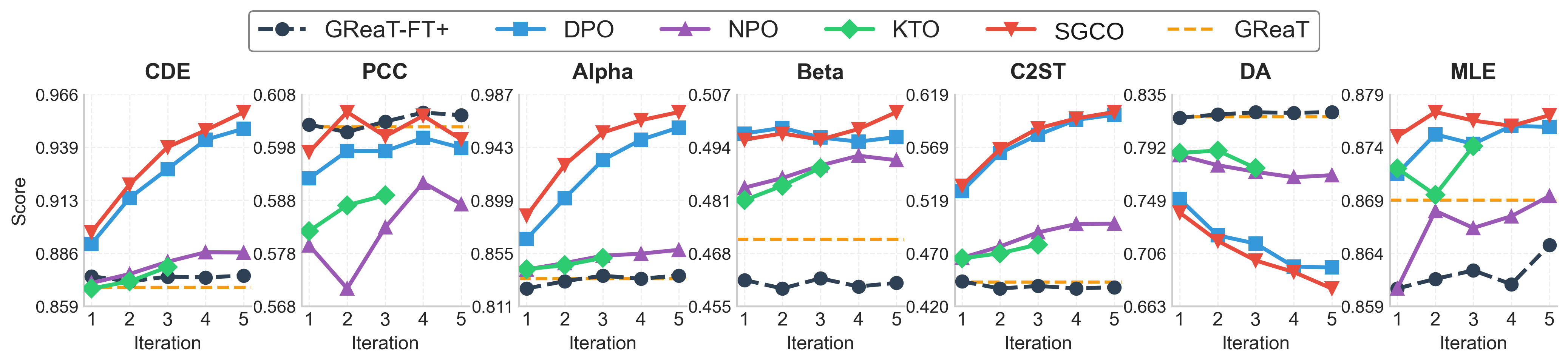}
\caption{Comparison of self-improvement trajectories between extended supervised fine-tuning and TabRSI methods over five rounds. Metrics are averaged across the five benchmarks, with Beijing excluded from classification AUC aggregation. SGCO maintains fidelity improvements over successive rounds, while KTO shows greater fluctuations later in training.}

    \label{fig:iterative}
\end{figure*}

\subsection{SGCO versus other policy-alignment objectives}
\label{sec:grpo-comparison}

\paragraph{Group-aggregated alignment versus instance-level preferences.}
The original post-training comparison isolates an important trend. DPO improves substantially over GReaT, but SGCO achieves the best average CDE, $\beta$-recall, C2ST quality score, and distinguishability-attack AUC among the tabulated objective variants, while retaining competitive utility. SGCO improves average CDE from 94.94 with DPO to 95.47 and $\alpha$-precision from 96.01 to 96.73. The important difference is in \emph{how} the resulting feedback is used: DPO assigns weights to individual chosen--rejected margins, whereas SGCO shares one nonlinear weight across the group-average contrast (Eq.~\ref{eq:sgco_gradient_comparison}). The group objective therefore uses the direction of the rewarded population contrast without requiring any particular row-to-row pairing.

\paragraph{GRPO comparison.}
To evaluate pool-level stratification, we compare SGCO with GRPO in a separate controlled post-training study. Table~\ref{tbl:grpo-main} reports the five dataset-level means and across-dataset deviations; Appendix~\ref{app:grpo-matched} presents all per-dataset results with their repeated-output deviations. GRPO's training reward increased, but CDE and $\alpha$-precision did not improve to the same extent. Compared with GRPO, SGCO achieves higher average CDE (94.41 versus 87.44) and $\alpha$-precision (96.99 versus 84.96); both measures are higher on all five datasets. The results suggest that SGCO makes more effective use of statistical feedback for improving distributional fidelity in our experiments.

\subsection{Reward ranking and group aggregation are essential to the improvement}
\label{sec:component_ablation}

\begin{table}[t]
\centering
\caption{Reward quality drives TabRSI improvement. Adult dataset, five iterations, $B=4$. Ranking-based SGCO is compared with random rewards and randomized group membership; lower Wasserstein/MMD/JSD is better, while MIA AUC is interpreted by distance to $0.5$.}
\label{tbl:ablation-signal-necessity}
\small
\setlength{\tabcolsep}{12pt}
\begin{tabular}{lcccc}
\toprule
\textbf{Method} & \textbf{Wasserstein$\downarrow$} & \textbf{MMD$\downarrow$} & \textbf{JSD$\downarrow$} & \textbf{MIA $\to0.5$} \\
\midrule
Baseline (GReaT)        & $0.0585{\scriptstyle\pm0.0002}$ & $0.0068{\scriptstyle\pm0.0001}$ & $0.0118{\scriptstyle\pm0.0005}$ & $0.5096{\scriptstyle\pm0.0004}$ \\
Random scoring          & $0.0589{\scriptstyle\pm0.0047}$ & $0.0060{\scriptstyle\pm0.0004}$ & $0.0098{\scriptstyle\pm0.0002}$ & $0.5090{\scriptstyle\pm0.0007}$ \\
Random group assignment & $0.0564{\scriptstyle\pm0.0003}$ & $0.0058{\scriptstyle\pm0.0006}$ & $0.0096{\scriptstyle\pm0.0003}$ & $\mathbf{0.5075{\scriptstyle\pm0.0007}}$ \\
\midrule
\textbf{SGCO}        & $\mathbf{0.0278{\scriptstyle\pm0.0004}}$ & $\mathbf{0.0012{\scriptstyle\pm0.0002}}$ & $\mathbf{0.0034{\scriptstyle\pm0.0001}}$ & $0.5091{\scriptstyle\pm0.0013}$ \\
\bottomrule
\end{tabular}
\end{table}

Table~\ref{tbl:ablation-signal-necessity} tests the information content of the feedback itself. With meaningful reward ranking, Adult Wasserstein distance drops from 0.0585 to \textbf{0.0278}, MMD from 0.0068 to \textbf{0.0012}, and JSD from 0.0118 to \textbf{0.0034}. Random scoring and random assignment of rows to high- and low-reward groups leave these metrics near their supervised starting point. Thus, the gains come from using \emph{quality-ordered} generated records rather than from performing extra optimization on arbitrary self-generated data. This also provides empirical support for the use of reward ranks as the entry point to SGCO.

\begin{figure*}[t]
    \centering
    \includegraphics[width=\linewidth]{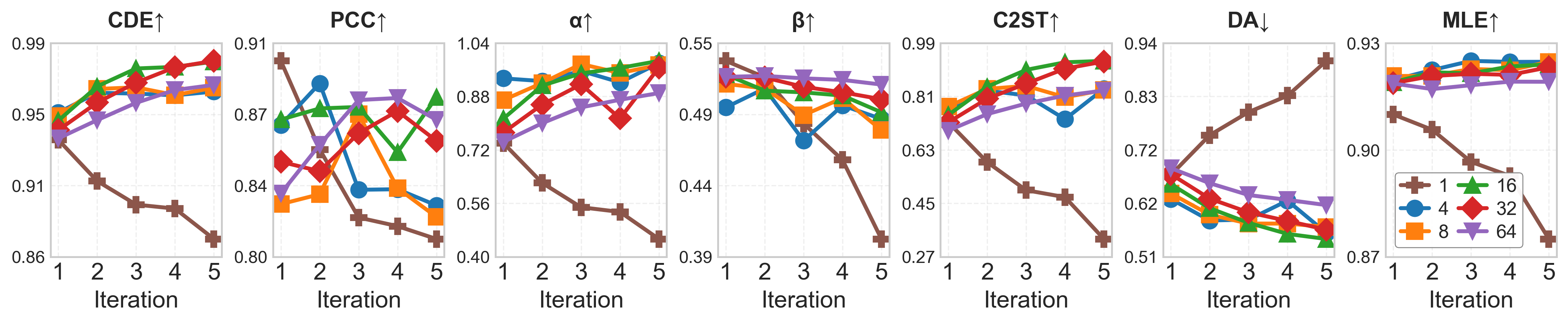}
    \caption{Group size governs the stability and sensitivity of SGCO. Five-round performance on Adult for $B\in\{1,4,8,16,32,64\}$. Moderate group sizes balance the variance of small-group contrasts against the smoothing effect of very large groups.}
    \label{fig:batch}
\end{figure*}

Figure~\ref{fig:batch} studies group size, directly probing the aggregation mechanism derived in Section~\ref{sec:sgco}. At $B=1$, the group mean reduces to an individual high-versus-low comparison and performance degrades across iterations. Moderate groups, particularly $B\in\{16,32\}$ on Adult, improve several metrics relative to the default $B=4$. The main experiments use $B=4$ under the available GPU budget. Larger groups provide stronger averaging of row-level fluctuations, and $B\in\{16,32\}$ performs well on several Adult metrics. Additional group-construction results, including random within-stratum sampling versus explicit top-to-bottom matching, are presented in Appendix~\ref{app:ablation_group}; the sensitivity to alignment strength $\beta$ is reported in Appendix~\ref{app:beta}.

\subsection{Adaptive reward modeling improves recursive improvement}
\label{sec:alternative_rewards}

\paragraph{The reward model follows the generator.}
Each policy update changes the distribution against which the reward classifier must discriminate. TabRSI responds by refitting the scorer on newly generated synthetic records before evaluating an independent alignment pool. Figure~\ref{fig:training} compares this \emph{refreshed} feedback with a scorer fixed at an earlier round on Adult. Under refresh, CDE and $\alpha$-precision improve and remain high across five rounds; under a fixed scorer, both measures decline as the generator continues to change. The retrained scorer also moves the distinguishability-attack AUC toward $0.5$ and sustains stronger downstream utility in this setting. The experiment directly links the adaptive feedback mechanism to the quality of successive generator updates.

\begin{figure*}[t]
    \centering
    \includegraphics[width=\linewidth]{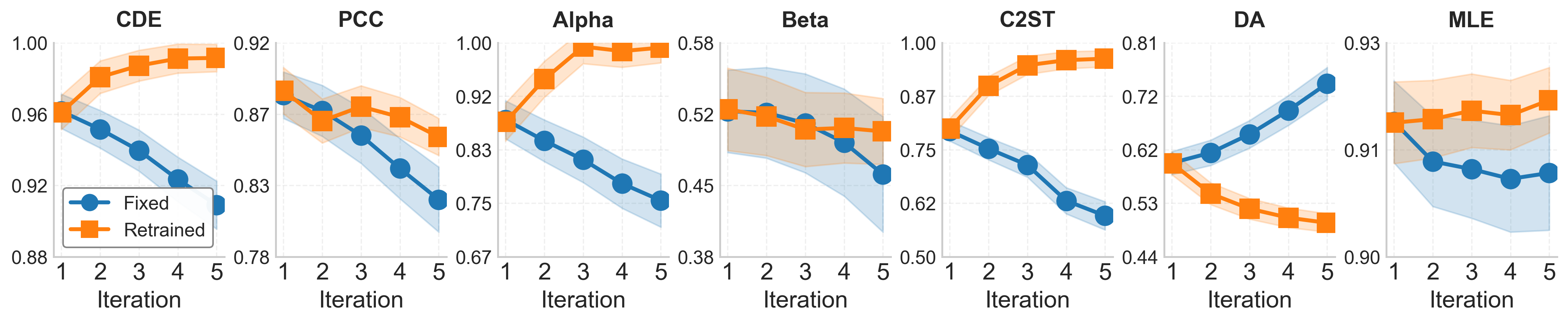}
     \caption{Comparison of fixed and refreshed reward models in TabRSI over five training rounds on Adult. Both settings use SGCO for policy updates. The refreshed real-versus-synthetic classifier (orange) is compared with a classifier held fixed across rounds (blue). Refreshing the scorer leads to higher CDE, $\alpha$-precision, and downstream utility, while moving DA AUC closer to $0.5$. These results show the benefit of updating statistical feedback as the generator improves.}
    \label{fig:training}
\end{figure*}

\paragraph{The same optimization accepts different reward families.}
Table~\ref{tbl:reward_comparison} summarizes the reward comparison across five dataset-level means; the full per-dataset measurements are retained in Appendix~\ref{app:further_results} (Table~\ref{tbl:reward_comparison_full}). Both discriminator-uncertainty and classifier-free DCR feedback improve several fidelity and utility measures over GReaT. SGCO therefore supports both learned and distance-based statistical rewards. The leakage control achieves anomalously small distributional distances but worsens MIA AUC and downstream utility, demonstrating why the independent scoring pool is essential. The uncertainty reward in Eq.~\eqref{eq:tabrsi-classifier-reward} deliberately favors discriminator ambiguity near $0.5$, rather than maximizing the real-class probability of individual generated rows.

\paragraph{Independent scoring matters.}
A deliberate leakage control reuses scorer-training records as policy-alignment records. This can produce deceptively low distances: on Adult the leaked-scorer condition reports a Wasserstein distance of 0.0183, but membership-inference AUC reaches 0.6609 and downstream AUC falls to 0.904. With separately sampled classifier feedback, the corresponding MIA AUC is 0.5094 and downstream AUC is 0.920. The same qualitative pattern appears on other benchmarks. Thus, \emph{scorer refresh and scorer/alignment separation play complementary roles}: the former adapts feedback to the changing policy, and the latter keeps the alignment signal independent of the scorer's fitting examples. Appendix~\ref{app:classifier_variants} further compares Random Forest and XGBoost reward models; Appendix~\ref{app:architecture} evaluates three autoregressive backbones, and Appendix~\ref{app:compute} reports the computational cost.

\section{Discussion}
\label{sec:discussion}

Tabular self-improvement requires feedback to evolve with the generator. A scorer trained on early synthetic records may prioritize artifacts the current policy has already corrected while missing newly emerging discrepancies. Its rankings can become stale, limiting or even reversing later improvements. TabRSI instead refits the discriminator against the current synthetic distribution every round and ranks an independent alignment pool. On Adult, refreshed feedback sustains CDE and $\alpha$-precision over five rounds, whereas a fixed scorer degrades (Figure~\ref{fig:training}). The leakage experiment further shows why out-of-sample scoring matters: reusing discriminator-training records for alignment can yield deceptively favorable distances alongside worse membership-inference diagnostics and downstream utility.

SGCO's advantage over GRPO is closely tied to how relative quality is defined. GRPO standardizes rewards within each sampled response group, an effective comparison when responses address the same reasoning problem. In tabular synthesis, however, rows usually represent draws from a distribution rather than competing answers to a specific question. A relatively faithful row can receive a negative GRPO advantage when sampled alongside still higher-scoring rows, while a poor row can receive a positive advantage in a weaker group. When within-group scores are close, normalization can also accentuate small differences in noisy statistical feedback. These group-dependent assignments need not reflect a record's standing in the broader synthetic population. SGCO instead ranks the full candidate pool and forms high- and low-reward strata before optimization, preserving a consistent preferred-versus-dispreferred role for each stratum within the round. Its group-average policy/reference contrast uses population-level rankings without within-group reward standardization. The matched comparison supports the practical value of this design: SGCO exceeds GRPO in CDE and $\alpha$-precision on all five datasets, with particularly large $\alpha$-precision gains on Adult ($0.9720$ versus $0.7597$) and Shoppers ($0.9820$ versus $0.7698$). DPO presents a different issue: although it also uses reward-ranked rows, its pair-specific nonlinear weights depend on arbitrarily matched row pairs. SGCO averages the high/low log-ratios before nonlinear weighting, yielding pairing invariance and a shared group gradient (Eq.~\ref{eq:sgco_gradient_comparison}). The corresponding average CDE and $\alpha$-precision improvements over adapted DPO ($94.94$ to $95.47$ and $96.01$ to $96.73$) complement this mathematical distinction.

The results show different gains across aspects of tabular quality.
CDE and density-region $\alpha$-precision improve strongly,
while average pairwise-correlation similarity remains close to
the supervised baseline (60.76 for GReaT and 60.70 for SGCO).
In the additional comparison with non-LM generators, SGCO
reduces Shape and Trend errors relative to GReaT across all
five datasets. Compared with diffusion-based models, the
best-performing SGCO variants achieve higher $\alpha$-precision
on Adult, Shoppers, and Beijing, higher $\beta$-recall on
Adult and Shoppers, and better downstream classification AUC
on Adult and Default. These results show that SGCO can
achieve strong sample fidelity, distribution coverage, and
downstream utility, although diffusion models generally
perform better in preserving feature correlations.
Future work can incorporate correlation-aware and rare-region
rewards to further improve feature dependencies and coverage.

In summary, these results establish TabRSI as a dual-evolution framework in which the generating policy and its statistical evaluator adapt recursively beyond supervised fine-tuning. TabRSI brings round-wise generator--scorer adaptation, independent scorer-training and alignment pools, and stratified group-contrast optimization into one recursive procedure for tabular language models. The same principle provides a foundation for continual tabular synthesis under data drift: as feature dependencies, minority-group coverage, or population structure change, both the generator and the statistical feedback guiding its updates can be renewed.
\section{Conclusion}
We introduced \textbf{TabRSI}, a recursive self-improvement framework that alternates between updating an autoregressive tabular generator and refitting its reward model against the generator's latest synthetic population. The refreshed reward ranks independently sampled candidates, and \textbf{SGCO} transforms high- and low-reward strata into a pairing-invariant policy update with shared group weighting. Across five mixed-type benchmarks, TabRSI substantially improves the fidelity of a supervised GReaT generator; in separately matched experiments, SGCO achieves higher column-density similarity and $\alpha$-precision than GRPO on every dataset. Reward-refresh, group-size, classifier-free reward, and scorer-separation studies show how the optimization rule and feedback adaptation work together during successive refinement rounds. TabRSI therefore establishes a practical approach to recursively improving tabular LMs through their own generated data and continually refreshed quality feedback.

\section*{Statement of Broader Impact}

Improved synthetic tabular data can facilitate data
sharing, benchmarking, and model development when
access to real records is limited. However, high
fidelity may coexist with memorization risks,
membership-inference vulnerabilities, and inadequate
subgroup coverage. The reported empirical attack
metrics do not constitute formal privacy guarantees.
Deployment involving sensitive data therefore
requires independent privacy, fairness, and
domain-specific evaluation.

\bibliographystyle{tmlr}
\bibliography{main}
\appendix
\section{Additional background}
\label{app:background}

\subsection{Tabular language model setup}
\label{app:problem_setup}

Let $\pi_\theta$ denote an autoregressive language model used to generate tabular rows. Following the GReaT framework~\citep{borisov2022language}, each tabular row
$x=(v_1,\ldots,v_M)$ is serialized into a token sequence
$y=(\tau_1,\ldots,\tau_T)=\mathrm{serialize}(x)$, and the model defines
\begin{equation}
\pi_\theta(y)
=
\prod_{t=1}^{T}
\pi_\theta(\tau_t\mid \tau_{<t}).
\end{equation}
Given a real dataset $\mathcal D_{\mathrm{real}}$, supervised fine-tuning minimizes
\begin{equation}
\mathcal L_{\mathrm{SFT}}(\theta)
=
-\mathbb E_{x\sim \mathcal D_{\mathrm{real}}}
\left[
\log \pi_\theta(\mathrm{serialize}(x))
\right].
\end{equation}
The resulting model $\pi_{\theta_0}$ is the initial tabular generator. We use $\pi_{\theta_0}$ as the fixed reference model $\pi_{\mathrm{ref}}$ during post-training.

\subsection{Alignment baselines for tabular language models}
\label{app:alignment_baselines}

Post-training with reinforcement learning or preference alignment is a natural direction for autoregressive tabular language models because they expose sequence log-probabilities. However, tabular quality is primarily statistical and is usually evaluated over groups of rows through distributional fidelity, feature dependence, downstream utility, and empirical attack indistinguishability. This makes human preference collection and instance-level desirability labels difficult to define.

We compare against general alignment objectives adapted to tabular generation using the same reward-ranked synthetic pools as SGCO.

\paragraph{Reward-ranked sample selection.}
At each post-training iteration, all adapted alignment baselines use the same pool construction and reward-scoring procedure as SGCO. We sort generated rows by reward and define high- and low-reward strata. By default, we use a top-vs-bottom construction: samples with the highest rewards form $X^+$ and samples with the lowest rewards form $X^-$. The ranking can also support softer middle-strata variants, such as adjacent or offset comparisons, which we study in Appendix~\ref{app:ablation_group}; however, the main experiments use the top-vs-bottom setting because it provides the clearest reward contrast. For DPO, we construct preferred--rejected pairs $(y^+,y^-)$ by matching samples from $X^+$ and $X^-$. For KTO, samples from $X^+$ are labeled desirable and samples from $X^-$ are labeled undesirable. For NPO, samples from $X^-$ define the negative set. This keeps the reward signal and generated samples fixed across methods, isolating the effect of the alignment objective.

\paragraph{Direct Preference Optimization (DPO).}
DPO~\citep{rafailov2024directpreferenceoptimizationlanguage} optimizes pairwise preferences $(y^+ \succ y^-)$:
\begin{equation}
\mathcal L_{\mathrm{DPO}}(\theta)
=
-\mathbb E_{(y^+,y^-)}
\left[
\log \sigma
\left(
\beta
\log
\frac{\pi_\theta(y^+)}{\pi_{\mathrm{ref}}(y^+)}
-
\beta
\log
\frac{\pi_\theta(y^-)}{\pi_{\mathrm{ref}}(y^-)}
\right)
\right].
\end{equation}
For tabular adaptation, $y^+$ is sampled from $X^+$ and $y^-$ from $X^-$.

\paragraph{KTO.}
KTO~\citep{ethayarajh2024ktomodelalignmentprospect} uses desirable and undesirable labels rather than explicit preference pairs. In our tabular adaptation, samples from $X^+$ are treated as desirable and samples from $X^-$ are treated as undesirable.

\paragraph{NPO.}
NPO~\citep{zhang2024negative} suppresses dispreferred samples with a negative-only objective. In our tabular adaptation, samples from $X^-$ define the negative set. This makes NPO especially relevant for suppression-style objectives such as forget-region suppression.

\paragraph{GRPO.}
GRPO~\citep{shao2024deepseekmath} also uses groups of samples, but it is a policy-gradient reinforcement learning method that normalizes rewards among multiple responses to the same prompt. SGCO instead forms reward-stratified high/low groups across generated tabular rows and optimizes relative policy-reference log-ratios between these groups.

\subsection{Loss variants and the SFT anchor batch}
\label{app:loss_variants}

Table~\ref{tbl:ablation_loss_variants} and the per-dataset tables include several
loss variants used as ablations on top of the corresponding base objective.
These loss variants share the same
generated alignment pools, reward scores, reference model
$\pi_{\mathrm{ref}}=\pi_{\theta_0}$, and reward-ranked high/low strata as the
base objective. For the KL-regularized and gradient-difference variants, the
trainer additionally constructs a non-empty SFT anchor batch
$\mathcal B_{\mathrm{anc}}$ from the original real training table used for
supervised fine-tuning. In our implementation, this anchor batch is formed from
the first 100 rows of \texttt{synthetic/<dataset>/real.csv}. The rows are loaded
in a chosen--rejected format with identical chosen and rejected entries; only the
chosen branch is used by the anchor regularizers.

\paragraph{KL-regularized variants.}
For DPO and NPO, the KL variant adds a token-level reference penalty on the
anchor batch:
\begin{equation}
\mathcal L_{\mathrm{method+KL}}(\theta)
=
\mathcal L_{\mathrm{method}}(\theta)
+
\lambda_{\mathrm{KL}}
\mathbb E_{y\in\mathcal B_{\mathrm{anc}}}
\left[
\frac{1}{|y|}
\sum_{t=1}^{|y|}
\mathrm{KL}
\left(
\pi_\theta(\cdot\mid y_{<t})
\,\|\, 
\pi_{\mathrm{ref}}(\cdot\mid y_{<t})
\right)
\right].
\end{equation}
This term discourages the post-trained policy from drifting away from the
supervised reference distribution on real anchor rows. In the implementation,
DPO+KL and NPO+KL use the trainer's default KL coefficient.

\paragraph{Gradient-difference variants.}
The gradient-difference variant adds a cross-entropy anchoring term on the same
anchor batch:
\begin{equation}
\mathcal L_{\mathrm{method+GD}}(\theta)
=
\mathcal L_{\mathrm{method}}(\theta)
+
\lambda_{\mathrm{GD}}
\mathbb E_{y\in\mathcal B_{\mathrm{anc}}}
\left[
-\log \pi_\theta(y)
\right].
\end{equation}
This term provides supplementary supervised anchoring on real rows while the base
alignment objective promotes high-reward generated rows and suppresses
low-reward or rejected generated rows. In the implementation, DPO+GD and the
SGCO log-sigmoid+GD variant use coefficient $1.0$, while the NPO+GD variant
uses the configured \texttt{grad\_diff\_coeff}.

\paragraph{Log-sigmoid SGCO variants.}
The main SGCO objective uses the bounded sigmoid surrogate
$\sigma(-\Delta_{\mathrm{group}})$. The log-sigmoid variant replaces this
surrogate with the group-level Bradley--Terry negative log-likelihood:
\begin{equation}
\mathcal L_{\mathrm{SGCO\mbox{-}logsig}}(\theta)
=
-\log \sigma(\Delta_{\mathrm{group}}(\theta)),
\qquad
\Delta_{\mathrm{group}}(\theta)
=
\bar r_\theta^{\mathrm{high}}
-
\bar r_\theta^{\mathrm{low}} .
\end{equation}
The combined log-sigmoid+GD variant further adds the anchor cross-entropy term:
\begin{equation}
\mathcal L_{\mathrm{SGCO\mbox{-}logsig+GD}}(\theta)
=
\mathcal L_{\mathrm{SGCO\mbox{-}logsig}}(\theta)
+
\lambda_{\mathrm{GD}}
\mathbb E_{y\in\mathcal B_{\mathrm{anc}}}
\left[
-\log \pi_\theta(y)
\right].
\end{equation}
The default SGCO objective uses the group-contrast loss alone.

\paragraph{Interpretation.}
Because $\mathcal B_{\mathrm{anc}}$ is drawn from the original SFT training
distribution, the KL variants act as anti-drift reference anchors and the
gradient-difference variants act as supplementary supervised anchors on real
rows. These ablations show that the bounded-sigmoid SGCO objective achieves the strongest average trade-off in Table~\ref{tbl:ablation_loss_variants} without an additional anchor term. The regularized variants have similar results, suggesting that the anchors mainly constrain changes from the supervised model.

\subsection{Reward signals used in SGCO}
\label{app:reward_taxonomy}

SGCO uses a scalar reward signal only to score and rank generated samples. Formally, a reward is any finite scalar function
\begin{equation}
\rho:\mathcal X\rightarrow \mathbb R.
\end{equation}
For quality-refinement experiments, rewards are often normalized to $[0,1]$. For constraint-based objectives such as forget-region suppression, rewards may be signed. The reward scorer is not differentiated through; it only determines group membership.

\paragraph{Distributional classifier reward.}
For generation-quality refinement, TabRSI refits its default reward model
$\phi_t:\mathcal X\rightarrow[0,1]$ at every iteration to distinguish real
records from fresh synthetic records sampled from the current policy. The
classifier is trained on $\widetilde{\mathcal D}^{(t)}_{\mathrm{score}}$ but ranks
only an independently sampled $\widetilde{\mathcal D}^{(t)}_{\mathrm{align}}$.
Its output is converted into the sample-level distinguishability reward:
\begin{equation}
\rho_{\mathrm{cls}}(\tilde x)
=
1-2|0.5-\phi_t(\tilde x)|.
\end{equation}
This reward is maximized when $\phi_t(\tilde x)=0.5$, i.e., when the classifier
is maximally uncertain about whether $\tilde x$ is real or synthetic. Thus,
higher values indicate that $\tilde x$ is harder to distinguish from real data.

Refitting the scorer after each policy update keeps the statistical feedback aligned with the current synthetic distribution. The reward ranks individual generated records, while held-out DA AUC evaluates distinguishability at the dataset level.

\paragraph{Distance-based reward.}
A classifier-free alternative is Distance-to-Closest-Record (DCR):
\begin{equation}
d(\tilde x)
=
\min_{x\in\mathcal D_{\mathrm{real}}}
\|\tilde x-x\|_2,
\qquad
\rho_{\mathrm{dcr}}(\tilde x)
=
1-
\frac{d(\tilde x)-d_{\min}}{d_{\max}-d_{\min}}.
\end{equation}
This reward measures proximity to the support of the reference data.

\paragraph{Target-distribution reward.}
For distribution-shift adaptation, the reference set in the distributional reward can be replaced by a target-domain alignment set. For example, the classifier or DCR reward can be computed against a shifted target distribution rather than the original training distribution.

\paragraph{Constraint-based reward.}
For empirical forget-region suppression, the reward need not compare generated samples to a density. Given a specified forget region $F\subseteq\mathcal X$, we use a signed reward
\begin{equation}
\rho_{\mathrm{forget}}(\tilde x)
=
-\mathbf 1[\tilde x\in F]
+
\mathbf 1[\tilde x\notin F].
\end{equation}
This penalizes generated samples inside the forget region and rewards samples in the retain region. This experiment measures suppression of the specified forget region in generated data.

\subsection{Reward-ranked group formation}
\label{app:group_formation}

At each post-training iteration, a synthetic pool
\[
\tilde X
=
\{\tilde x_i\}_{i=1}^{N}
\]
is generated and each sample is scored by the current reward $\rho(\tilde x_i)$. The pool is sorted by reward, and two strata are formed:
\begin{equation}
X^+
=
\{\tilde x:\rho(\tilde x)\geq \tau_p\},
\qquad
X^-
=
\{\tilde x:\rho(\tilde x)\leq \tau_q\},
\end{equation}
where $\tau_p$ and $\tau_q$ are percentile thresholds. In the default setting, $X^+$ and $X^-$ are the top and bottom halves of the scored pool.

A SGCO mini-batch samples two groups:
\begin{equation}
\mathcal B_{\mathrm{high}}
\sim
\mathrm{Uniform}(X^+),
\qquad
\mathcal B_{\mathrm{low}}
\sim
\mathrm{Uniform}(X^-),
\end{equation}
with
\[
|\mathcal B_{\mathrm{high}}|
=
|\mathcal B_{\mathrm{low}}|
=
B.
\]
This construction is group-based, not pair-based. For adapted pairwise baselines such as DPO, the same ranked pool can be converted into preferred--rejected pairs, but such pairing is not part of the SGCO objective.

\section{Group-level Bradley--Terry derivation of SGCO}
\label{app:sgco_bt_derivation}

This appendix derives SGCO as a group-level analogue of Bradley--Terry preference modeling. The purpose is to formalize why SGCO compares reward-stratified groups rather than individual preferred--rejected pairs.

\subsection{From pairwise preferences to group preferences}

In standard pairwise preference optimization, a preferred sample $y^+$ and a rejected sample $y^-$ are compared through a Bradley--Terry model:
\begin{equation}
P_\theta(y^+ \succ y^-)
=
\sigma
\left(
r_\theta(y^+) - r_\theta(y^-)
\right),
\end{equation}
where $r_\theta(y)$ is an implicit reward. In preference optimization with a fixed reference model, this implicit reward is commonly written as a policy-reference log-ratio:
\begin{equation}
r_\theta(y)
=
\beta
\log
\frac{\pi_\theta(y)}{\pi_{\mathrm{ref}}(y)}.
\end{equation}
Here $\beta>0$ controls the strength of the alignment update.

Tabular generation differs from text preference modeling because the quality of a synthetic table is usually distributional. Individual rows may look plausible even when a group of generated rows fails to preserve marginals, correlations, utility-relevant structure, or empirical attack indistinguishability. SGCO therefore replaces instance-level comparison with group-level comparison.

\subsection{Reward-stratified groups}

At each post-training iteration, generated samples are scored by a task-specified reward signal $\rho:\mathcal X\rightarrow\mathbb R$. The reward is used only to rank samples and form two reward-stratified groups:
\[
\mathcal{B}_{\mathrm{high}}
\subset
\mathcal{G}_{\mathrm{high}},
\qquad
\mathcal{B}_{\mathrm{low}}
\subset
\mathcal{G}_{\mathrm{low}},
\qquad
|\mathcal{B}_{\mathrm{high}}|
=
|\mathcal{B}_{\mathrm{low}}|
=
B.
\]
The reward scorer is not differentiated through; it only determines group membership. This allows the reward to be a classifier score, a distance-based score, a target-distribution score, or a constraint-based signal.

\subsection{Group-Level implicit reward}

Given the sample-level implicit reward
\[
r_\theta(y)
=
\beta
\log
\frac{\pi_\theta(y)}{\pi_{\mathrm{ref}}(y)},
\]
we define the group-level implicit reward as the average over each group:
\begin{equation}
\bar r_\theta^{\mathrm{high}}
=
\frac{1}{B}
\sum_{y\in\mathcal{B}_{\mathrm{high}}}
r_\theta(y),
\qquad
\bar r_\theta^{\mathrm{low}}
=
\frac{1}{B}
\sum_{y\in\mathcal{B}_{\mathrm{low}}}
r_\theta(y).
\end{equation}
The corresponding stratified group contrast is
\begin{equation}
\Delta_{\mathrm{group}}(\theta)
=
\bar r_\theta^{\mathrm{high}}
-
\bar r_\theta^{\mathrm{low}}.
\end{equation}

\subsection{Group-level Bradley--Terry model}

We model the probability that the high-reward group is preferred to the low-reward group using a Bradley--Terry form:
\begin{equation}
P_\theta
\left(
\mathcal{B}_{\mathrm{high}}
\succ
\mathcal{B}_{\mathrm{low}}
\right)
=
\sigma
\left(
\Delta_{\mathrm{group}}(\theta)
\right)
=
\sigma
\left(
\bar r_\theta^{\mathrm{high}}
-
\bar r_\theta^{\mathrm{low}}
\right).
\end{equation}
Maximizing this probability encourages the policy to assign higher relative likelihood to high-reward groups than to low-reward groups, both measured against the fixed reference model.

The standard group-level Bradley--Terry negative log-likelihood form is
\begin{equation}
\mathcal L_{\mathrm{SGCO}}^{\mathrm{BT}}(\theta)
=
-\log \sigma
\left(
\bar r^{\mathrm{high}}_\theta-\bar r^{\mathrm{low}}_\theta
\right).
\end{equation}
In the main experiments, we instead use the bounded sigmoid surrogate
\begin{equation}
\mathcal{L}_{\mathrm{SGCO}}(\theta)
=
\sigma
\left(
-\Delta_{\mathrm{group}}(\theta)
\right)
=
\sigma
\left(
\bar r_\theta^{\mathrm{low}}
-
\bar r_\theta^{\mathrm{high}}
\right),
\end{equation}
where
\begin{equation}
\Delta_{\mathrm{group}}(\theta)
=
\bar r_\theta^{\mathrm{high}}
-
\bar r_\theta^{\mathrm{low}}.
\end{equation}
The bounded and log-sigmoid objectives both decrease with $\Delta_{\mathrm{group}}(\theta)$, preserving the same high-versus-low preference direction while giving different gradient weights. Minimizing either
objective increases the relative policy-reference log-ratio of the high-reward
group over the low-reward group. We use the bounded surrogate in the main
experiments because it has a finite numerical range and provides a stable update
during iterative post-training.

\subsection{Gradient structure}

Let
\[
u(\theta)
=
\bar r_\theta^{\mathrm{low}}
-
\bar r_\theta^{\mathrm{high}}.
\]
Then
\[
\mathcal{L}_{\mathrm{SGCO}}(\theta)
=
\sigma(u(\theta)).
\]
By the chain rule,
\begin{equation}
\nabla_\theta
\mathcal{L}_{\mathrm{SGCO}}(\theta)
=
\sigma'(u(\theta))
\left(
\nabla_\theta
\bar r_\theta^{\mathrm{low}}
-
\nabla_\theta
\bar r_\theta^{\mathrm{high}}
\right).
\end{equation}
Thus, both groups remain active in the computation graph. The gradient increases the relative likelihood of high-reward groups and decreases the relative likelihood of low-reward groups. This differs from negative-only objectives, which primarily suppress undesirable samples, and from pairwise methods, which require explicit one-to-one preferred--rejected pair construction.

\subsection{Singleton limit}
\label{app:singleton_limit}

When $B=1$, each high- and low-reward group contains only one sample. Let
$y^+ \in \mathcal B_{\mathrm{high}}$ and
$y^- \in \mathcal B_{\mathrm{low}}$. Then
\[
\bar r_\theta^{\mathrm{high}}
=
r_\theta(y^+),
\qquad
\bar r_\theta^{\mathrm{low}}
=
r_\theta(y^-),
\]
and therefore
\[
\Delta_{\mathrm{group}}(\theta)
=
r_\theta(y^+)
-
r_\theta(y^-).
\]
The SGCO loss reduces to
\[
\mathcal L_{\mathrm{SGCO}}(\theta)
=
\sigma
\left(
-\Delta_{\mathrm{group}}(\theta)
\right)
=
\sigma
\left(
r_\theta(y^-)
-
r_\theta(y^+)
\right).
\]
Minimizing this loss encourages
\[
r_\theta(y^+) > r_\theta(y^-),
\]
so the update still has the intended high-vs-low contrastive direction.
However, the group-level mechanism of SGCO disappears at $B=1$. In this
singleton limit, each optimizer step is driven by one high-reward sample and one
low-reward sample rather than by an average over reward-stratified groups. The $B=1$ ablation therefore isolates the effect of removing group averaging.

For $B>1$, SGCO restores the intended group-level signal:
\[
\Delta_{\mathrm{group}}(\theta)
=
\frac{1}{B}
\sum_{i=1}^{B}
r_\theta(y_i^{\mathrm{high}})
-
\frac{1}{B}
\sum_{i=1}^{B}
r_\theta(y_i^{\mathrm{low}}).
\]
This quantity is an empirical estimate of a population-level high-vs-low reward
contrast. Averaging over reward-stratified groups reduces row-level fluctuations
and produces a more stable distributional alignment signal. The next subsection
explains why this group-mean construction is well matched to tabular generation,
and Appendix~\ref{app:theoretical_proofs} analyzes the resulting variance
reduction.

\subsection{Why group-mean averaging matches tabular alignment}
\label{app:why_groupmean}

The group-mean aggregation in SGCO reflects the population-level nature of
tabular distribution alignment. In tabular synthesis, the objective is not to
decide whether one individual row is preferable to another individual row.
Instead, the goal is to make the synthetic row distribution induced by
$\pi_\theta$ match a target tabular distribution under population-level quality
criteria.

Abstractly, let $P_\theta$ denote the synthetic distribution induced by the
tabular language model and let $P_\star$ denote the target distribution. The
tabular alignment goal can be viewed as optimizing
\begin{equation}
\theta^\star
\in
\arg\max_\theta
\mathcal Q(P_\theta,P_\star),
\end{equation}
where $\mathcal Q$ is a distributional quality functional. In practice,
$\mathcal Q$ is approximated by metrics such as distributional fidelity,
feature-dependence preservation, downstream utility, and privacy-oriented
indistinguishability. These criteria are evaluated over populations of rows,
not isolated row-level comparisons.

This motivates the group-level contrast used by SGCO. For fixed high- and
low-reward strata, define the population-level log-ratio contrast
\begin{equation}
\Delta_{\mathrm{pop}}(\theta)
=
\mathbb E_{y\sim \mathcal G_{\mathrm{high}}}
\left[
\beta
\log
\frac{\pi_\theta(y)}{\pi_{\mathrm{ref}}(y)}
\right]
-
\mathbb E_{y\sim \mathcal G_{\mathrm{low}}}
\left[
\beta
\log
\frac{\pi_\theta(y)}{\pi_{\mathrm{ref}}(y)}
\right].
\end{equation}
A SGCO mini-batch estimates this population contrast by averaging over
reward-stratified groups:
\begin{equation}
\Delta_{\mathrm{group}}(\theta)
=
\bar r_\theta^{\mathrm{high}}
-
\bar r_\theta^{\mathrm{low}},
\end{equation}
where
\begin{equation}
\bar r_\theta^{\mathrm{high}}
=
\frac{\beta}{B}
\sum_{y\in\mathcal B_{\mathrm{high}}}
\log
\frac{\pi_\theta(y)}{\pi_{\mathrm{ref}}(y)},
\qquad
\bar r_\theta^{\mathrm{low}}
=
\frac{\beta}{B}
\sum_{y\in\mathcal B_{\mathrm{low}}}
\log
\frac{\pi_\theta(y)}{\pi_{\mathrm{ref}}(y)}.
\end{equation}
The loss used in SGCO is then
\begin{equation}
\mathcal L_{\mathrm{SGCO}}(\theta)
=
\sigma
\left(
-\Delta_{\mathrm{group}}(\theta)
\right)
=
\sigma
\left(
\bar r_\theta^{\mathrm{low}}
-
\bar r_\theta^{\mathrm{high}}
\right).
\end{equation}
Thus, SGCO first forms a group-level estimate of the high-vs-low reward contrast
and then applies the nonlinear alignment loss to this aggregate quantity.

This aggregation order is aligned with the tabular objective. Since tabular
quality is assessed through distribution-level statistics, the update should
reflect the collective behavior of reward-stratified groups rather than the
idiosyncrasies of individual samples. Averaging the policy--reference log-ratios
before applying the loss produces a consensus direction: high-reward groups are
collectively encouraged and low-reward groups are collectively discouraged.

The resulting gradient has the form
\begin{equation}
\nabla_\theta
\mathcal L_{\mathrm{SGCO}}(\theta)
=
-\sigma
\left(
-\Delta_{\mathrm{group}}(\theta)
\right)
\left(
1-
\sigma
\left(
-\Delta_{\mathrm{group}}(\theta)
\right)
\right)
\nabla_\theta
\Delta_{\mathrm{group}}(\theta).
\end{equation}
Since
\begin{equation}
\nabla_\theta
\Delta_{\mathrm{group}}(\theta)
=
\nabla_\theta
\bar r_\theta^{\mathrm{high}}
-
\nabla_\theta
\bar r_\theta^{\mathrm{low}},
\end{equation}
both high- and low-reward groups contribute to the update. The scalar sigmoid
factor depends on the group-level contrast, while the direction is determined by
the average difference between the two reward-stratified groups.

This design is also consistent with the stability analysis in
Appendix~\ref{app:theoretical_proofs}. Because the group-level contrast averages
over $B$ samples from each stratum, the variance of the group-averaged
log-ratio gradient decreases as $\mathcal O(1/B)$ under i.i.d. group sampling.
Moderate group sizes therefore reduce noisy row-level fluctuations and provide a
more stable signal for iterative tabular post-training.

Group-mean aggregation is designed for tabular generation, where the quality of a synthetic distribution is measured across many records. It lets SGCO use reward-stratified population contrasts instead of individual pair assignments.

\section{SGCO surrogate stability properties}
\label{app:theoretical_proofs}

This appendix analyzes the fixed-round SGCO surrogate obtained after a synthetic sample pool has been generated, scored, and partitioned into high- and low-reward groups. The bound concerns within-round SGCO optimization; the full TabRSI recursion also changes the generator, fitted reward model, and group assignments across iterations. Instead, we establish stability properties of the fixed surrogate optimized within one alignment round: bounded gradients, $\mathcal{O}(1/B)$ variance reduction from group averaging, and a standard stochastic-gradient descent bound under smoothness.

At iteration $t$, let $\tilde{\mathcal D}^{(t)}$ be a generated synthetic pool and let $\rho_t:\mathcal X\rightarrow\mathbb R$ be a scalar reward signal. The reward is used to rank generated samples and construct high- and low-reward strata
$\mathcal G_{\mathrm{high}}^{(t)}$ and $\mathcal G_{\mathrm{low}}^{(t)}$. During the alignment update, the generated pool, reward scores, and group assignments are fixed. Gradients are taken only through the policy-reference log-ratio terms.

\subsection{Notation and assumptions}

Each generated row is represented by a serialized token sequence
$y=(\tau_1,\ldots,\tau_T)$. By autoregressive factorization,
\begin{equation}
\log \pi_\theta(y)
=
\sum_{t=1}^{T}
\log \pi_\theta(\tau_t\mid \tau_{<t}),
\qquad
\nabla_\theta \log \pi_\theta(y)
=
\sum_{t=1}^{T}
\nabla_\theta
\log \pi_\theta(\tau_t\mid \tau_{<t}).
\end{equation}

A SGCO mini-batch is formed by independently sampling
\begin{equation}
\mathcal B_{\mathrm{high}}
=
\{y_i^{\mathrm{high}}\}_{i=1}^{B}
\stackrel{\mathrm{i.i.d.}}{\sim}
\mathrm{Uniform}(\mathcal G_{\mathrm{high}}),
\qquad
\mathcal B_{\mathrm{low}}
=
\{y_i^{\mathrm{low}}\}_{i=1}^{B}
\stackrel{\mathrm{i.i.d.}}{\sim}
\mathrm{Uniform}(\mathcal G_{\mathrm{low}}).
\end{equation}
This construction does not require pairwise correspondence between high- and low-reward samples.

The group-averaged log-ratios are
\begin{equation}
\bar r_\theta^{\mathrm{high}}
=
\frac{\beta}{B}
\sum_{i=1}^{B}
\log
\frac{\pi_\theta(y_i^{\mathrm{high}})}
{\pi_{\mathrm{ref}}(y_i^{\mathrm{high}})},
\qquad
\bar r_\theta^{\mathrm{low}}
=
\frac{\beta}{B}
\sum_{i=1}^{B}
\log
\frac{\pi_\theta(y_i^{\mathrm{low}})}
{\pi_{\mathrm{ref}}(y_i^{\mathrm{low}})}.
\end{equation}
The fixed-batch SGCO surrogate is
\begin{equation}
\mathcal L_{\mathrm{SGCO}}(\theta)
=
\sigma
\left(
\bar r_\theta^{\mathrm{low}}
-
\bar r_\theta^{\mathrm{high}}
\right),
\label{eq:app_sgco_loss}
\end{equation}
where $\sigma(u)=(1+e^{-u})^{-1}$.

\begin{assumption}[Bounded per-sequence policy gradient]
\label{asm:boundedG}
There exists $G>0$ such that
\[
\|\nabla_\theta \log \pi_\theta(y)\|
\leq
G
\]
for every admissible sequence $y$ and parameter value $\theta$ considered during post-training.
\end{assumption}

\begin{assumption}[$L$-smooth fixed surrogate]
\label{asm:smooth}
For a fixed generated pool and fixed group-sampling distribution, the expected SGCO surrogate is $L$-smooth in $\theta$.
\end{assumption}

\subsection{Main stability statement}

\begin{theorem}[Stability of the fixed-round SGCO surrogate]
\label{thm:sgco_stability}
Fix a generated synthetic pool, reward scores, and the induced high- and low-reward group-sampling distributions. Let
\[
\mathcal J(\theta)
=
\mathbb E[
\mathcal L_{\mathrm{SGCO}}(\theta)
]
\]
be the expected SGCO surrogate over mini-batches sampled from these fixed groups. Under Assumptions~\ref{asm:boundedG} and~\ref{asm:smooth}, the following properties hold:
\begin{enumerate}
    \item The fixed-batch SGCO loss has bounded gradients:
    \[
    \|\nabla_\theta \mathcal L_{\mathrm{SGCO}}(\theta)\|
    \leq
    \frac{\beta G}{2}.
    \]
    \item Under i.i.d. sampling within each reward group, the variance of the group-averaged log-ratio gradient decreases as $\mathcal O(1/B)$.
    \item If the stochastic mini-batch gradient $\hat g_k$ is unbiased with variance bounded by $\sigma_{\mathcal L}^2$, then for step size $\eta\leq 1/L$,
    \[
    \min_{0\leq k<K}
    \mathbb E[
    \|\nabla \mathcal J(\theta_k)\|^2
    ]
    \leq
    \frac{
    2(\mathcal J(\theta_0)-\mathcal J^*)
    }{
    \eta K
    }
    +
    \eta L\sigma_{\mathcal L}^2,
    \]
    where $\mathcal J^*=\inf_\theta \mathcal J(\theta)$.
\end{enumerate}
\end{theorem}

The proposition gives a within-round optimization bound for fixed reward strata. The main experiments test how the complete generator--scorer loop behaves over successive rounds.

\subsection{Supporting lemmas}

\begin{lemma}[Bounded gradient]
\label{lem:grad_bound}
Under Assumption~\ref{asm:boundedG}, the fixed-batch SGCO loss satisfies
\begin{equation}
\|\nabla_\theta \mathcal L_{\mathrm{SGCO}}(\theta)\|
\leq
\frac{\beta G}{2}.
\end{equation}
\end{lemma}

\begin{proof}
Let
\[
u(\theta)
=
\bar r_\theta^{\mathrm{low}}
-
\bar r_\theta^{\mathrm{high}},
\qquad
\mathcal L_{\mathrm{SGCO}}(\theta)
=
\sigma(u(\theta)),
\]
where
\[
\sigma(u)
=
(1+e^{-u})^{-1}.
\]
The derivative of the sigmoid is
\[
\sigma'(u)
=
\frac{e^{-u}}{(1+e^{-u})^2}
=
\sigma(u)(1-\sigma(u)).
\]
Since $\sigma(u)\in(0,1)$ for all $u$, the product
$\sigma(u)(1-\sigma(u))$ is maximized at $\sigma(u)=1/2$. Therefore,
\[
0 < \sigma'(u)\leq \frac{1}{4}.
\]

By the chain rule,
\begin{equation}
\nabla_\theta \mathcal L_{\mathrm{SGCO}}(\theta)
=
\sigma'(u(\theta))
\left(
\nabla_\theta \bar r_\theta^{\mathrm{low}}
-
\nabla_\theta \bar r_\theta^{\mathrm{high}}
\right).
\end{equation}
Taking norms and applying the triangle inequality gives
\begin{equation}
\|\nabla_\theta \mathcal L_{\mathrm{SGCO}}(\theta)\|
\leq
\frac{1}{4}
\left(
\|\nabla_\theta \bar r_\theta^{\mathrm{low}}\|
+
\|\nabla_\theta \bar r_\theta^{\mathrm{high}}\|
\right).
\end{equation}

Since $\pi_{\mathrm{ref}}$ is fixed, its log-probability has zero gradient with respect to $\theta$. Therefore,
\[
\nabla_\theta \bar r_\theta^{\mathrm{high}}
=
\frac{\beta}{B}
\sum_{i=1}^{B}
\nabla_\theta
\log \pi_\theta(y_i^{\mathrm{high}}).
\]
By Assumption~\ref{asm:boundedG},
\[
\|\nabla_\theta \bar r_\theta^{\mathrm{high}}\|
\leq
\frac{\beta}{B}
\sum_{i=1}^{B}
\|\nabla_\theta\log\pi_\theta(y_i^{\mathrm{high}})\|
\leq
\frac{\beta}{B}
\sum_{i=1}^{B}
G
=
\beta G.
\]
The same argument gives
\[
\|\nabla_\theta \bar r_\theta^{\mathrm{low}}\|
\leq
\beta G.
\]
Substituting these two bounds into the previous inequality yields
\[
\|\nabla_\theta \mathcal L_{\mathrm{SGCO}}(\theta)\|
\leq
\frac{1}{4}
\left(
\beta G+\beta G
\right)
=
\frac{\beta G}{2}.
\]
\end{proof}

\begin{lemma}[$\mathcal O(1/B)$ variance of group-averaged gradients]
\label{lem:variance}
Let
\[
g_i(\theta)
=
\beta
\nabla_\theta
\log\pi_\theta(y_i^{\mathrm{high}}),
\qquad
i=1,\ldots,B,
\]
where $y_i^{\mathrm{high}}$ are drawn i.i.d. from
$\mathrm{Uniform}(\mathcal G_{\mathrm{high}})$. Let
\[
\mu
=
\mathbb E[g_i(\theta)]
\]
and assume
\[
\mathbb E[\|g_i-\mu\|^2]
\leq
\sigma_g^2.
\]
Then
\begin{equation}
\mathbb E
\left[
\left\|
\nabla_\theta \bar r_\theta^{\mathrm{high}}
-
\mu
\right\|^2
\right]
\leq
\frac{\sigma_g^2}{B}.
\end{equation}
The same bound holds for the low-reward group.
\end{lemma}

\begin{proof}
Because $\pi_{\mathrm{ref}}$ is fixed,
\[
\nabla_\theta \bar r_\theta^{\mathrm{high}}
=
\frac{1}{B}
\sum_{i=1}^{B}
g_i(\theta).
\]
Using independence,
\begin{align}
\mathbb E
\left[
\left\|
\frac{1}{B}
\sum_{i=1}^{B}
(g_i-\mu)
\right\|^2
\right]
&=
\frac{1}{B^2}
\sum_{i=1}^{B}
\mathbb E[\|g_i-\mu\|^2]
+
\frac{1}{B^2}
\sum_{i\neq j}
\mathbb E[(g_i-\mu)^\top(g_j-\mu)] \\
&\leq
\frac{1}{B^2}
B\sigma_g^2
=
\frac{\sigma_g^2}{B},
\end{align}
where the cross terms vanish by independence.
\end{proof}

\subsection{Proof of theorem~\ref{thm:sgco_stability}}

\begin{proof}
The first claim of Theorem~\ref{thm:sgco_stability} follows directly from Lemma~\ref{lem:grad_bound}, and the second claim follows from Lemma~\ref{lem:variance}. It remains to prove the SGD descent bound for the fixed surrogate.

By $L$-smoothness,
\[
\mathcal J(\theta_{k+1})
\leq
\mathcal J(\theta_k)
-
\eta
\langle
\nabla\mathcal J(\theta_k),
\hat g_k
\rangle
+
\frac{L\eta^2}{2}
\|\hat g_k\|^2.
\]
Taking conditional expectations and using unbiasedness,
\[
\mathbb E[
\mathcal J(\theta_{k+1})
\mid
\theta_k]
\leq
\mathcal J(\theta_k)
-
\eta
\|\nabla\mathcal J(\theta_k)\|^2
+
\frac{L\eta^2}{2}
\left(
\|\nabla\mathcal J(\theta_k)\|^2
+
\sigma_{\mathcal L}^2
\right).
\]
For $\eta\leq 1/L$,
\[
\mathbb E[
\mathcal J(\theta_{k+1})
\mid
\theta_k]
\leq
\mathcal J(\theta_k)
-
\frac{\eta}{2}
\|\nabla\mathcal J(\theta_k)\|^2
+
\frac{L\eta^2}{2}
\sigma_{\mathcal L}^2.
\]
Taking total expectations, summing over $k=0,\ldots,K-1$, and rearranging gives
\[
\frac{1}{K}
\sum_{k=0}^{K-1}
\mathbb E[
\|\nabla\mathcal J(\theta_k)\|^2
]
\leq
\frac{
2(\mathcal J(\theta_0)-\mathcal J^*)
}{
\eta K
}
+
\eta L\sigma_{\mathcal L}^2.
\]
Since the minimum over $k$ is no larger than the average over $k$, we obtain
\[
\min_{0\leq k<K}
\mathbb E[
\|\nabla \mathcal J(\theta_k)\|^2
]
\leq
\frac{
2(\mathcal J(\theta_0)-\mathcal J^*)
}{
\eta K
}
+
\eta L\sigma_{\mathcal L}^2.
\]
The $\mathcal O(1/B)$ dependence of the group-averaged contribution to the stochastic-gradient variance follows from Lemma~\ref{lem:variance}.
\end{proof}

\paragraph{Scope.}
Theorem~\ref{thm:sgco_stability} bounds optimization for one SGCO update with fixed reward-ranked strata. The empirical experiments evaluate the full TabRSI recursion as the generator and scorer change across rounds.

\section{Experimental setup}
\label{app:exp_setup}

\subsection{Datasets}
\label{app:dataset}

We evaluate on five mixed-type tabular datasets from the UCI Machine Learning Repository: Adult~\citep{adultdataset}, Default~\citep{defaultdataset}, Shoppers~\citep{shoppersdataset}, Magic~\citep{magicdataset}, and Beijing PM$_{2.5}$~\citep{beijingpm25dataset}. Adult, Default, Shoppers, and Magic are classification datasets, while Beijing is a regression dataset. Each dataset is split into training, validation, and test sets with an 8:1:1 ratio, except Adult, for which we use the official test set and split the remaining data into training and validation sets with an 8:1 ratio. Dataset statistics are summarized in Table~\ref{tbl:exp-dataset}.

\begin{table}[htbp] 
    \centering
    \caption{Dataset statistics. ``\# Num'' and ``\# Cat'' denote the number of numerical and categorical columns.}
    \label{tbl:exp-dataset}
    \footnotesize
    \begin{threeparttable}
    \resizebox{\textwidth}{!}{
	\begin{tabular}{lcccccccc}
            \toprule
            \textbf{Dataset} & \# Samples  & \# Num & \# Cat & \# Max Cat & \# Train & \# Validation & \# Test & Task Type  \\
            \midrule 
            \textbf{Adult} & $48,842$ & $6$ & $9$ & $42$ & $28,943$ & $3,618$ & $16,281$ & Binary Classification  \\
            \textbf{Default} & $30,000$ & $14$ & $11$ & $11$ & $24,000$ & $3,000$ & $3,000$ & Binary Classification   \\
            \textbf{Shoppers} & $12,330$ & $10$ & $8$ & $20$ & $9,864$   & $1,233$ & $1,233$ & Binary Classification   \\
            \textbf{Magic} & $19,020$ & $10$ & $1$ & $2$ & $15,216$   & $1,902$ & $1,902$ & Binary Classification   \\
            \textbf{Beijing} & $41,757$ & $7$ & $5$ & $31$ & $33,405$  & $4,176$ & $4,176$ & Regression   \\
		\bottomrule 
		\end{tabular}
    }        
  \end{threeparttable}
\end{table}

Brief descriptions of the datasets are provided below:
\begin{itemize}
    \item \textbf{Adult}: The Adult Census Income dataset contains demographic and employment-related features. The task is to predict whether income exceeds \$50,000.
    \item \textbf{Default}: The Default of Credit Card Clients dataset contains demographic, credit, payment, and bill-statement information. The task is to predict default payment next month.
    \item \textbf{Shoppers}: The Online Shoppers Purchasing Intention dataset contains web-session features. The task is to predict whether a session ends in a purchase.
    \item \textbf{Magic}: The MAGIC Gamma Telescope dataset simulates gamma-particle registration in a ground-based atmospheric Cherenkov telescope. The task is binary classification.
    \item \textbf{Beijing}: The Beijing PM2.5 dataset contains hourly PM2.5 and meteorological records. The task is to predict PM2.5 values.
\end{itemize}

\subsection{Baselines}
\label{app:baseline}

We compare against three families of tabular generators.

\paragraph{Diffusion models.}
\begin{itemize}
    \item \textbf{TabDDPM}~\citep{kotelnikov2023tabddpm}: A denoising diffusion model adapted to mixed-type tabular data.
    \item \textbf{TabSyn}~\citep{zhang2023mixed}: A transformer-based diffusion approach for modeling tabular correlations.
    \item \textbf{TabDiff}~\citep{shi2024tabdiff}: A diffusion model for tabular data with specialized noise schedules and conditioning mechanisms.
\end{itemize}

\paragraph{GAN/VAE models.}
\begin{itemize}
    \item \textbf{CTGAN}~\citep{xu2019modeling}: A generative adversarial network designed for mixed-type tabular data.
    \item \textbf{TVAE}~\citep{xu2019modeling}: A variational autoencoder for tabular data generation.
\end{itemize}

\paragraph{Language-model-based generators.}
\begin{itemize}
    \item \textbf{GReaT}~\citep{borisov2022language}: A tabular generator that serializes rows as text and models tabular synthesis as autoregressive language modeling.
    \item \textbf{GReaT-FT+}: An extended supervised fine-tuning baseline matched to the additional training budget of post-training methods.
\end{itemize}

We adapt DPO~\citep{rafailov2024directpreferenceoptimizationlanguage},
KTO~\citep{ethayarajh2024ktomodelalignmentprospect}, and
NPO~\citep{zhang2024negative} for reward-ranked tabular generation;
their corresponding objectives and sample selection rules are given in
Appendix~\ref{app:alignment_baselines}.

\paragraph{Compared post-training setup.}
The primary loss-variant comparison uses a common supervised GReaT
initialization and fixed SFT reference policy. At each round, a policy
generates a scorer-training pool and an independent alignment pool. The
former is used with real examples to fit a new Random Forest
discriminator; the latter is ranked by the induced reward. SGCO samples
high- and low-reward strata, DPO creates top-versus-bottom
chosen--rejected pairs, KTO labels high- and low-reward samples as
desirable and undesirable, and NPO uses low-reward samples as negative
examples. All variants follow the same round-level feedback protocol,
although their candidate pools diverge as the learned policies evolve.
The default reward is the classifier-uncertainty score defined in
Eq.~\eqref{eq:tabrsi-classifier-reward}; classifier-free DCR experiments
use the same alignment machinery without fitting a discriminator.
Detailed losses and group construction are in
Appendices~\ref{app:loss_variants} and~\ref{app:group_formation}.
We also compare SGCO with GRPO~\citep{shao2024deepseekmath} over five training rounds on the same five datasets. Both methods use statistical rewards to update a supervised tabular language model. GRPO computes normalized advantages within sampled response groups, while SGCO ranks the generated population and optimizes the contrast between high- and low-reward strata. The GRPO comparison and full results are provided in Appendix~\ref{app:grpo-matched}.

\subsection{Evaluation metrics}
\label{app:evaluation}

We evaluate synthetic data along three axes: downstream utility, fidelity, and empirical attack diagnostics. Per-dataset results summarize ten generated-output/evaluation repeats. Across-dataset means and standard deviations use the five dataset-level means.

\paragraph{Utility.}
We measure machine learning efficiency (MLE) by training an XGBoost model~\citep{chen2016xgboost} on synthetic data and evaluating it on held-out real test data. For classification datasets, we report AUC, where higher is better. For Beijing, we report RMSE, where lower is better:
\begin{equation}
    \text{RMSE} = \sqrt{\frac{1}{n} \sum_{i=1}^{n} (y_i - \hat{y}_i)^2}.
\end{equation}

\paragraph{Fidelity.}
We use both low-order and high-order fidelity metrics. Low-order fidelity includes column-density similarity (CDE) and pairwise-correlation similarity (PCC), where higher values indicate better preservation of marginal and pairwise structure. In contextual comparisons with diffusion-based models, we also report Shape and Trend errors following prior work, where lower values are better. High-order fidelity is evaluated using $\alpha$-precision and $\beta$-recall~\citep{alaa2022faithful}: $\alpha$-precision measures the fraction of synthetic samples lying in high-density regions of the real data distribution, while $\beta$-recall measures coverage of the real data manifold.

\paragraph{Distributional distances.}
We report Wasserstein distance~\citep{villani2009optimal}, maximum mean
discrepancy (MMD)~\citep{gretton2012kernel}, and Jensen--Shannon
divergence (JSD)~\citep{lin1991divergence}; smaller values indicate
closer agreement between the real and synthetic distributions. The
$\alpha$-precision and $\beta$-recall measures use the definitions of
\citet{alaa2022faithful} and are interpreted as sample fidelity and
support coverage, respectively; higher values are better.

\paragraph{Classifier two-sample test.}
We report the classifier two-sample test (C2ST)~\citep{lopezpaz2017c2st} using the SDMetrics quality-score convention, in which higher scores are better. It is a quality score and is distinct from the raw discriminator AUC, for which $0.5$ indicates chance-level discrimination.

\paragraph{Empirical attack diagnostics.}
We report distinguishability-attack (DA) AUC and
membership-inference-attack (MIA) AUC~\citep{shokri2017membership}.
DA classifies real versus synthetic records; MIA classifies records
used for fitting the generator versus held-out real records. Both are
reported as raw attack AUCs; values closer to $0.5$ indicate weaker
attacks, with $0.5$ corresponding to chance discrimination. We use a
Random Forest with 100 trees for the DA classifier. These diagnostics measure empirical attack behavior under the tested settings.

\paragraph{Repeated evaluations and uncertainty.}
For each dataset, we report the mean and standard deviation across ten generated-output/evaluation repeats. In the main aggregate tables, mean and population standard deviation are calculated across dataset-level means. The separate GRPO comparison used five training seeds. Appendix~\ref{app:grpo-matched} gives the detailed GRPO setup and results.

\section{Experiment implementation}
\label{app:implementation}

\subsection{Hyperparameters}
\label{app:hyperparams}

The detailed hyperparameters for our tabular language-model post-training experiments are summarized in Table~\ref{tbl:tabrep_hyp}. All iterative post-training methods start from the same supervised fine-tuned GReaT checkpoint and follow the same reward-construction and update protocol; after the first round, the candidate pools can differ as policies evolve. For the default SGCO configuration, we use a Random Forest distinguishability scorer retrained at each iteration, a fixed SFT reference model, and the simplified SGCO objective described in Section~\ref{sec:sgco}. Baseline tabular synthesizers such as CTGAN, TVAE, TabDDPM, TabSyn, and TabDiff are implemented using their official codebases and tuned following the recommendations in their original papers.

\begin{table}[ht]
\centering
\caption{Hyperparameters for tabular language-model post-training methods.}
\label{tbl:tabrep_hyp}
\begin{tabular}{@{}lll@{}}
\toprule
\textbf{Component} & \textbf{Parameter} & \textbf{Value / Search Range} \\
\midrule
\textbf{Base Model} 
& Architecture & DistilGPT-2 (82M) by default \\
& Alternative backbones & GPT-2, GPT-Neo-125M \\
& Tokenization & \texttt{"<col1> is <val1>,..."} \\
& SFT epochs & 100 \\
& SFT batch size & 8 \\
\midrule
\textbf{Shared Hyperparameters} 
& Optimizer & AdamW \\
& Learning rate $\eta$ & $\{5\times10^{-6}, 5\times10^{-7}, 5\times10^{-8}\}$ \\
& Weight decay & 0.01 \\
& Alignment strength $\beta$ & $\{0.1,1,10,100\}$ \\
& Training rounds $T$ & 5 by default \\
& Reference model & Fixed SFT checkpoint \\
\midrule
\textbf{Default Reward Scorer} 
& Model & Random Forest \\
& Number of trees & 100 \\
& Max depth & None \\
& Criterion & Gini impurity \\
& Retrain each round & Yes \\
\midrule
\textbf{SGCO Specific} 
& Group size $B$ & $\{1,4,8,16,32,64\}$ \\
& Default group construction & Top/bottom reward strata \\
& Loss & $\sigma(\bar r_\theta^{\mathrm{low}}-\bar r_\theta^{\mathrm{high}})$ \\
\bottomrule
\end{tabular}
\end{table}

\subsection{Unified iterative post-training pipeline}
\label{app:algorithm_implementation}

Algorithm~\ref{alg:iterative_alignment_framework} gives the complete generator--reward-model alternation. In the classifier-based path, the scorer is trained anew on one pool from the current generator and applied to an independent pool; after alignment, the updated generator supplies the next round's pools. SGCO contrasts the reward-stratified groups, while DPO, NPO and KTO use their respective sample-level supervision. No gradients pass through the scorer. We use $B=4$ as the default group size to make post-training feasible under the available GPU-memory and wall-clock budget; Figure~\ref{fig:batch} tests larger groups up to $B=64$ on Adult. The group-size study covers $B=1$ to $64$ on Adult.

\begin{algorithm}[h]
\caption{TabRSI: alternating generator adaptation and refreshed reward feedback}
\label{alg:iterative_alignment_framework}
\footnotesize
\begin{algorithmic}[1]
\REQUIRE Real dataset $\mathcal{D}_{\mathrm{real}}$, base LM $\pi_{\mathrm{base}}$,
rounds $T$, alignment steps per round $K$, method
$\mathcal{M}\in\{\mathrm{DPO},\mathrm{NPO},\mathrm{KTO},\mathrm{SGCO}\}$,
reward signal $s_t$, learning rate $\eta$, alignment strength $\beta$, group size
$B$ for SGCO
\ENSURE Post-trained policy $\pi_{\theta_T}$

\STATE $\pi_{\theta_0}\gets \mathrm{SFT}(\pi_{\mathrm{base}},\mathcal{D}_{\mathrm{real}})$
\STATE $\pi_{\mathrm{ref}}\gets \pi_{\theta_0}$ \COMMENT{Fixed reference model}

\FOR{$t=1$ to $T$}
    \STATE $\theta^{(0)} \gets \theta_{t-1}$
    \STATE Generate scorer-training pool $\tilde{\mathcal D}^{(t)}_{\mathrm{score}}\sim \pi_{\theta_{t-1}}$
    \STATE Generate alignment pool $\tilde{\mathcal D}^{(t)}_{\mathrm{align}}\sim \pi_{\theta_{t-1}}$

    \STATE \textbf{Reward construction and scoring:}
    \IF{$s_t=s_{\mathrm{cls}}$}
        \STATE Refit scorer $\phi_t$ against the current generator: $\mathcal{D}_{\mathrm{real}}$ vs. $\tilde{\mathcal D}^{(t)}_{\mathrm{score}}$
        \STATE Score each $\tilde{x}\in\tilde{\mathcal D}^{(t)}_{\mathrm{align}}$ by
        \[
        s_{\mathrm{cls}}(\tilde{x})=1-2|0.5-\phi_t(\tilde{x})|
        \]
    \ELSIF{$s_t=s_{\mathrm{dcr}}$}
        \STATE Score each $\tilde{x}\in\tilde{\mathcal D}^{(t)}_{\mathrm{align}}$ by the DCR reward
        \[
        s_{\mathrm{dcr}}(\tilde{x})
        =
        1-
        \frac{
        d(\tilde{x})-d_{\min}
        }{
        d_{\max}-d_{\min}
        },
        \qquad
        d(\tilde{x})
        =
        \min_{x\in\mathcal{D}_{\mathrm{real}}}
        \|\tilde{x}-x\|_2
        \]
    \ELSIF{$s_t=s_{\mathrm{target}}$}
        \STATE Score each $\tilde{x}\in\tilde{\mathcal D}^{(t)}_{\mathrm{align}}$ against a target-domain alignment set
    \ENDIF

    \STATE Rank samples in $\tilde{\mathcal D}^{(t)}_{\mathrm{align}}$ by reward score
    \STATE Construct high-reward stratum $X^+$ and low-reward stratum $X^-$

    \FOR{$k=0$ to $K-1$}
        \STATE Define policy-reference log-ratio under current parameters:
        \[
        r_{\theta^{(k)}}(y)
        =
        \beta
        \log
        \frac{\pi_{\theta^{(k)}}(y)}{\pi_{\mathrm{ref}}(y)}
        \]

        \IF{$\mathcal{M}=\mathrm{SGCO}$}
            \STATE Sample groups
            $\mathcal{B}_{\mathrm{high}}\sim \mathrm{Uniform}(X^+)$ and
            $\mathcal{B}_{\mathrm{low}}\sim \mathrm{Uniform}(X^-)$ with
            $|\mathcal{B}_{\mathrm{high}}|=|\mathcal{B}_{\mathrm{low}}|=B$
            \STATE Compute group-averaged implicit rewards:
            \[
            \bar r_{\theta^{(k)}}^{\mathrm{high}}
            =
            \frac{1}{B}
            \sum_{y\in\mathcal{B}_{\mathrm{high}}}
            r_{\theta^{(k)}}(y),
            \qquad
            \bar r_{\theta^{(k)}}^{\mathrm{low}}
            =
            \frac{1}{B}
            \sum_{y\in\mathcal{B}_{\mathrm{low}}}
            r_{\theta^{(k)}}(y)
            \]
            \STATE $\mathcal L^{(k)}\gets
            \sigma\left(\bar r_{\theta^{(k)}}^{\mathrm{low}}
            -
            \bar r_{\theta^{(k)}}^{\mathrm{high}}\right)$

        \ELSIF{$\mathcal{M}=\mathrm{DPO}$}
            \STATE Sample preferred/rejected mini-batch pairs
            $(y^+,y^-)$ with $y^+\in X^+$ and $y^-\in X^-$
            \STATE $\mathcal L^{(k)}\gets
            -\mathbb{E}_{(y^+,y^-)}
            \log\sigma\left(
            r_{\theta^{(k)}}(y^+)-r_{\theta^{(k)}}(y^-)
            \right)$

        \ELSIF{$\mathcal{M}=\mathrm{NPO}$}
            \STATE Sample negative mini-batch from $X^-$
            \STATE $\mathcal L^{(k)}\gets
            -\mathbb{E}_{y^-\in X^-}
            \log\sigma\left(-r_{\theta^{(k)}}(y^-)\right)$

        \ELSIF{$\mathcal{M}=\mathrm{KTO}$}
            \STATE Sample desirable rows from $X^+$ and undesirable rows from $X^-$
            \STATE Compute KTO desirable/undesirable loss $\mathcal L^{(k)}$
            using $r_{\theta^{(k)}}(y)$
        \ENDIF

        \STATE $\theta^{(k+1)}\gets
        \theta^{(k)}-\eta\nabla_\theta \mathcal L^{(k)}$
    \ENDFOR

    \STATE $\theta_t\gets \theta^{(K)}$ \COMMENT{Next round regenerates samples and refits the scorer}
\ENDFOR

\STATE \textbf{return} $\pi_{\theta_T}$
\end{algorithmic}
\end{algorithm}

\subsection{Computational resources}
\label{app:compute_resources}

All experiments were conducted on a high-performance computing cluster with NVIDIA GeForce RTX 4090 (24GB) and NVIDIA A100 (80GB) GPUs. The software stack used Python 3.8, PyTorch 2.1, CUDA 11.8, Transformers 4.36, scikit-learn 1.3, and XGBoost 1.7.6. Supervised fine-tuning required approximately 3--6 hours per dataset depending on dataset size and hardware. Each post-training iteration required approximately 10--20 minutes, with the autoregressive LM forward/backward pass dominating runtime.

\section{Further self-improvement results}
\label{app:further_results}

\subsection{Per-dataset self-improvement tables}
\label{app:finetuning-baselines}

Tables~\ref{tbl:exp-adult-formatted}--\ref{tbl:exp-beijing-formatted} report detailed per-dataset results for GReaT, GReaT-FT+, adapted alignment baselines, and SGCO variants. These tables complement the averaged main result in Table~\ref{tbl:ablation_loss_variants}. Across datasets, SGCO generally improves the GReaT backbone and provides strong fidelity and empirical attack indistinguishability. 
\begin{table*}[h]
\centering
\caption{
Performance comparison on Adult dataset. Results show methods' performance on the Adult dataset with best results highlighted in \textbf{bold}.}
\label{tbl:exp-adult-formatted}
\begin{threeparttable}
\resizebox{\textwidth}{!}{
\begin{tabular}{lccccccc}
\toprule[1pt]
\textbf{Method} & \textbf{CDE$\uparrow$} & \textbf{PCC$\uparrow$} & \textbf{$\alpha$$\uparrow$} & \textbf{$\beta$$\uparrow$} & \textbf{C2ST$\uparrow$} & \textbf{DA $\to 0.5$} & \textbf{MLE$\uparrow$} \\
\midrule

\multicolumn{8}{l}{\textit{Baselines}} \\
GReaT (Original) & $92.55{\scriptstyle\pm0.02}$ & $87.82{\scriptstyle\pm1.54}$ & $68.28{\scriptstyle\pm0.15}$ & $51.46{\scriptstyle\pm0.23}$ & $62.14{\scriptstyle\pm0.12}$ & $0.7343{\scriptstyle\pm0.0028}$ & $0.9121{\scriptstyle\pm0.0017}$ \\
GReaT-FT+ & $92.52{\scriptstyle\pm0.03}$ & $87.47{\scriptstyle\pm4.04}$ & $68.85{\scriptstyle\pm0.20}$ & $51.09{\scriptstyle\pm0.36}$ & $62.58{\scriptstyle\pm0.72}$ & $0.7348{\scriptstyle\pm0.0057}$ & $0.9091{\scriptstyle\pm0.0018}$ \\
\midrule

\multicolumn{8}{l}{\textit{DPO variants}} \\
TabDPO (base) & $98.80{\scriptstyle\pm0.07}$ & $88.16{\scriptstyle\pm4.04}$ & $98.75{\scriptstyle\pm0.24}$ & $52.35{\scriptstyle\pm0.29}$ & $94.85{\scriptstyle\pm0.22}$ & $0.5176{\scriptstyle\pm0.0044}$ & $0.9203{\scriptstyle\pm0.0018}$ \\
+ KL penalty & $98.94{\scriptstyle\pm0.03}$ & $89.39{\scriptstyle\pm2.90}$ & $98.63{\scriptstyle\pm0.31}$ & $50.47{\scriptstyle\pm0.42}$ & $95.87{\scriptstyle\pm0.46}$ & $0.5133{\scriptstyle\pm0.0040}$ & $0.9195{\scriptstyle\pm0.0016}$ \\
+ Gradient diff. & $98.83{\scriptstyle\pm0.05}$ & $\bf{90.86{\scriptstyle\pm2.38}}$ & $98.07{\scriptstyle\pm0.23}$ & $50.42{\scriptstyle\pm0.22}$ & $95.53{\scriptstyle\pm0.43}$ & $0.5190{\scriptstyle\pm0.0054}$ & $0.9195{\scriptstyle\pm0.0012}$ \\
\midrule

\multicolumn{8}{l}{\textit{NPO variants}} \\
TabNPO (base) & $94.03{\scriptstyle\pm0.03}$ & $84.22{\scriptstyle\pm4.26}$ & $75.72{\scriptstyle\pm0.19}$ & $52.69{\scriptstyle\pm0.21}$ & $71.01{\scriptstyle\pm0.85}$ & $0.6601{\scriptstyle\pm0.0071}$ & $0.9171{\scriptstyle\pm0.0022}$ \\
+ KL penalty & $94.02{\scriptstyle\pm0.05}$ & $83.29{\scriptstyle\pm3.62}$ & $75.59{\scriptstyle\pm0.37}$ & $52.58{\scriptstyle\pm0.13}$ & $71.14{\scriptstyle\pm0.81}$ & $0.6564{\scriptstyle\pm0.0028}$ & $0.9121{\scriptstyle\pm0.0045}$ \\
+ Gradient diff. & $94.19{\scriptstyle\pm0.05}$ & $83.75{\scriptstyle\pm3.07}$ & $76.40{\scriptstyle\pm0.27}$ & $\bf{52.72{\scriptstyle\pm0.15}}$ & $71.55{\scriptstyle\pm0.56}$ & $0.6534{\scriptstyle\pm0.0038}$ & $0.9172{\scriptstyle\pm0.0023}$ \\ 
\midrule

\multicolumn{8}{l}{\textit{KTO variants}} \\
TabKTO (base) & $93.53{\scriptstyle\pm0.06}$ & $83.80{\scriptstyle\pm2.33}$ & $73.64{\scriptstyle\pm0.26}$ & $52.70{\scriptstyle\pm0.21}$ & $69.30{\scriptstyle\pm0.66}$ & $0.6668{\scriptstyle\pm0.0039}$ & $0.9154{\scriptstyle\pm0.0031}$ \\
+ Logsigmoid & $93.51{\scriptstyle\pm0.08}$ & $82.37{\scriptstyle\pm3.42}$ & $73.43{\scriptstyle\pm0.25}$ & $52.47{\scriptstyle\pm0.27}$ & $70.04{\scriptstyle\pm0.89}$ & $0.6695{\scriptstyle\pm0.0025}$ & $0.9147{\scriptstyle\pm0.0044}$ \\
+ Logs. + Grad. diff & $93.52{\scriptstyle\pm0.06}$ & $82.30{\scriptstyle\pm3.68}$ & $73.47{\scriptstyle\pm0.30}$ & $52.66{\scriptstyle\pm0.35}$ & $69.51{\scriptstyle\pm0.57}$ & $0.6726{\scriptstyle\pm0.0030}$ & $0.9146{\scriptstyle\pm0.0035}$ \\
\midrule

\multicolumn{8}{l}{\textit{SGCO variants (Ours)}} \\
SGCO (base) & $\bf{99.13{\scriptstyle\pm0.04}}$ & $88.13{\scriptstyle\pm5.14}$ & $\bf{99.44{\scriptstyle\pm0.26}}$ & $52.14{\scriptstyle\pm0.33}$ & $\bf{96.27{\scriptstyle\pm0.28}}$ & $\bf{0.4997{\scriptstyle\pm0.0025}}$ & $\bf{0.9214{\scriptstyle\pm0.0006}}$ \\
+ Logsigmoid & $98.94{\scriptstyle\pm0.05}$ & $90.62{\scriptstyle\pm3.26}$ & $98.72{\scriptstyle\pm0.27}$ & $50.69{\scriptstyle\pm0.25}$ & $95.64{\scriptstyle\pm0.35}$ & $0.5107{\scriptstyle\pm0.0037}$ & $0.9215{\scriptstyle\pm0.0014}$ \\
+ Logs. + Grad. diff & $98.79{\scriptstyle\pm0.06}$ & $88.91{\scriptstyle\pm3.76}$ & $98.83{\scriptstyle\pm0.23}$ & $51.04{\scriptstyle\pm0.42}$ & $94.82{\scriptstyle\pm0.41}$ & $0.5136{\scriptstyle\pm0.0033}$ & $0.9198{\scriptstyle\pm0.0032}$ \\

\bottomrule[1pt]
\end{tabular}}
\end{threeparttable}
\end{table*}

\begin{table*}[h]
\centering
\caption{
Performance comparison on Default dataset. Results show methods' performance on the Default dataset with best results highlighted in \textbf{bold}.}
\label{tbl:exp-default-formatted}
\begin{threeparttable}
\resizebox{\textwidth}{!}{
\begin{tabular}{lccccccc}
\toprule[1pt]
\textbf{Method} & \textbf{CDE$\uparrow$} & \textbf{PCC$\uparrow$} & \textbf{$\alpha$$\uparrow$} & \textbf{$\beta$$\uparrow$} & \textbf{C2ST$\uparrow$} & \textbf{DA $\to 0.5$} & \textbf{MLE$\uparrow$} \\
\midrule

\multicolumn{8}{l}{\textit{Baselines}} \\
GReaT (Original) & $80.17{\scriptstyle\pm0.06}$ & $30.80{\scriptstyle\pm0.65}$ & $86.32{\scriptstyle\pm2.41}$ & $41.73{\scriptstyle\pm0.45}$ & $11.31{\scriptstyle\pm0.19}$ & $0.8634{\scriptstyle\pm0.0028}$ & $0.7603{\scriptstyle\pm0.0034}$ \\
GReaT-FT+ & $80.13{\scriptstyle\pm0.08}$ & $30.75{\scriptstyle\pm0.18}$ & $86.86{\scriptstyle\pm0.14}$ & $41.69{\scriptstyle\pm0.33}$ & $11.60{\scriptstyle\pm0.67}$ & $0.8669{\scriptstyle\pm0.0057}$ & $0.7619{\scriptstyle\pm0.0014}$ \\
\midrule

\multicolumn{8}{l}{\textit{DPO variants}} \\
TabDPO (base) & $\bf{94.95{\scriptstyle\pm0.04}}$ & $30.53{\scriptstyle\pm0.38}$ & $94.72{\scriptstyle\pm0.11}$ & $45.51{\scriptstyle\pm0.29}$ & $13.04{\scriptstyle\pm0.12}$ & $0.7751{\scriptstyle\pm0.0022}$ & $0.7767{\scriptstyle\pm0.0044}$ \\
+ KL penalty & $94.39{\scriptstyle\pm0.06}$ & $30.11{\scriptstyle\pm0.36}$ & $94.47{\scriptstyle\pm0.17}$ & $45.08{\scriptstyle\pm0.18}$ & $13.04{\scriptstyle\pm0.07}$ & $0.7806{\scriptstyle\pm0.0026}$ & $0.7781{\scriptstyle\pm0.0054}$ \\
+ Gradient diff. & $94.77{\scriptstyle\pm0.09}$ & $29.77{\scriptstyle\pm0.35}$ & $\bf{96.46{\scriptstyle\pm0.17}}$ & $45.71{\scriptstyle\pm0.32}$ & $13.31{\scriptstyle\pm0.18}$ & $0.7654{\scriptstyle\pm0.0025}$ & $0.7814{\scriptstyle\pm0.0034}$ \\
\midrule

\multicolumn{8}{l}{\textit{NPO variants}} \\
TabNPO (base) & $81.41{\scriptstyle\pm0.05}$ & $30.75{\scriptstyle\pm0.28}$ & $87.49{\scriptstyle\pm0.42}$ & $42.93{\scriptstyle\pm0.36}$ & $11.66{\scriptstyle\pm0.36}$ & $0.8351{\scriptstyle\pm0.0055}$ & $0.7781{\scriptstyle\pm0.0065}$ \\
+ KL penalty & $82.01{\scriptstyle\pm0.06}$ & $30.07{\scriptstyle\pm0.35}$ & $87.23{\scriptstyle\pm0.29}$ & $42.14{\scriptstyle\pm0.41}$ & $11.69{\scriptstyle\pm0.45}$ & $0.8393{\scriptstyle\pm0.0037}$ & $0.7732{\scriptstyle\pm0.0062}$ \\
+ Gradient diff. & $82.36{\scriptstyle\pm0.03}$ & $30.22{\scriptstyle\pm0.66}$ & $87.95{\scriptstyle\pm0.14}$ & $42.62{\scriptstyle\pm0.28}$ & $11.86{\scriptstyle\pm0.38}$ & $0.8349{\scriptstyle\pm0.0031}$ & $0.7756{\scriptstyle\pm0.0045}$ \\ 
\midrule

\multicolumn{8}{l}{\textit{KTO variants}} \\
TabKTO (base) & $81.30{\scriptstyle\pm0.08}$ & $30.18{\scriptstyle\pm0.59}$ & $87.44{\scriptstyle\pm0.32}$ & $42.91{\scriptstyle\pm0.22}$ & $11.62{\scriptstyle\pm0.43}$ & $0.8410{\scriptstyle\pm0.0030}$ & $0.7727{\scriptstyle\pm0.0042}$ \\
+ Logsigmoid & $81.12{\scriptstyle\pm0.08}$ & $30.42{\scriptstyle\pm0.32}$ & $87.06{\scriptstyle\pm0.21}$ & $42.61{\scriptstyle\pm0.47}$ & $11.63{\scriptstyle\pm0.22}$ & $0.8411{\scriptstyle\pm0.0018}$ & $0.7775{\scriptstyle\pm0.0053}$ \\
+ Logs. + Grad. diff & $81.09{\scriptstyle\pm0.03}$ & $30.18{\scriptstyle\pm0.14}$ & $87.29{\scriptstyle\pm0.26}$ & $42.51{\scriptstyle\pm0.31}$ & $11.68{\scriptstyle\pm0.46}$ & $0.8392{\scriptstyle\pm0.0036}$ & $\bf{0.7840{\scriptstyle\pm0.0022}}$ \\
\midrule

\multicolumn{8}{l}{\textit{SGCO variants (Ours)}} \\
SGCO (base) & $94.64{\scriptstyle\pm0.03}$ & $30.62{\scriptstyle\pm0.45}$ & $94.75{\scriptstyle\pm0.28}$ & $45.45{\scriptstyle\pm0.16}$ & $13.36{\scriptstyle\pm0.45}$ & $0.7823{\scriptstyle\pm0.0035}$ & $0.7752{\scriptstyle\pm0.0006}$ \\
+ Logsigmoid & $94.00{\scriptstyle\pm0.03}$ & $30.70{\scriptstyle\pm0.46}$ & $96.25{\scriptstyle\pm0.13}$ & $45.15{\scriptstyle\pm0.55}$ & $12.91{\scriptstyle\pm0.25}$ & $0.7942{\scriptstyle\pm0.0026}$ & $0.7774{\scriptstyle\pm0.0022}$ \\
+ Logs. + Grad. diff & $94.57{\scriptstyle\pm0.08}$ & $\bf{30.91{\scriptstyle\pm0.56}}$ & $95.05{\scriptstyle\pm0.21}$ & $\bf{46.47{\scriptstyle\pm0.35}}$ & $\bf{13.46{\scriptstyle\pm0.38}}$ & $\bf{0.7608{\scriptstyle\pm0.0018}}$ & $0.7822{\scriptstyle\pm0.0042}$ \\

\bottomrule[1pt]
\end{tabular}}
\end{threeparttable}
\end{table*}

\begin{table*}[h]
\centering
\caption{
Performance comparison on Shoppers dataset. Results show methods' performance on the Shoppers dataset with best results highlighted in \textbf{bold}.}
\label{tbl:exp-shoppers-formatted}
\begin{threeparttable}
\resizebox{\textwidth}{!}{
\begin{tabular}{lccccccc}
\toprule[1pt]
\textbf{Method} & \textbf{CDE$\uparrow$} & \textbf{PCC$\uparrow$} & \textbf{$\alpha$$\uparrow$} & \textbf{$\beta$$\uparrow$} & \textbf{C2ST$\uparrow$} & \textbf{DA $\to 0.5$} & \textbf{MLE$\uparrow$} \\
\midrule

\multicolumn{8}{l}{\textit{Baselines}} \\
GReaT (Original) & $85.58{\scriptstyle\pm0.03}$ & $55.05{\scriptstyle\pm0.24}$ & $79.10{\scriptstyle\pm0.45}$ & $45.29{\scriptstyle\pm0.52}$ & $14.26{\scriptstyle\pm0.34}$ & $0.8327{\scriptstyle\pm0.0055}$ & $0.9012{\scriptstyle\pm0.0068}$ \\
GReaT-FT+ & $86.10{\scriptstyle\pm0.06}$ & $54.47{\scriptstyle\pm0.24}$ & $90.67{\scriptstyle\pm0.35}$ & $43.56{\scriptstyle\pm0.46}$ & $14.17{\scriptstyle\pm0.54}$ & $0.8446{\scriptstyle\pm0.0033}$ & $0.8867{\scriptstyle\pm0.0038}$ \\
\midrule

\multicolumn{8}{l}{\textit{DPO variants}} \\
TabDPO (base) & $88.25{\scriptstyle\pm0.04}$ & $54.55{\scriptstyle\pm0.32}$ & $94.68{\scriptstyle\pm0.52}$ & $47.59{\scriptstyle\pm0.72}$ & $16.04{\scriptstyle\pm1.04}$ & $0.7680{\scriptstyle\pm0.0756}$ & $0.9063{\scriptstyle\pm0.0045}$ \\
+ KL penalty & $87.64{\scriptstyle\pm0.09}$ & $53.96{\scriptstyle\pm0.21}$ & $97.91{\scriptstyle\pm0.51}$ & $47.30{\scriptstyle\pm1.07}$ & $15.57{\scriptstyle\pm1.22}$ & $0.7604{\scriptstyle\pm0.0799}$ & $0.9045{\scriptstyle\pm0.0035}$ \\
+ Gradient diff. & $88.58{\scriptstyle\pm0.04}$ & $54.09{\scriptstyle\pm0.95}$ & $95.84{\scriptstyle\pm0.75}$ & $48.44{\scriptstyle\pm0.88}$ & $16.23{\scriptstyle\pm1.04}$ & $0.7536{\scriptstyle\pm0.0822}$ & $0.9084{\scriptstyle\pm0.0056}$ \\
\midrule

\multicolumn{8}{l}{\textit{NPO variants}} \\
TabNPO (base) & $87.83{\scriptstyle\pm0.04}$ & $54.90{\scriptstyle\pm0.68}$ & $76.99{\scriptstyle\pm0.32}$ & $46.85{\scriptstyle\pm0.25}$ & $16.01{\scriptstyle\pm0.38}$ & $0.8023{\scriptstyle\pm0.0046}$ & $0.9066{\scriptstyle\pm0.0044}$ \\
+ KL penalty & $87.29{\scriptstyle\pm0.03}$ & $52.12{\scriptstyle\pm0.10}$ & $77.75{\scriptstyle\pm0.15}$ & $47.45{\scriptstyle\pm0.42}$ & $15.77{\scriptstyle\pm0.22}$ & $0.8016{\scriptstyle\pm0.0058}$ & $0.9095{\scriptstyle\pm0.0062}$ \\
+ Gradient diff. & $88.34{\scriptstyle\pm0.08}$ & $53.25{\scriptstyle\pm0.03}$ & $76.08{\scriptstyle\pm0.85}$ & $48.91{\scriptstyle\pm0.73}$ & $16.42{\scriptstyle\pm0.47}$ & $0.7991{\scriptstyle\pm0.0063}$ & $0.9005{\scriptstyle\pm0.0052}$ \\ 
\midrule

\multicolumn{8}{l}{\textit{KTO variants}} \\
TabKTO (base) & $86.78{\scriptstyle\pm0.05}$ & $53.90{\scriptstyle\pm0.46}$ & $78.37{\scriptstyle\pm0.25}$ & $47.42{\scriptstyle\pm0.31}$ & $15.20{\scriptstyle\pm0.32}$ & $0.8068{\scriptstyle\pm0.0045}$ & $0.9028{\scriptstyle\pm0.0027}$ \\
+ Logsigmoid & $86.82{\scriptstyle\pm0.03}$ & $53.49{\scriptstyle\pm0.18}$ & $78.49{\scriptstyle\pm0.31}$ & $47.24{\scriptstyle\pm0.27}$ & $15.16{\scriptstyle\pm0.23}$ & $0.8040{\scriptstyle\pm0.0045}$ & $0.9065{\scriptstyle\pm0.0033}$ \\
+ Logs. + Grad. diff & $86.74{\scriptstyle\pm0.07}$ & $53.99{\scriptstyle\pm0.14}$ & $77.71{\scriptstyle\pm0.43}$ & $47.22{\scriptstyle\pm0.28}$ & $15.04{\scriptstyle\pm0.42}$ & $0.8087{\scriptstyle\pm0.0024}$ & $0.9104{\scriptstyle\pm0.0037}$ \\
\midrule

\multicolumn{8}{l}{\textit{SGCO variants (Ours)}} \\
SGCO (base) & $\bf{90.37{\scriptstyle\pm0.04}}$ & $54.13{\scriptstyle\pm0.09}$ & $94.71{\scriptstyle\pm0.26}$ & $\bf{50.22{\scriptstyle\pm0.48}}$ & $\bf{17.13{\scriptstyle\pm0.19}}$ & $\bf{0.7274{\scriptstyle\pm0.0015}}$ & $0.9078{\scriptstyle\pm0.0036}$ \\
+ Logsigmoid & $87.64{\scriptstyle\pm0.04}$ & $\bf{55.21{\scriptstyle\pm0.11}}$ & $\bf{99.61{\scriptstyle\pm0.46}}$ & $47.00{\scriptstyle\pm0.42}$ & $15.65{\scriptstyle\pm0.53}$ & $0.7391{\scriptstyle\pm0.0018}$ & $\bf{0.9110{\scriptstyle\pm0.0095}}$ \\
+ Logs. + Grad. diff & $87.98{\scriptstyle\pm0.08}$ & $53.65{\scriptstyle\pm0.06}$ & $96.71{\scriptstyle\pm0.18}$ & $48.90{\scriptstyle\pm0.36}$ & $15.87{\scriptstyle\pm0.28}$ & $0.7463{\scriptstyle\pm0.0042}$ & $0.9045{\scriptstyle\pm0.0021}$ \\

\bottomrule[1pt]
\end{tabular}}
\end{threeparttable}
\end{table*}

\begin{table*}[h!]
\centering
\caption{
Performance comparison on Magic dataset. Results show methods' performance on the Magic dataset with best results highlighted in \textbf{bold}.}
\label{tbl:exp-magic-formatted}
\begin{threeparttable}
\resizebox{\textwidth}{!}{
\begin{tabular}{lccccccc}
\toprule[1pt]
\textbf{Method} & \textbf{CDE$\uparrow$} & \textbf{PCC$\uparrow$} & \textbf{$\alpha$$\uparrow$} & \textbf{$\beta$$\uparrow$} & \textbf{C2ST$\uparrow$} & \textbf{DA $\to 0.5$} & \textbf{MLE$\uparrow$} \\
\midrule

\multicolumn{8}{l}{\textit{Baselines}} \\
GReaT (Original) & $83.97{\scriptstyle\pm0.42}$ & $89.51{\scriptstyle\pm0.54}$ & $87.26{\scriptstyle\pm0.61}$ & $39.30{\scriptstyle\pm0.47}$ & $45.35{\scriptstyle\pm0.68}$ & $0.8768{\scriptstyle\pm0.0038}$ & $0.9008{\scriptstyle\pm0.0034}$ \\
GReaT-FT+ & $85.19{\scriptstyle\pm0.11}$ & $89.44{\scriptstyle\pm0.30}$ & $89.44{\scriptstyle\pm0.40}$ & $39.96{\scriptstyle\pm0.25}$ & $50.66{\scriptstyle\pm0.63}$ & $0.8599{\scriptstyle\pm0.0025}$ & $0.9009{\scriptstyle\pm0.0058}$ \\
\midrule

\multicolumn{8}{l}{\textit{DPO variants}} \\
TabDPO (base) & $95.14{\scriptstyle\pm0.10}$ & $89.36{\scriptstyle\pm1.82}$ & $93.74{\scriptstyle\pm0.12}$ & $47.00{\scriptstyle\pm0.29}$ & $\bf{85.50{\scriptstyle\pm0.49}}$ & $0.7099{\scriptstyle\pm0.0047}$ & $0.9087{\scriptstyle\pm0.0068}$ \\
+ KL penalty & $94.88{\scriptstyle\pm0.09}$ & $87.20{\scriptstyle\pm1.60}$ & $93.04{\scriptstyle\pm0.26}$ & $46.85{\scriptstyle\pm0.46}$ & $84.39{\scriptstyle\pm0.65}$ & $0.7094{\scriptstyle\pm0.0046}$ & $0.9091{\scriptstyle\pm0.0059}$ \\
+ Gradient diff. & $94.99{\scriptstyle\pm0.14}$ & $89.02{\scriptstyle\pm1.78}$ & $94.15{\scriptstyle\pm0.25}$ & $46.87{\scriptstyle\pm0.49}$ & $84.10{\scriptstyle\pm0.95}$ & $0.7086{\scriptstyle\pm0.0051}$ & $0.9089{\scriptstyle\pm0.0031}$ \\
\midrule

\multicolumn{8}{l}{\textit{NPO variants}} \\
TabNPO (base) & $88.65{\scriptstyle\pm0.08}$ & $87.91{\scriptstyle\pm1.91}$ & $90.27{\scriptstyle\pm0.62}$ & $45.09{\scriptstyle\pm0.85}$ & $60.64{\scriptstyle\pm0.37}$ & $0.7900{\scriptstyle\pm0.0081}$ & $0.9084{\scriptstyle\pm0.0025}$ \\
+ KL penalty & $89.26{\scriptstyle\pm0.11}$ & $79.06{\scriptstyle\pm2.32}$ & $88.05{\scriptstyle\pm0.48}$ & $45.32{\scriptstyle\pm0.53}$ & $68.59{\scriptstyle\pm0.45}$ & $0.7765{\scriptstyle\pm0.0033}$ & $0.9066{\scriptstyle\pm0.0051}$ \\
+ Gradient diff. & $89.01{\scriptstyle\pm0.06}$ & $80.51{\scriptstyle\pm2.34}$ & $87.24{\scriptstyle\pm0.41}$ & $44.28{\scriptstyle\pm0.34}$ & $62.94{\scriptstyle\pm0.89}$ & $0.7716{\scriptstyle\pm0.0067}$ & $0.9069{\scriptstyle\pm0.0012}$ \\ 
\midrule

\multicolumn{8}{l}{\textit{KTO variants}} \\
TabKTO (base) & $86.26{\scriptstyle\pm0.09}$ & $87.75{\scriptstyle\pm1.26}$ & $87.72{\scriptstyle\pm0.25}$ & $42.83{\scriptstyle\pm0.32}$ & $52.90{\scriptstyle\pm0.24}$ & $0.8087{\scriptstyle\pm0.0037}$ & $0.9051{\scriptstyle\pm0.0042}$ \\
+ Logsigmoid & $86.39{\scriptstyle\pm0.08}$ & $86.11{\scriptstyle\pm0.35}$ & $87.24{\scriptstyle\pm0.31}$ & $42.51{\scriptstyle\pm0.17}$ & $54.79{\scriptstyle\pm0.37}$ & $0.8067{\scriptstyle\pm0.0044}$ & $0.9053{\scriptstyle\pm0.0052}$ \\
+ Logs. + Grad. diff & $86.33{\scriptstyle\pm0.05}$ & $87.01{\scriptstyle\pm0.32}$ & $86.86{\scriptstyle\pm0.29}$ & $42.85{\scriptstyle\pm0.44}$ & $54.33{\scriptstyle\pm0.23}$ & $0.8074{\scriptstyle\pm0.0041}$ & $0.9057{\scriptstyle\pm0.0037}$ \\
\midrule

\multicolumn{8}{l}{\textit{SGCO variants (Ours)}} \\
SGCO (base) & $\bf{95.58{\scriptstyle\pm0.25}}$ & $89.76{\scriptstyle\pm1.54}$ & $\bf{95.26{\scriptstyle\pm0.47}}$ & $\bf{47.48{\scriptstyle\pm0.53}}$ & $83.25{\scriptstyle\pm0.68}$ & $\bf{0.7034{\scriptstyle\pm0.0055}}$ & $0.9087{\scriptstyle\pm0.0046}$ \\
+ Logsigmoid & $94.73{\scriptstyle\pm0.08}$ & $\bf{92.12{\scriptstyle\pm2.41}}$ & $95.22{\scriptstyle\pm0.37}$ & $47.06{\scriptstyle\pm0.85}$ & $82.94{\scriptstyle\pm0.55}$ & $0.7085{\scriptstyle\pm0.0057}$ & $0.9068{\scriptstyle\pm0.0044}$ \\
+ Logs. + Grad. diff & $95.05{\scriptstyle\pm0.10}$ & $84.36{\scriptstyle\pm3.25}$ & $94.76{\scriptstyle\pm0.43}$ & $47.19{\scriptstyle\pm0.22}$ & $83.67{\scriptstyle\pm0.31}$ & $0.7178{\scriptstyle\pm0.0078}$ & $\bf{0.9093{\scriptstyle\pm0.0052}}$ \\

\bottomrule[1pt]
\end{tabular}}
\end{threeparttable}
\end{table*}

\begin{table*}[h]
\centering
\caption{
Performance comparison on Beijing dataset. Results show methods' performance on the Beijing dataset with best results highlighted in \textbf{bold}.}
\label{tbl:exp-beijing-formatted}
\begin{threeparttable}
\resizebox{\textwidth}{!}{
\begin{tabular}{lccccccc}
\toprule[1pt]
\textbf{Method} & \textbf{CDE$\uparrow$} & \textbf{PCC$\uparrow$} & \textbf{$\alpha$$\uparrow$} & \textbf{$\beta$$\uparrow$} & \textbf{C2ST$\uparrow$} & \textbf{DA $\to 0.5$} & \textbf{MLE$\downarrow$} \\
\midrule

\multicolumn{8}{l}{\textit{Baselines}} \\
GReaT (Original) & $92.05{\scriptstyle\pm0.06}$ & $40.61{\scriptstyle\pm0.54}$ & $96.05{\scriptstyle\pm0.55}$ & $57.71{\scriptstyle\pm0.43}$ & $30.71{\scriptstyle\pm0.52}$ & $0.7798{\scriptstyle\pm0.0036}$ & $0.6628{\scriptstyle\pm0.0047}$ \\
GReaT-FT+ & $93.18{\scriptstyle\pm0.09}$ & $40.61{\scriptstyle\pm0.85}$ & $94.27{\scriptstyle\pm0.30}$ & $55.31{\scriptstyle\pm0.76}$ & $30.82{\scriptstyle\pm0.14}$ & $0.7695{\scriptstyle\pm0.0035}$ & $0.6541{\scriptstyle\pm0.0038}$ \\
\midrule

\multicolumn{8}{l}{\textit{DPO variants}} \\
TabDPO (base) & $97.58{\scriptstyle\pm0.10}$ & $39.11{\scriptstyle\pm1.09}$ & $98.14{\scriptstyle\pm0.18}$ & $57.88{\scriptstyle\pm0.23}$ & $31.56{\scriptstyle\pm0.05}$ & $0.6944{\scriptstyle\pm0.0034}$ & $0.5961{\scriptstyle\pm0.0039}$ \\
+ KL penalty & $97.61{\scriptstyle\pm0.03}$ & $38.86{\scriptstyle\pm0.57}$ & $97.43{\scriptstyle\pm0.25}$ & $58.25{\scriptstyle\pm0.29}$ & $31.40{\scriptstyle\pm0.04}$ & $0.7043{\scriptstyle\pm0.0030}$ & $0.6085{\scriptstyle\pm0.0042}$ \\
+ Gradient diff. & $97.82{\scriptstyle\pm0.07}$ & $38.64{\scriptstyle\pm0.97}$ & $98.23{\scriptstyle\pm0.21}$ & $57.80{\scriptstyle\pm0.17}$ & $31.42{\scriptstyle\pm0.04}$ & $0.7093{\scriptstyle\pm0.0029}$ & $0.5927{\scriptstyle\pm0.0012}$ \\
\midrule

\multicolumn{8}{l}{\textit{NPO variants}} \\
TabNPO (base) & $91.65{\scriptstyle\pm0.04}$ & $39.13{\scriptstyle\pm0.98}$ & $98.54{\scriptstyle\pm0.25}$ & $59.04{\scriptstyle\pm0.72}$ & $29.90{\scriptstyle\pm0.31}$ & $0.7504{\scriptstyle\pm0.0035}$ & $0.6183{\scriptstyle\pm0.0085}$ \\
+ KL penalty & $91.30{\scriptstyle\pm0.04}$ & $38.98{\scriptstyle\pm0.05}$ & $98.38{\scriptstyle\pm0.16}$ & $58.35{\scriptstyle\pm0.33}$ & $29.99{\scriptstyle\pm0.24}$ & $0.7519{\scriptstyle\pm0.0052}$ & $0.6232{\scriptstyle\pm0.0022}$ \\
+ Gradient diff. & $91.41{\scriptstyle\pm0.08}$ & $38.70{\scriptstyle\pm0.06}$ & $98.43{\scriptstyle\pm0.45}$ & $\bf{59.09{\scriptstyle\pm0.37}}$ & $30.20{\scriptstyle\pm0.53}$ & $0.7504{\scriptstyle\pm0.0042}$ & $0.6118{\scriptstyle\pm0.0033}$ \\ 
\midrule

\multicolumn{8}{l}{\textit{KTO variants}} \\
TabKTO (base) & $91.58{\scriptstyle\pm0.06}$ & $38.95{\scriptstyle\pm0.31}$ & $98.52{\scriptstyle\pm0.22}$ & $59.03{\scriptstyle\pm0.34}$ & $30.02{\scriptstyle\pm0.42}$ & $0.7505{\scriptstyle\pm0.0032}$ & $0.6379{\scriptstyle\pm0.0037}$ \\
+ Logsigmoid & $91.59{\scriptstyle\pm0.07}$ & $38.79{\scriptstyle\pm0.35}$ & $98.42{\scriptstyle\pm0.21}$ & $58.83{\scriptstyle\pm0.17}$ & $30.01{\scriptstyle\pm0.32}$ & $0.7495{\scriptstyle\pm0.0065}$ & $0.6428{\scriptstyle\pm0.0087}$ \\
+ Logs. + Grad. diff & $91.67{\scriptstyle\pm0.08}$ & $38.86{\scriptstyle\pm0.27}$ & $98.43{\scriptstyle\pm0.28}$ & $58.87{\scriptstyle\pm0.47}$ & $29.94{\scriptstyle\pm0.29}$ & $0.7506{\scriptstyle\pm0.0024}$ & $0.6315{\scriptstyle\pm0.0025}$ \\
\midrule

\multicolumn{8}{l}{\textit{SGCO variants (Ours)}} \\
SGCO (base) & $97.61{\scriptstyle\pm0.04}$ & $\bf{40.85{\scriptstyle\pm0.24}}$ & $\bf{99.51{\scriptstyle\pm0.26}}$ & $59.01{\scriptstyle\pm0.33}$ & $\bf{31.67{\scriptstyle\pm0.28}}$ & $\bf{0.6942{\scriptstyle\pm0.0035}}$ & $\bf{0.5740{\scriptstyle\pm0.0046}}$ \\
+ Logsigmoid & $\bf{98.98{\scriptstyle\pm0.05}}$ & $39.85{\scriptstyle\pm0.23}$ & $98.72{\scriptstyle\pm0.27}$ & $58.69{\scriptstyle\pm0.25}$ & $30.24{\scriptstyle\pm0.35}$ & $0.6977{\scriptstyle\pm0.0017}$ & $0.5832{\scriptstyle\pm0.0034}$ \\
+ Logs. + Grad. diff & $98.79{\scriptstyle\pm0.06}$ & $39.91{\scriptstyle\pm0.16}$ & $98.83{\scriptstyle\pm0.23}$ & $59.04{\scriptstyle\pm0.42}$ & $31.12{\scriptstyle\pm0.41}$ & $0.6971{\scriptstyle\pm0.0033}$ & $0.5798{\scriptstyle\pm0.0022}$ \\

\bottomrule[1pt]
\end{tabular}}
\end{threeparttable}
\end{table*}

\subsection{Detailed reward-choice evaluation}
\label{app:reward_choice_full}
Table~\ref{tbl:reward_comparison_full} gives the five-dataset results underlying the across-dataset summary in Table~\ref{tbl:reward_comparison}. Within-dataset deviations are computed from repeated generation and evaluation.
\begin{table*}[t]
\centering
\caption{\textbf{Detailed per-dataset reward comparison.} We compare separated classifier and classifier-free DCR rewards under the same SGCO setting. The gray data leakage rows omit scorer-sample separation, excluded from bolding. For MLE, Beijing reports RMSE. Best results are in \textbf{bold}.}
\label{tbl:reward_comparison_full}
\resizebox{\textwidth}{!}{
\begin{tabular}{llccccc}
\toprule
\textbf{Dataset} & \textbf{Reward} & \textbf{Wasserstein$\downarrow$} & \textbf{MMD$\downarrow$} & \textbf{JSD$\downarrow$} & \textbf{MIA AUC$\to$0.5} & \textbf{MLE} \\
\midrule
\multirow{4}{*}{Adult}
& Baseline (GReaT) & $0.0585{\scriptstyle\pm0.0009}$ & $0.0068{\scriptstyle\pm0.0005}$ & $0.0118{\scriptstyle\pm0.0007}$ & $0.5096{\scriptstyle\pm0.0004}$ & $0.912{\scriptstyle\pm0.001}$ \\
& DCR ($\rho_{\mathrm{dcr}}$) & $\mathbf{0.0379{\scriptstyle\pm0.0052}}$ & $\mathbf{0.0024{\scriptstyle\pm0.0009}}$ & $0.0054{\scriptstyle\pm0.0014}$ & $\mathbf{0.5087{\scriptstyle\pm0.0009}}$ & $0.918{\scriptstyle\pm0.002}$ \\
& Classifier ($\rho_{\mathrm{cls}}$) & $0.0384{\scriptstyle\pm0.0077}$ & $0.0025{\scriptstyle\pm0.0012}$ & $\mathbf{0.0053{\scriptstyle\pm0.0016}}$ & $0.5094{\scriptstyle\pm0.0010}$ & $\mathbf{0.920{\scriptstyle\pm0.001}}$ \\
\rowcolor{gray!20}
& Data Leakage ($\rho_{\mathrm{leak}}$) & $0.0183{\scriptstyle\pm0.0007}$ & $0.0131{\scriptstyle\pm0.0013}$ & $0.0089{\scriptstyle\pm0.0007}$ & $0.6609{\scriptstyle\pm0.0142}$ & $0.904{\scriptstyle\pm0.005}$ \\
\midrule

\multirow{4}{*}{Shoppers}
& Baseline (GReaT) & $0.2543{\scriptstyle\pm0.0003}$ & $0.1410{\scriptstyle\pm0.0005}$ & $0.1760{\scriptstyle\pm0.0006}$ & $\mathbf{0.4887{\scriptstyle\pm0.0007}}$ & $0.899{\scriptstyle\pm0.002}$ \\
& DCR ($\rho_{\mathrm{dcr}}$) & $\mathbf{0.2483{\scriptstyle\pm0.0054}}$ & $\mathbf{0.1408{\scriptstyle\pm0.0065}}$ & $\mathbf{0.1700{\scriptstyle\pm0.0023}}$ & $0.4853{\scriptstyle\pm0.0014}$ & $\mathbf{0.905{\scriptstyle\pm0.002}}$ \\
& Classifier ($\rho_{\mathrm{cls}}$) & $0.2538{\scriptstyle\pm0.0039}$ & $0.1418{\scriptstyle\pm0.0079}$ & $0.1719{\scriptstyle\pm0.0016}$ & $0.4849{\scriptstyle\pm0.0022}$ & $0.904{\scriptstyle\pm0.003}$ \\
\rowcolor{gray!20}
& Data Leakage ($\rho_{\mathrm{leak}}$) & $0.0355{\scriptstyle\pm0.0008}$ & $0.0174{\scriptstyle\pm0.0009}$ & $0.0201{\scriptstyle\pm0.0005}$ & $0.7836{\scriptstyle\pm0.0061}$ & $0.896{\scriptstyle\pm0.004}$ \\
\midrule

\multirow{4}{*}{Beijing}
& Baseline (GReaT) & $0.3582{\scriptstyle\pm0.0004}$ & $0.1011{\scriptstyle\pm0.0003}$ & $0.2371{\scriptstyle\pm0.0002}$ & $0.5164{\scriptstyle\pm0.0006}$ & $0.618{\scriptstyle\pm0.011}$ \\
& DCR ($\rho_{\mathrm{dcr}}$) & $0.3583{\scriptstyle\pm0.0006}$ & $0.1006{\scriptstyle\pm0.0006}$ & $0.2371{\scriptstyle\pm0.0003}$ & $\mathbf{0.5105{\scriptstyle\pm0.0027}}$ & $0.648{\scriptstyle\pm0.013}$ \\
& Classifier ($\rho_{\mathrm{cls}}$) & $\mathbf{0.3501{\scriptstyle\pm0.0039}}$ & $\mathbf{0.0992{\scriptstyle\pm0.0006}}$ & $\mathbf{0.2344{\scriptstyle\pm0.0011}}$ & $0.5126{\scriptstyle\pm0.0023}$ & $\mathbf{0.592{\scriptstyle\pm0.005}}$ \\
\rowcolor{gray!20}
& Data Leakage ($\rho_{\mathrm{leak}}$) & $0.0246{\scriptstyle\pm0.0001}$ & $0.0023{\scriptstyle\pm0.0003}$ & $0.0065{\scriptstyle\pm0.0001}$ & $0.6770{\scriptstyle\pm0.0069}$ & $0.706{\scriptstyle\pm0.010}$ \\
\midrule

\multirow{4}{*}{Default}
& Baseline (GReaT) & $0.4213{\scriptstyle\pm0.0008}$ & $0.3590{\scriptstyle\pm0.0005}$ & $0.2964{\scriptstyle\pm0.0006}$ & $0.5085{\scriptstyle\pm0.0002}$ & $\mathbf{0.780{\scriptstyle\pm0.003}}$ \\
& DCR ($\rho_{\mathrm{dcr}}$) & $0.4218{\scriptstyle\pm0.0006}$ & $\mathbf{0.3589{\scriptstyle\pm0.0004}}$ & $0.2927{\scriptstyle\pm0.0008}$ & $0.5061{\scriptstyle\pm0.0013}$ & $0.776{\scriptstyle\pm0.005}$ \\
& Classifier ($\rho_{\mathrm{cls}}$) & $\mathbf{0.4211{\scriptstyle\pm0.0004}}$ & $0.3590{\scriptstyle\pm0.0004}$ & $\mathbf{0.2926{\scriptstyle\pm0.0008}}$ & $\mathbf{0.5060{\scriptstyle\pm0.0011}}$ & $0.777{\scriptstyle\pm0.004}$ \\
\rowcolor{gray!20}
& Data Leakage ($\rho_{\mathrm{leak}}$) & $0.0181{\scriptstyle\pm0.0004}$ & $0.0025{\scriptstyle\pm0.0003}$ & $0.0109{\scriptstyle\pm0.0008}$ & $0.8084{\scriptstyle\pm0.0048}$ & $0.771{\scriptstyle\pm0.008}$ \\
\midrule

\multirow{4}{*}{Magic}
& Baseline (GReaT) & $0.0495{\scriptstyle\pm0.0002}$ & $0.0053{\scriptstyle\pm0.0004}$ & $0.0329{\scriptstyle\pm0.0005}$ & $0.4987{\scriptstyle\pm0.0003}$ & $0.905{\scriptstyle\pm0.002}$ \\
& DCR ($\rho_{\mathrm{dcr}}$) & $\mathbf{0.0472{\scriptstyle\pm0.0155}}$ & $0.0075{\scriptstyle\pm0.0028}$ & $0.0293{\scriptstyle\pm0.0061}$ & $0.4973{\scriptstyle\pm0.0019}$ & $\mathbf{0.907{\scriptstyle\pm0.005}}$ \\
& Classifier ($\rho_{\mathrm{cls}}$) & $0.0484{\scriptstyle\pm0.0110}$ & $\mathbf{0.0044{\scriptstyle\pm0.0006}}$ & $\mathbf{0.0180{\scriptstyle\pm0.0045}}$ & $\mathbf{0.4989{\scriptstyle\pm0.0025}}$ & $0.904{\scriptstyle\pm0.003}$ \\
\rowcolor{gray!20}
& Data Leakage ($\rho_{\mathrm{leak}}$) & $0.0305{\scriptstyle\pm0.0005}$ & $0.0195{\scriptstyle\pm0.0018}$ & $0.0226{\scriptstyle\pm0.0019}$ & $0.7496{\scriptstyle\pm0.0093}$ & $0.901{\scriptstyle\pm0.006}$ \\
\bottomrule
\end{tabular}}
\end{table*}

\subsection{Per-round iterative trajectories}
\label{app:iterative}

Figures~\ref{fig:iterative_adult}--\ref{fig:iterative_beijing} show per-dataset trajectories across five post-training rounds. These figures complement the averaged trajectory in Figure~\ref{fig:iterative}. They show that SGCO generally improves rapidly in the first few rounds and remains more stable than adapted instance-level baselines. KTO shows instability in later rounds, consistent with the main text.

\begin{figure*}[h]
    \centering
    \footnotesize
    \begin{tabular}{@{}c@{\hspace{0.25em}}c@{\hspace{0.25em}}c@{\hspace{0.25em}}c@{}}
        \subfloat[CDE$\uparrow$]{\includegraphics[width=0.23\textwidth]{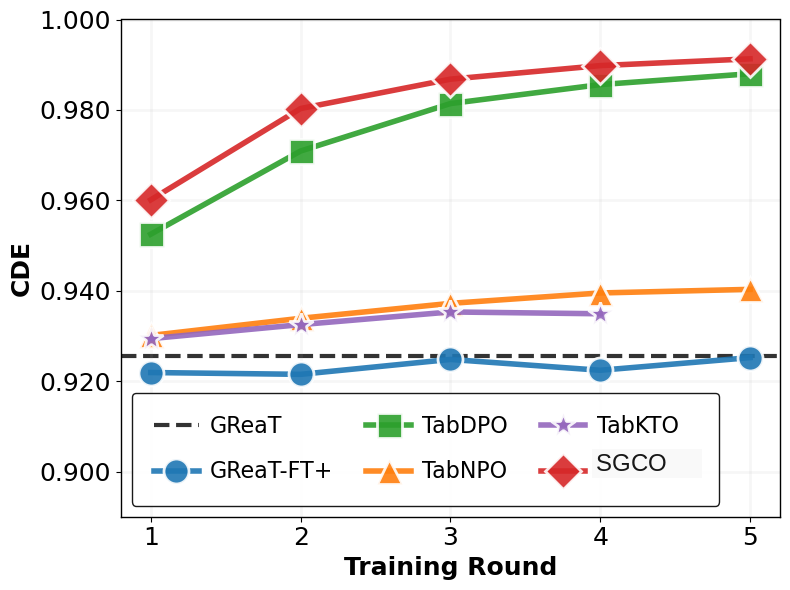}} & 
        \subfloat[PCC$\uparrow$]{\includegraphics[width=0.23\textwidth]{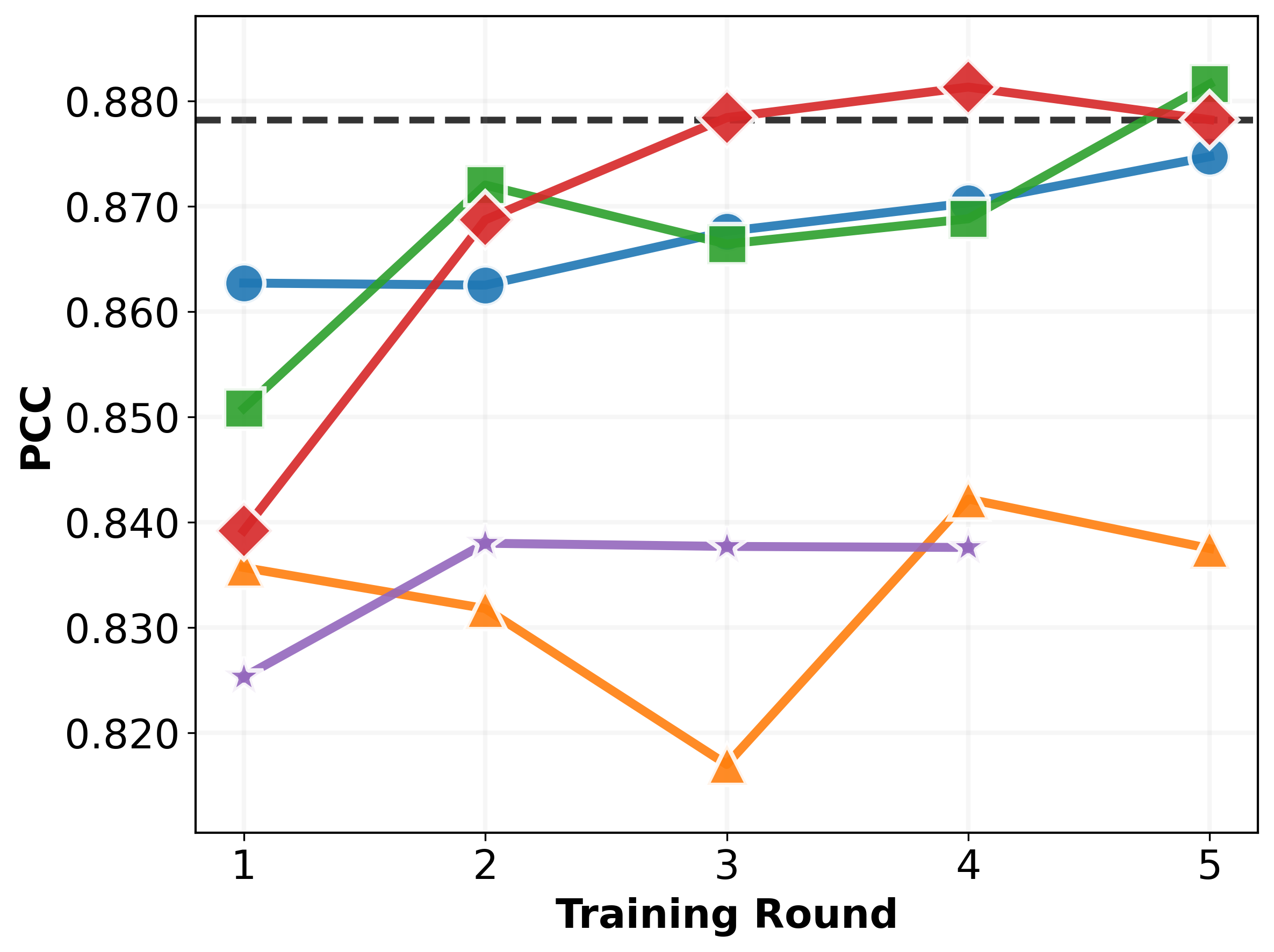}} & 
        \subfloat[$\alpha\uparrow$]{\includegraphics[width=0.23\textwidth]{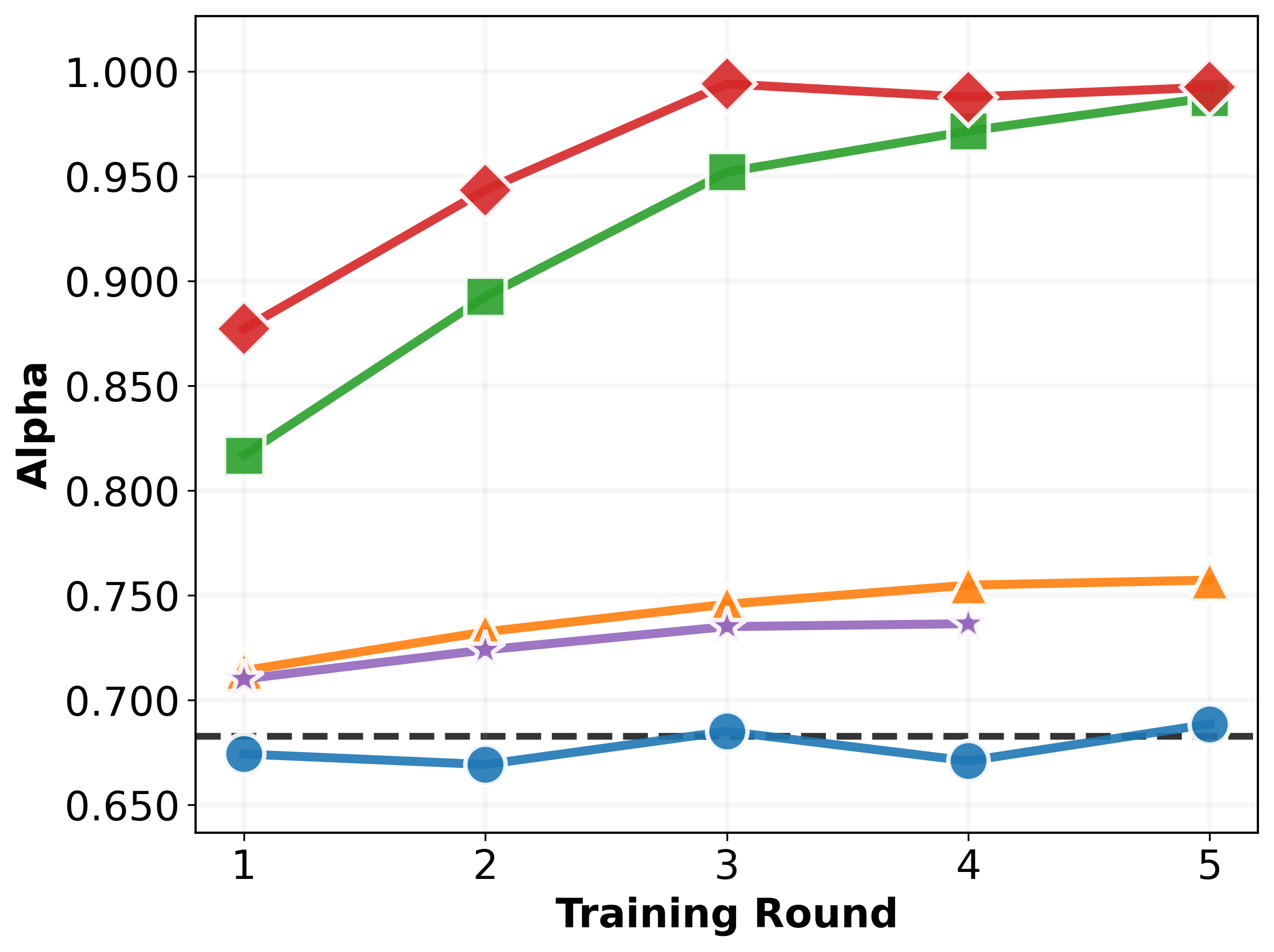}} &
        \subfloat[$\beta\uparrow$]{\includegraphics[width=0.23\textwidth]{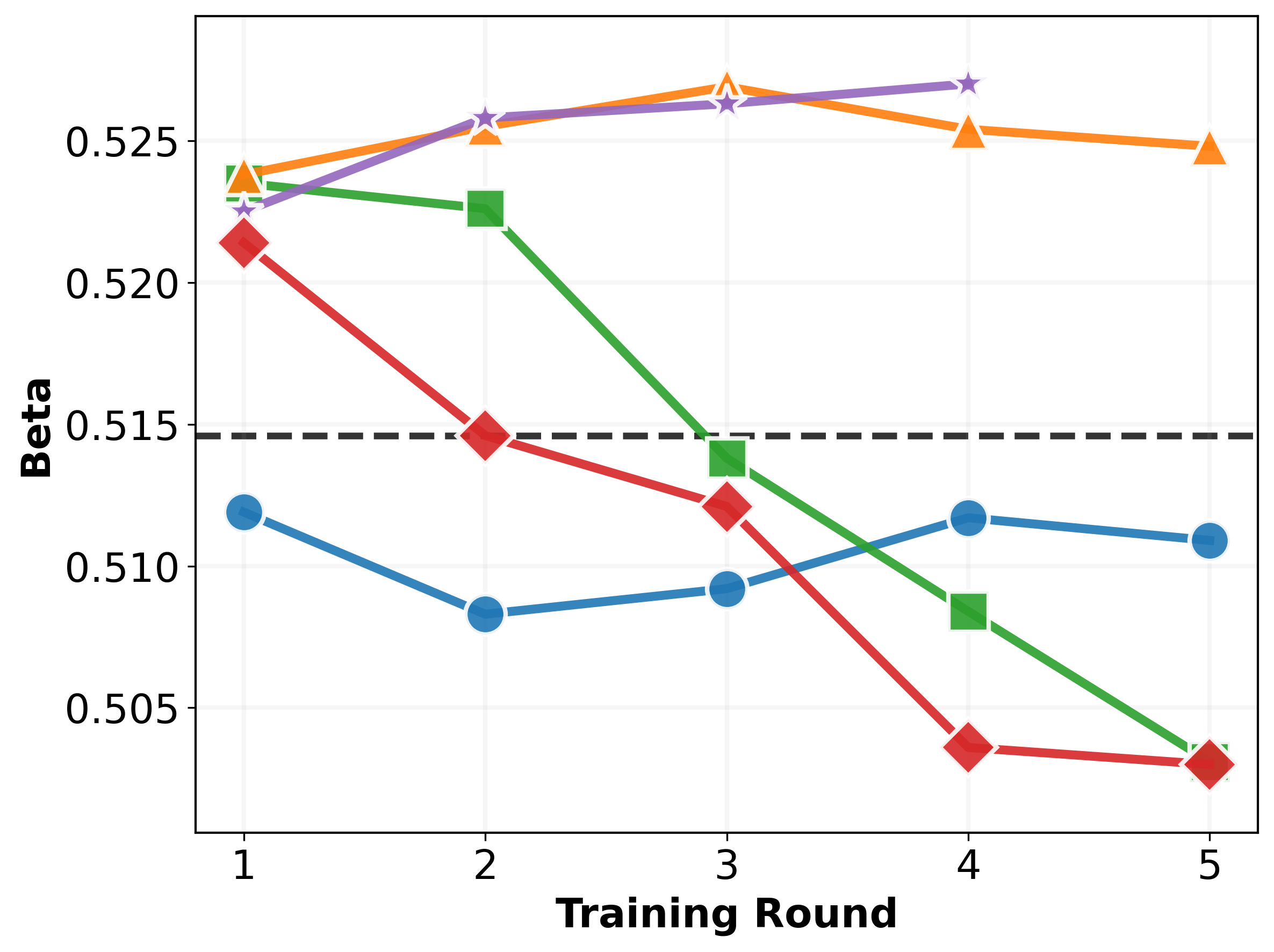}} \\ [1.5ex]
        \subfloat[C2ST$\uparrow$]{\includegraphics[width=0.23\textwidth]{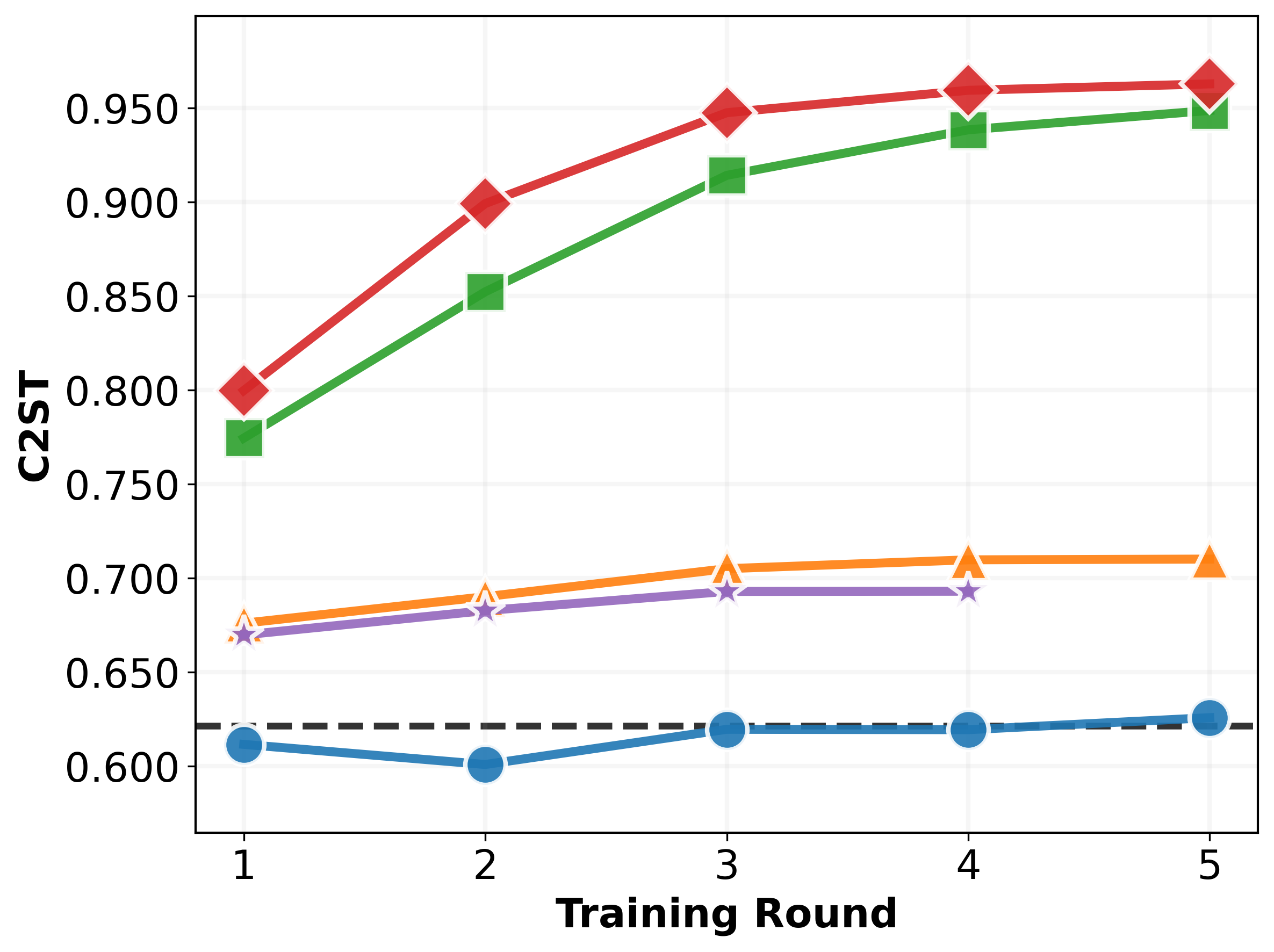}} &
        \subfloat[DA(AUC)$\to0.5$]{\includegraphics[width=0.23\textwidth]{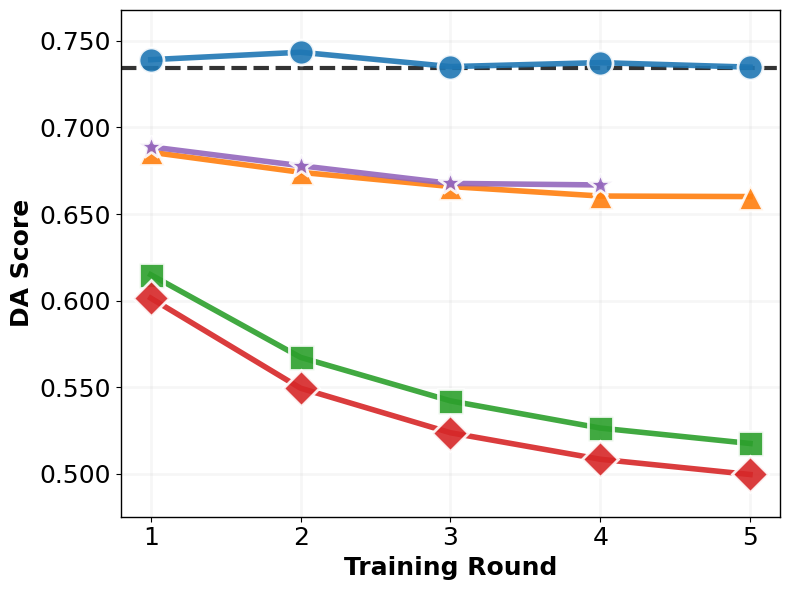}} &
        \subfloat[MLE$\uparrow$]{\includegraphics[width=0.23\textwidth]{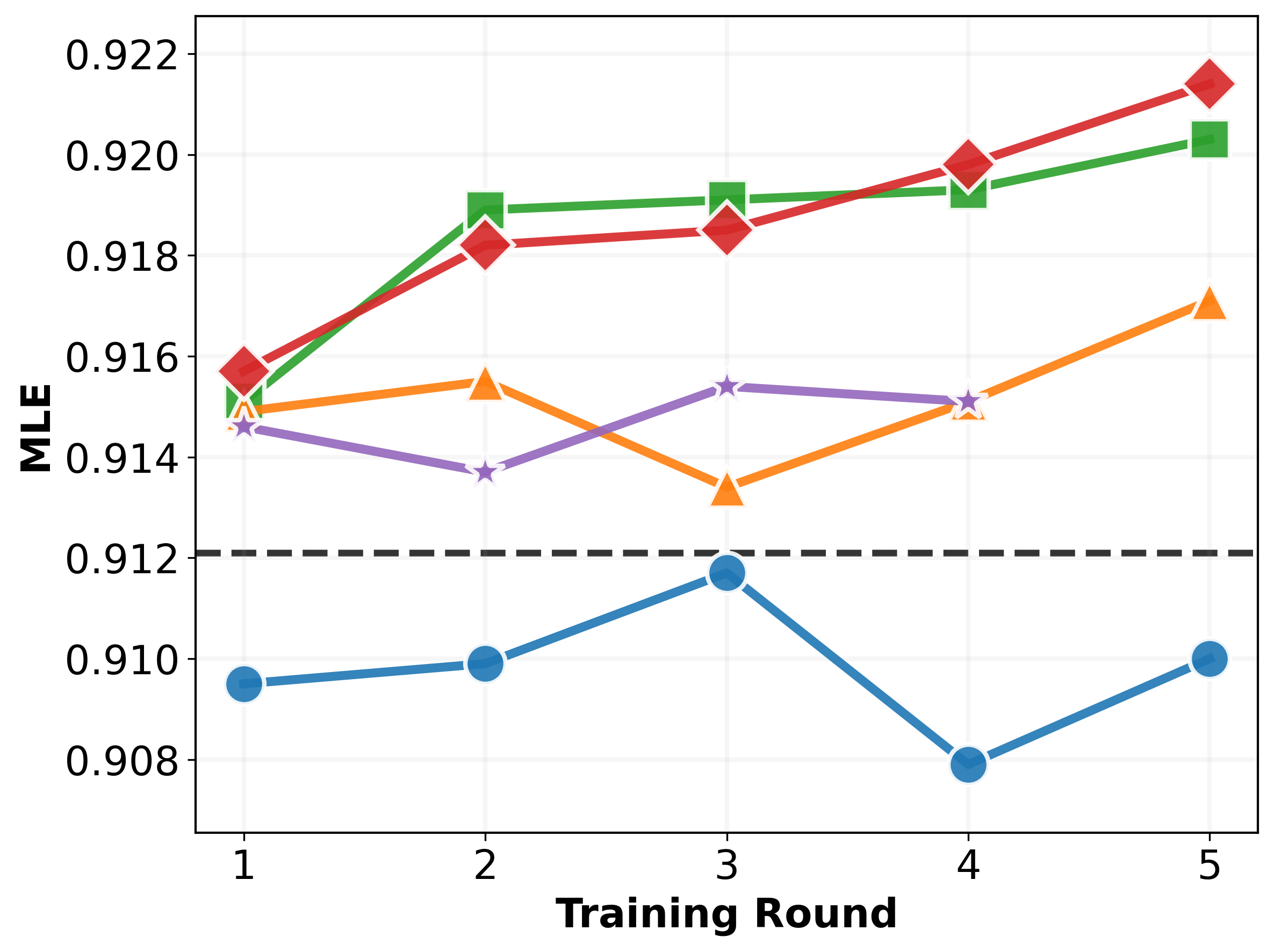}} &
        \\ [1.5ex]
    \end{tabular}
    \caption{Iterative performance progression across training rounds 1--5 on Adult.}
    \label{fig:iterative_adult}
\end{figure*}

\begin{figure*}[h!]
    \centering
    \footnotesize
    \begin{tabular}{@{}c@{\hspace{0.25em}}c@{\hspace{0.25em}}c@{\hspace{0.25em}}c@{}}
        \subfloat[CDE$\uparrow$]{\includegraphics[width=0.23\textwidth]{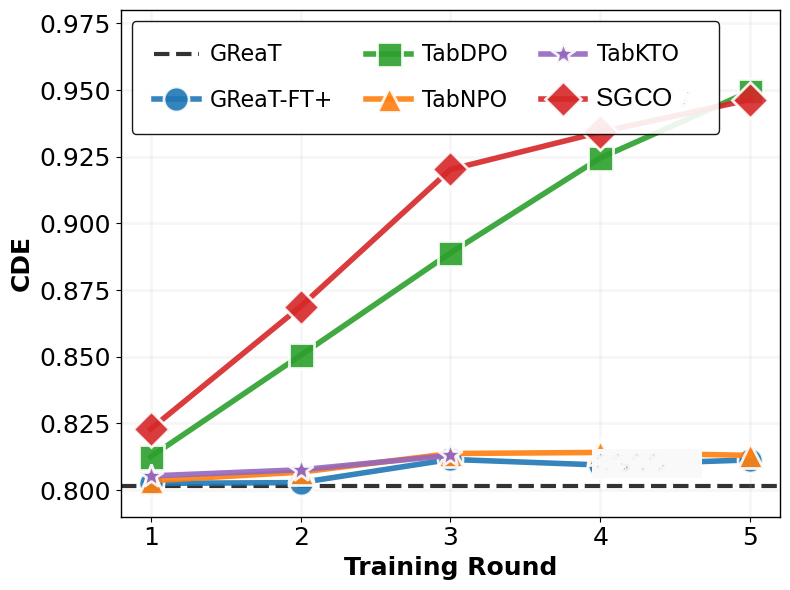}} & 
        \subfloat[PCC$\uparrow$]{\includegraphics[width=0.23\textwidth]{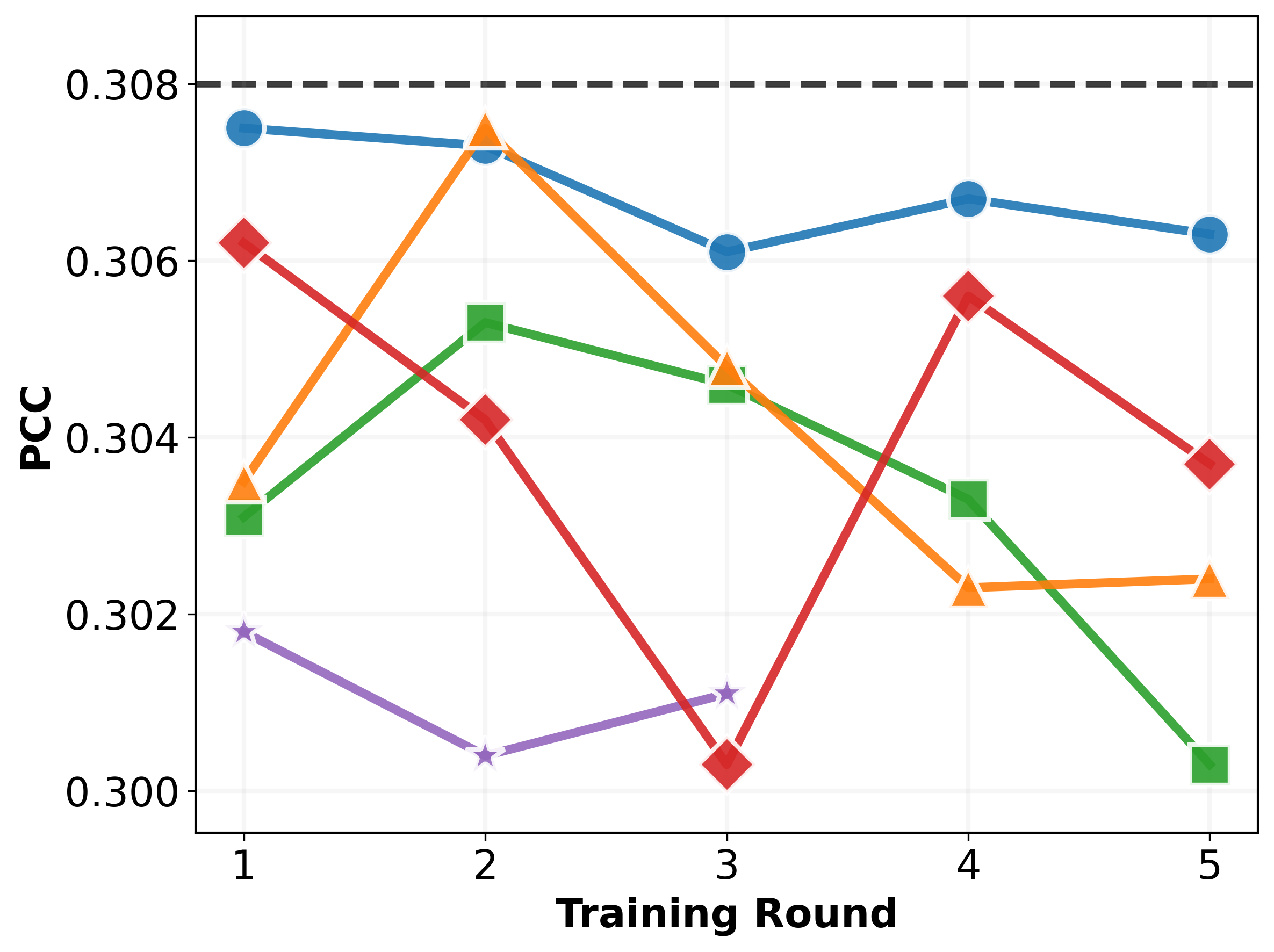}} & 
        \subfloat[$\alpha\uparrow$]{\includegraphics[width=0.23\textwidth]{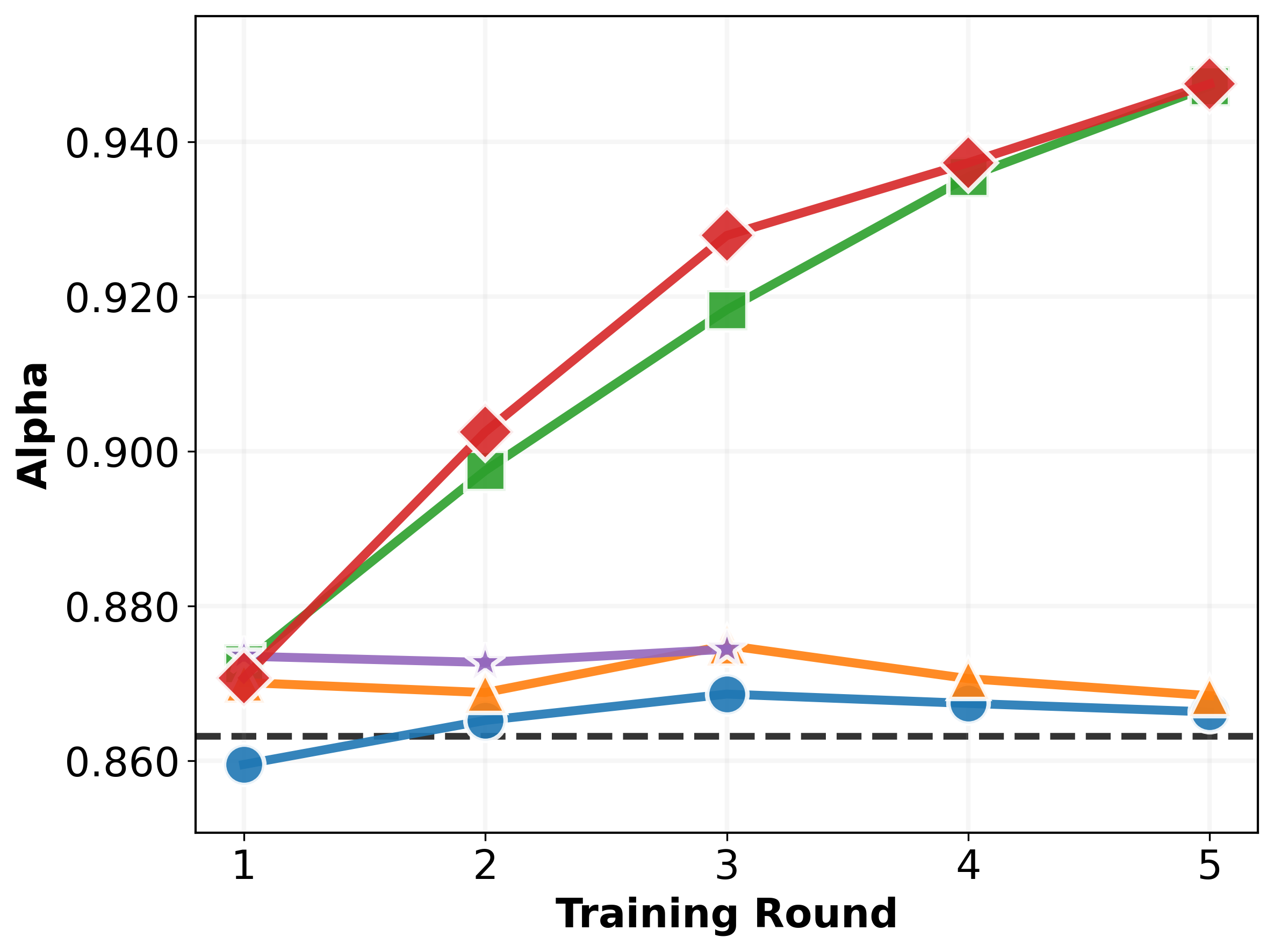}} &
        \subfloat[$\beta\uparrow$]{\includegraphics[width=0.23\textwidth]{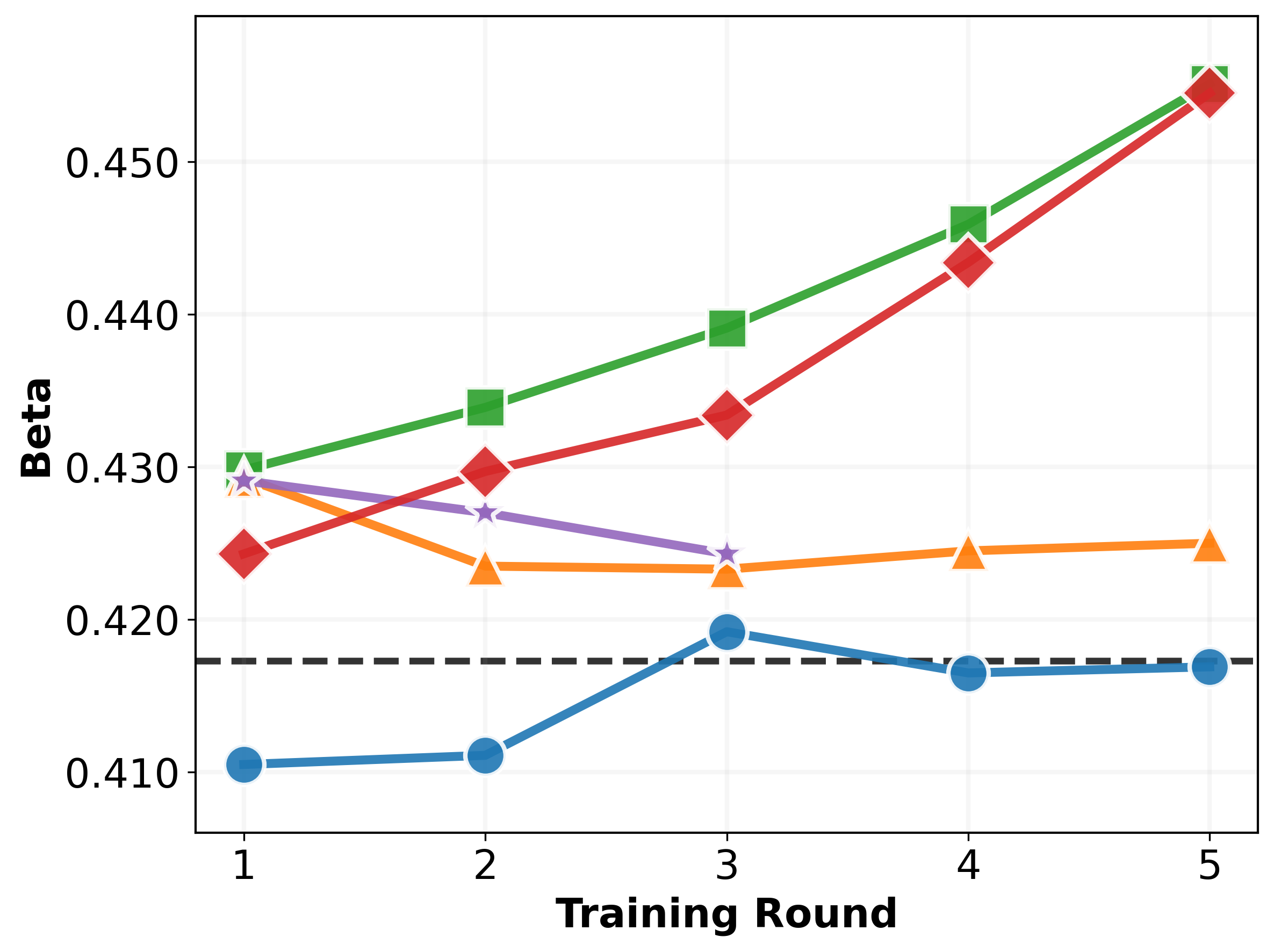}} \\ [1.5ex]
        \subfloat[C2ST$\uparrow$]{\includegraphics[width=0.23\textwidth]{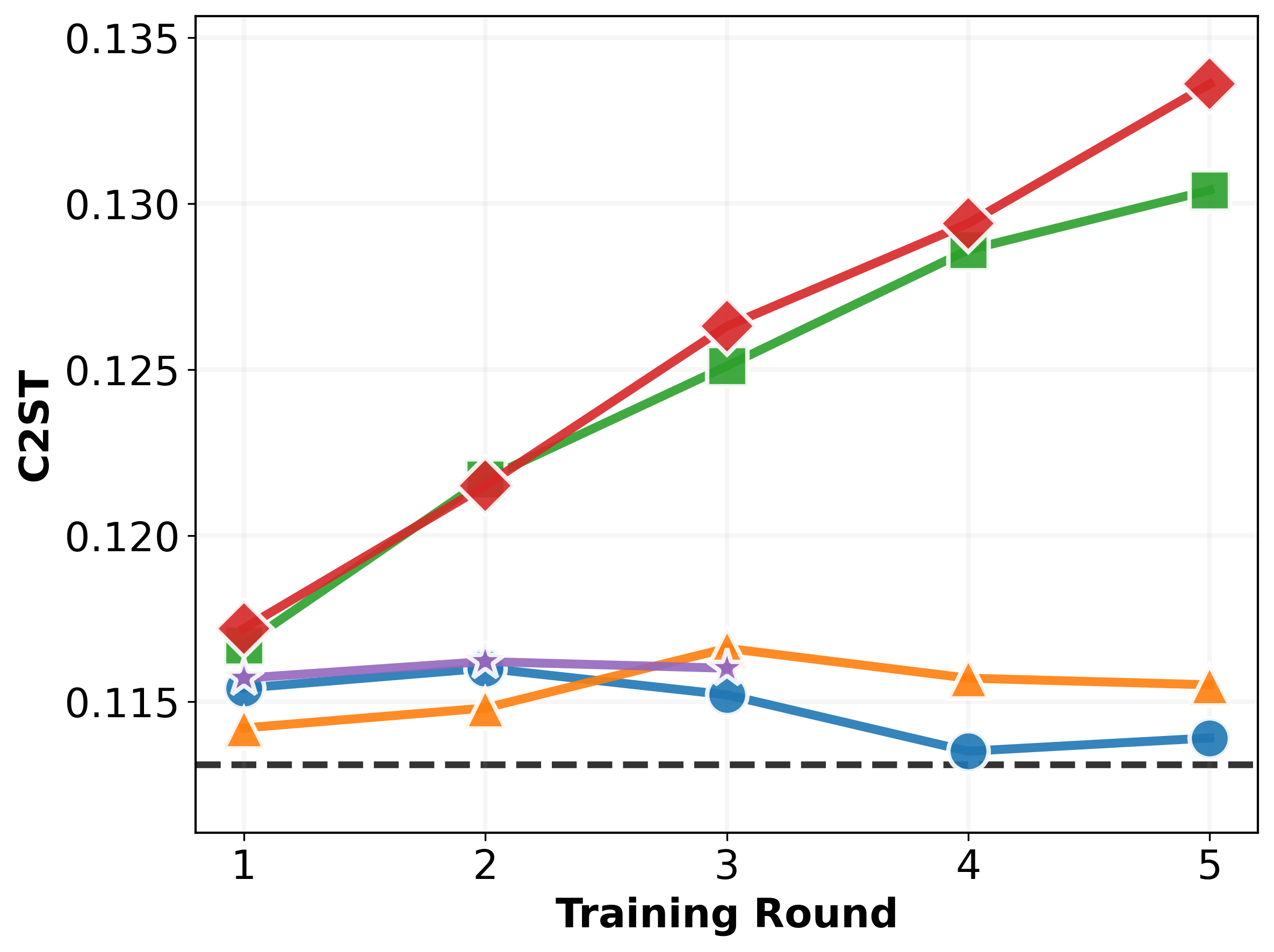}} &
        \subfloat[DA(AUC)$\to0.5$]{\includegraphics[width=0.23\textwidth]{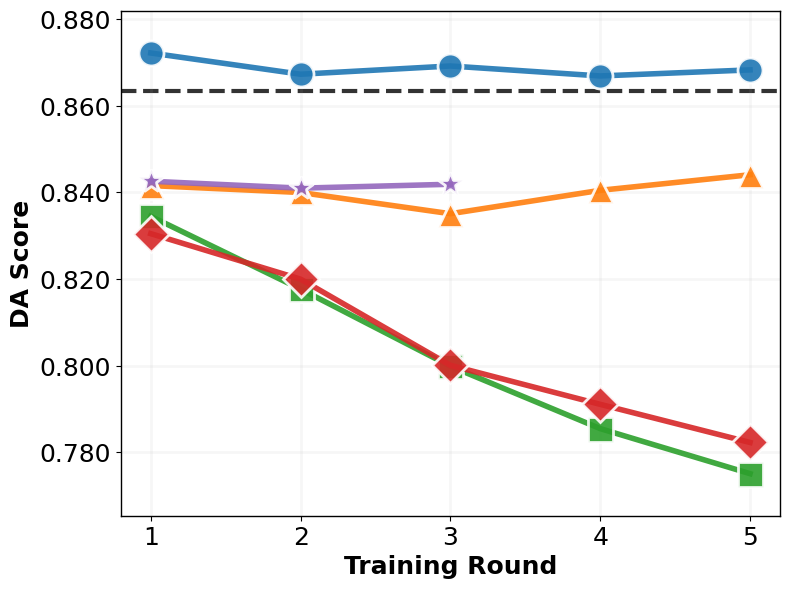}} &
        \subfloat[MLE$\uparrow$]{\includegraphics[width=0.23\textwidth]{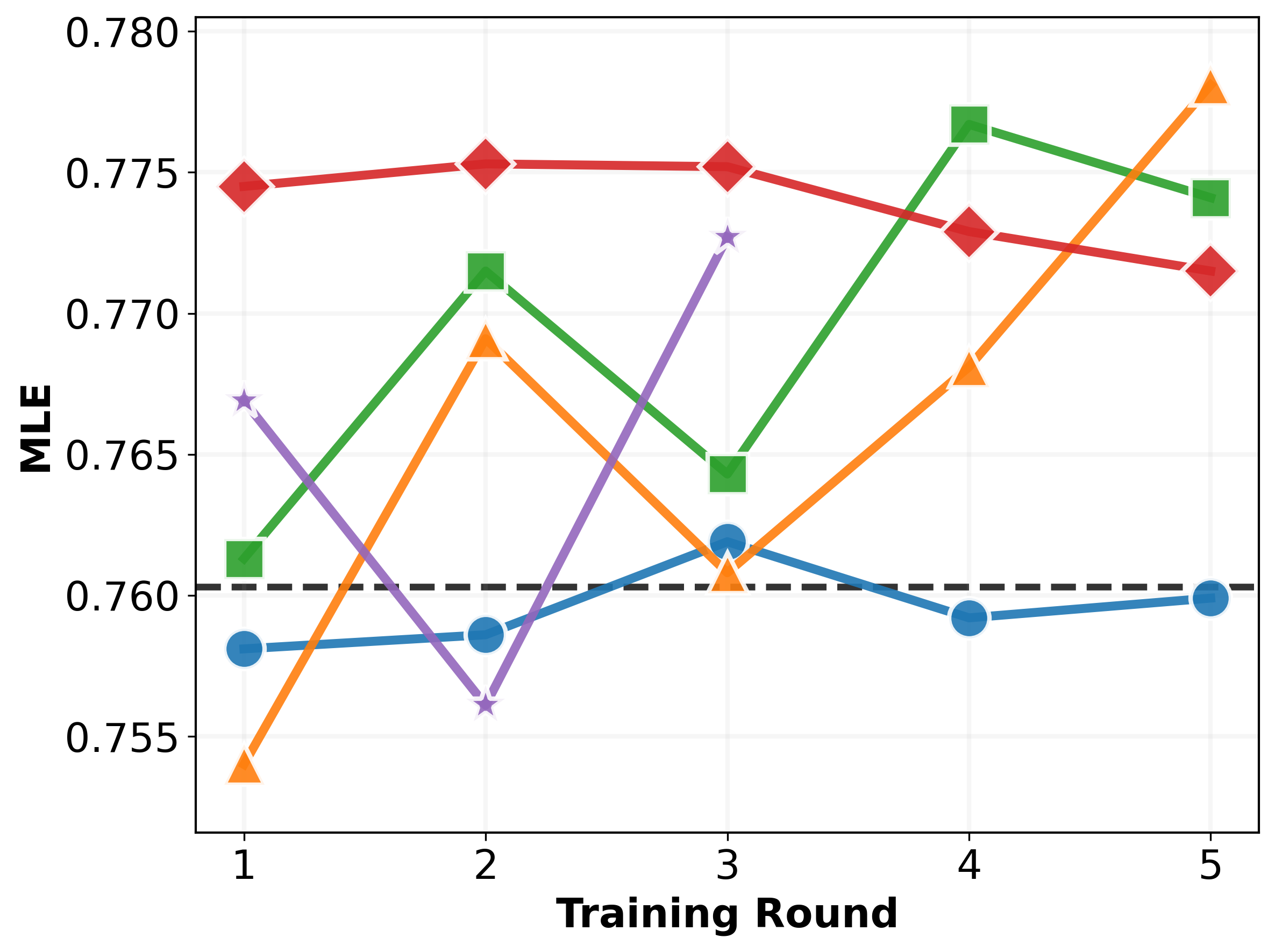}} &
        \\ [1.5ex]
    \end{tabular}
    \caption{Iterative performance progression across training rounds 1--5 on Default.}
    \label{fig:iterative_default}
\end{figure*}

\begin{figure*}[h]
    \centering
    \footnotesize
    \begin{tabular}{@{}c@{\hspace{0.25em}}c@{\hspace{0.25em}}c@{\hspace{0.25em}}c@{}}
        \subfloat[CDE$\uparrow$]{\includegraphics[width=0.23\textwidth]{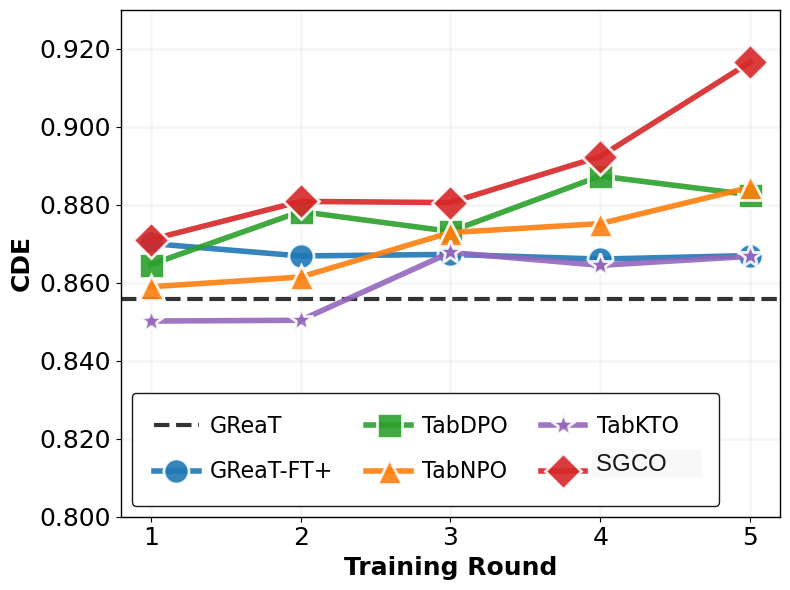}} & 
        \subfloat[PCC$\uparrow$]{\includegraphics[width=0.23\textwidth]{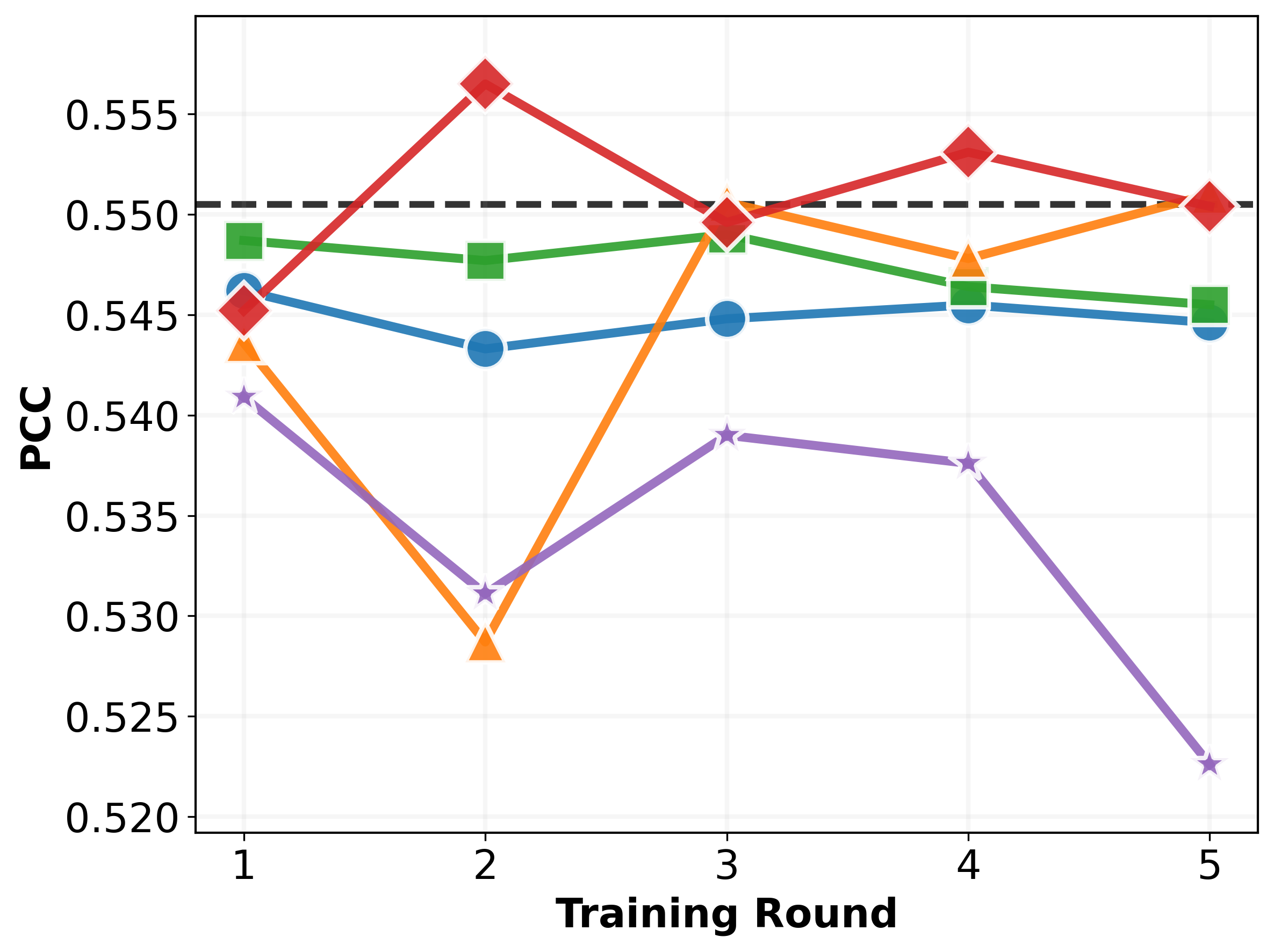}} & 
        \subfloat[$\alpha\uparrow$]{\includegraphics[width=0.23\textwidth]{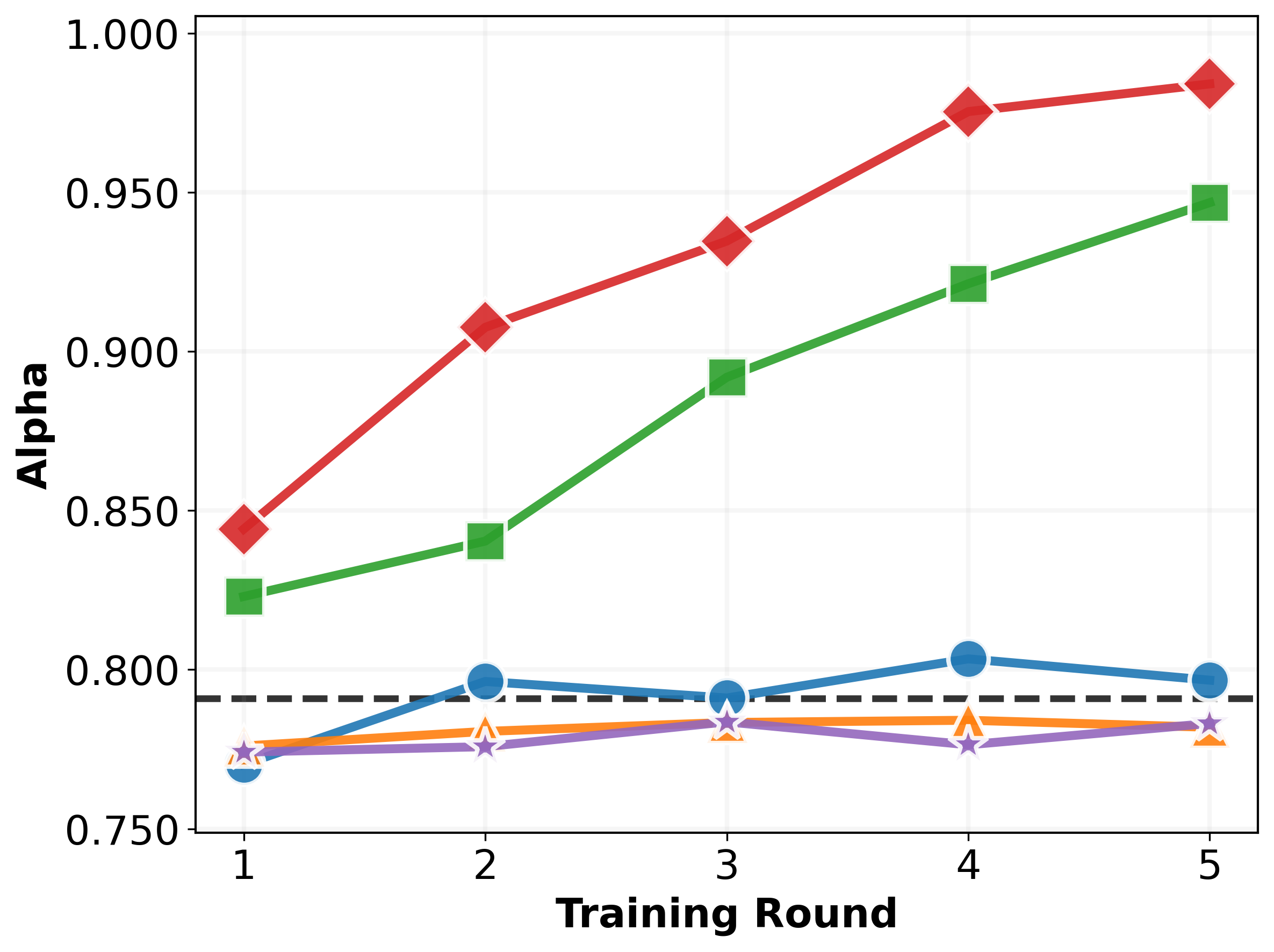}} &
        \subfloat[$\beta\uparrow$]{\includegraphics[width=0.23\textwidth]{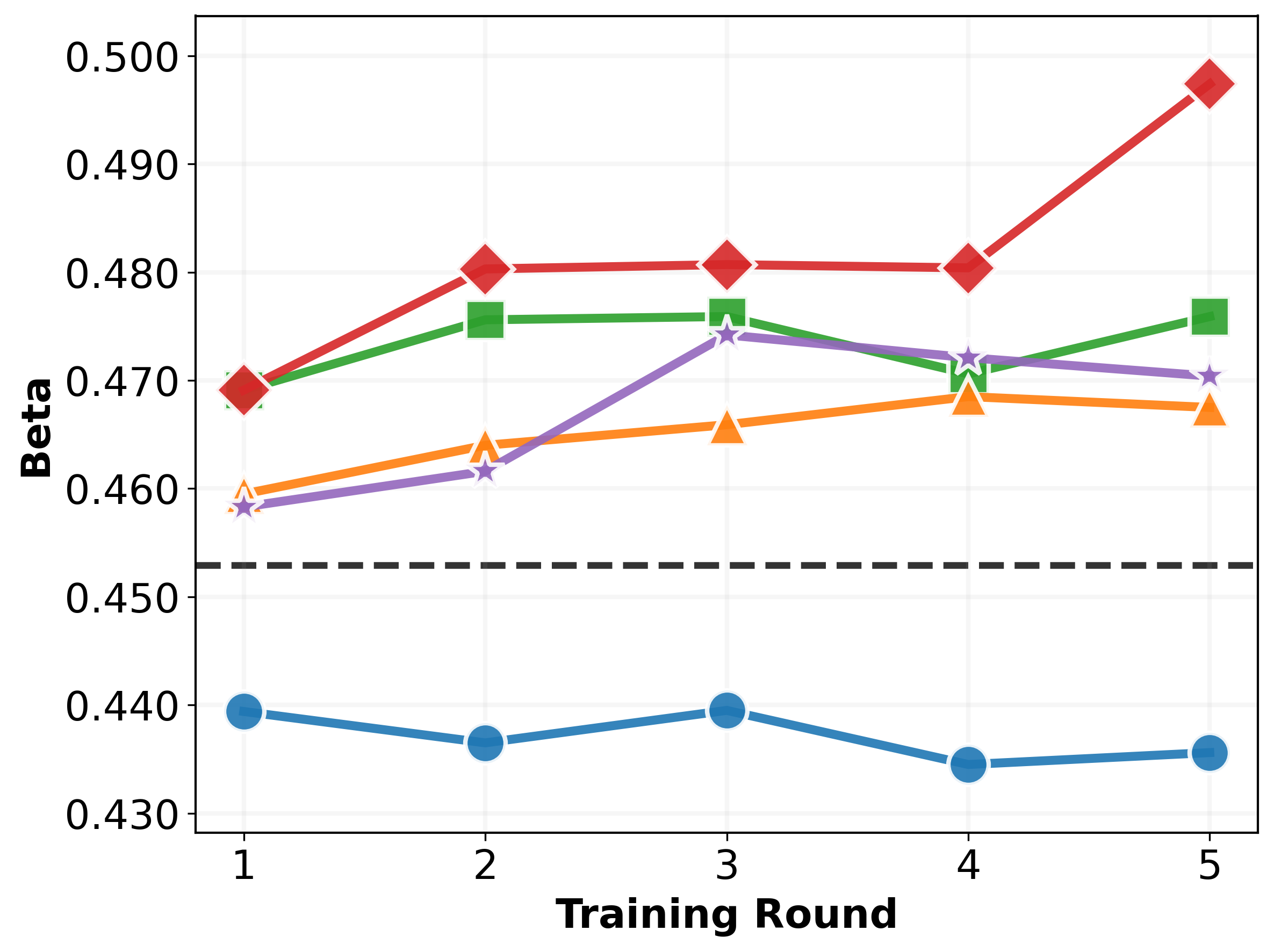}} \\ [1.5ex]
        \subfloat[C2ST$\uparrow$]{\includegraphics[width=0.23\textwidth]{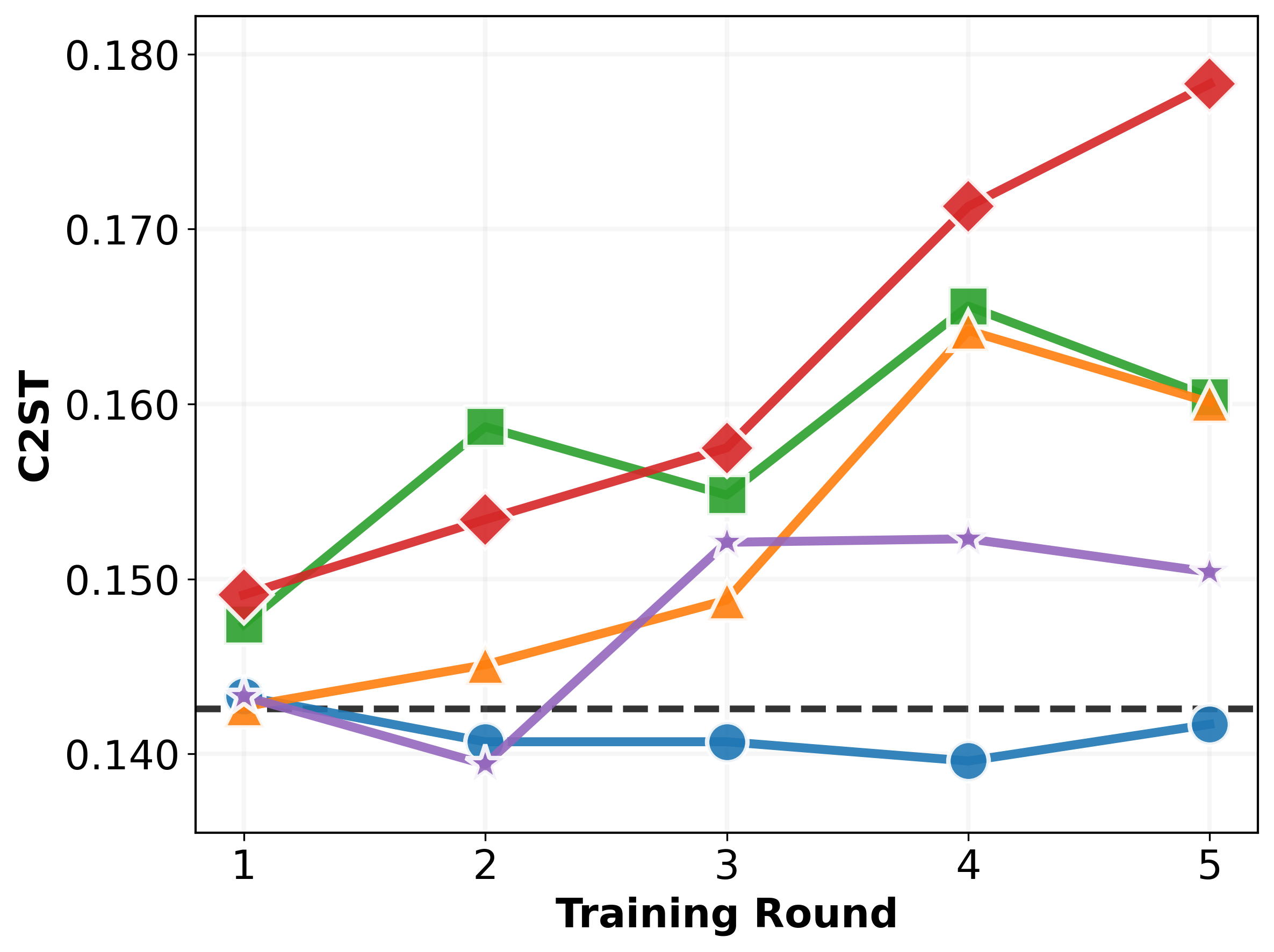}} &
        \subfloat[DA(AUC)$\to0.5$]{\includegraphics[width=0.23\textwidth]{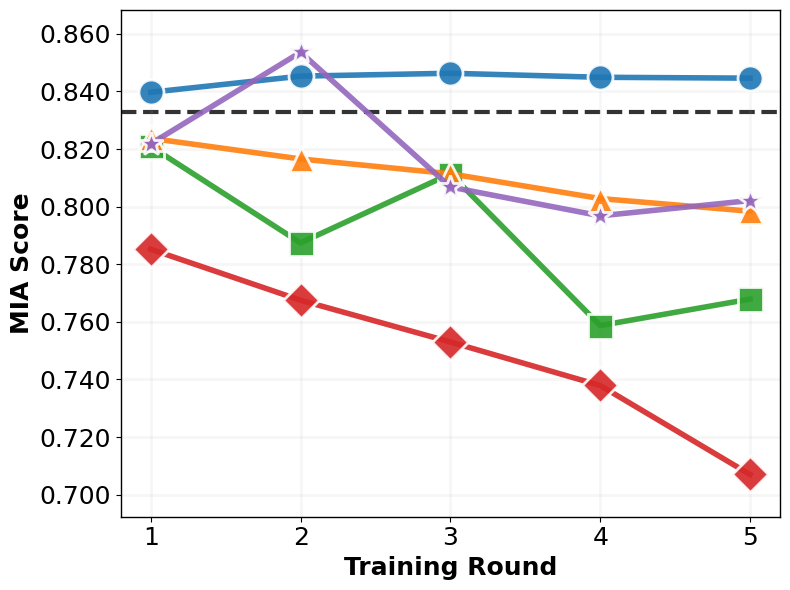}} &
        \subfloat[MLE$\uparrow$]{\includegraphics[width=0.23\textwidth]{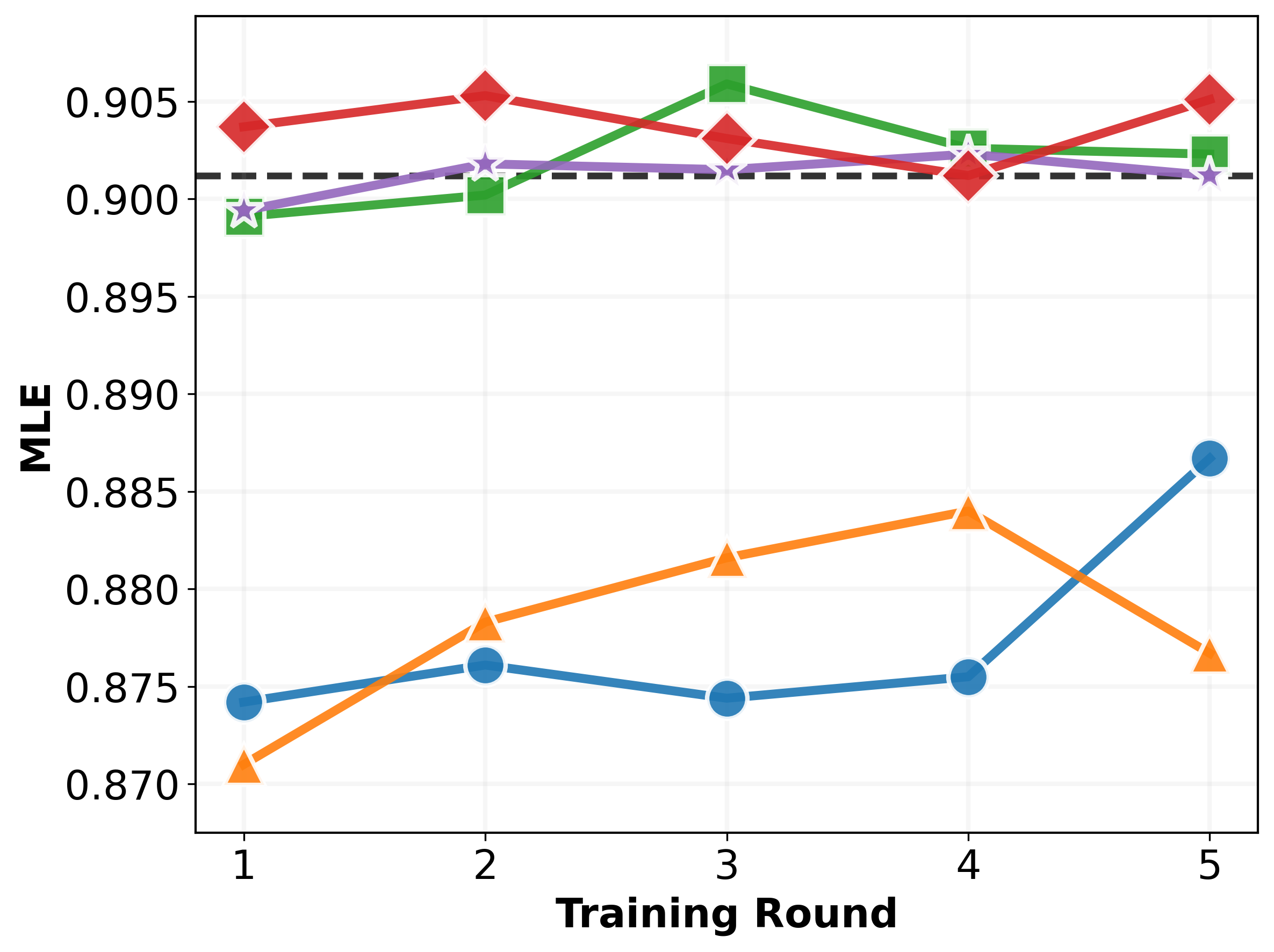}} &
        \\ [1.5ex]
    \end{tabular}
    \caption{Iterative performance progression across training rounds 1--5 on Shoppers.}
    \label{fig:iterative_shoppers}
\end{figure*}

\begin{figure*}[h]
    \centering
    \footnotesize
    \begin{tabular}{@{}c@{\hspace{0.25em}}c@{\hspace{0.25em}}c@{\hspace{0.25em}}c@{}}
        \subfloat[CDE$\uparrow$]{\includegraphics[width=0.23\textwidth]{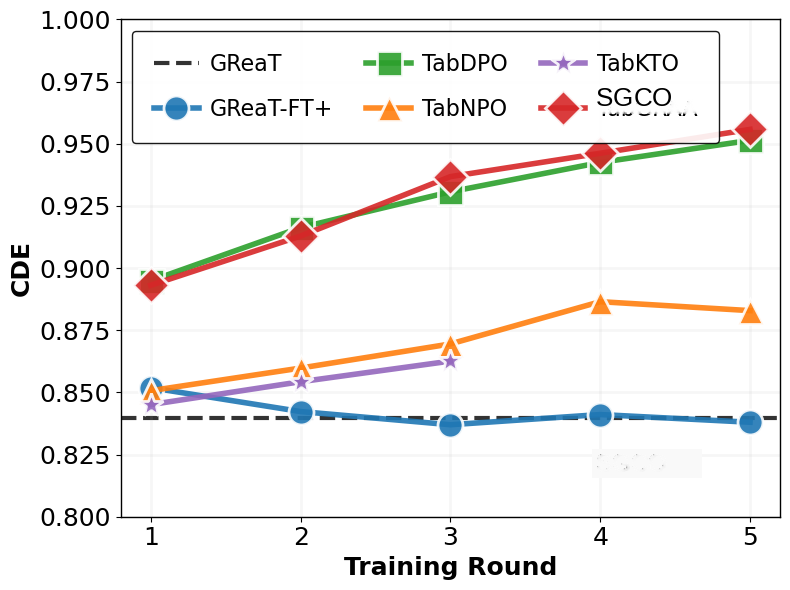}} & 
        \subfloat[PCC$\uparrow$]{\includegraphics[width=0.23\textwidth]{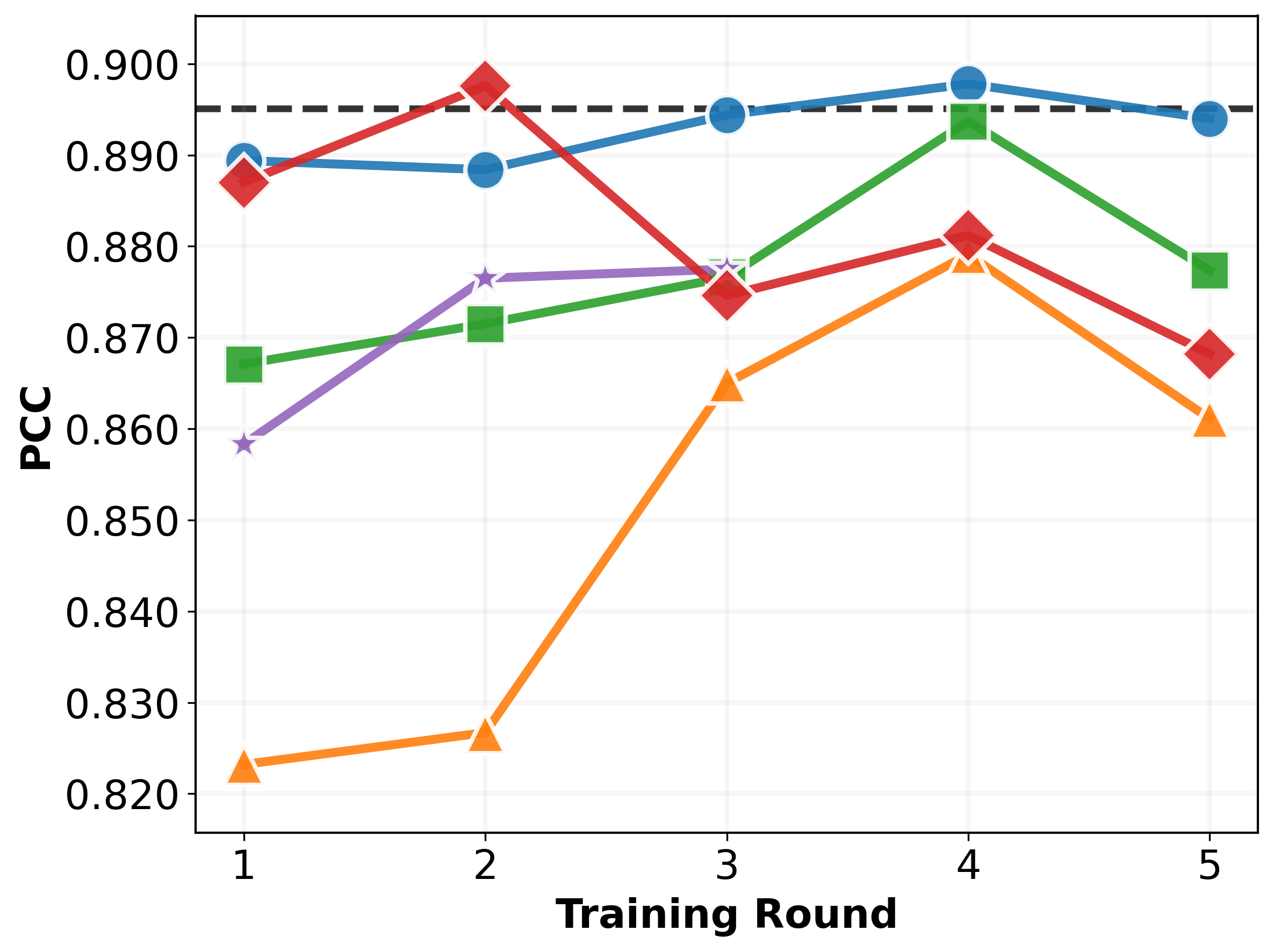}} & 
        \subfloat[$\alpha\uparrow$]{\includegraphics[width=0.23\textwidth]{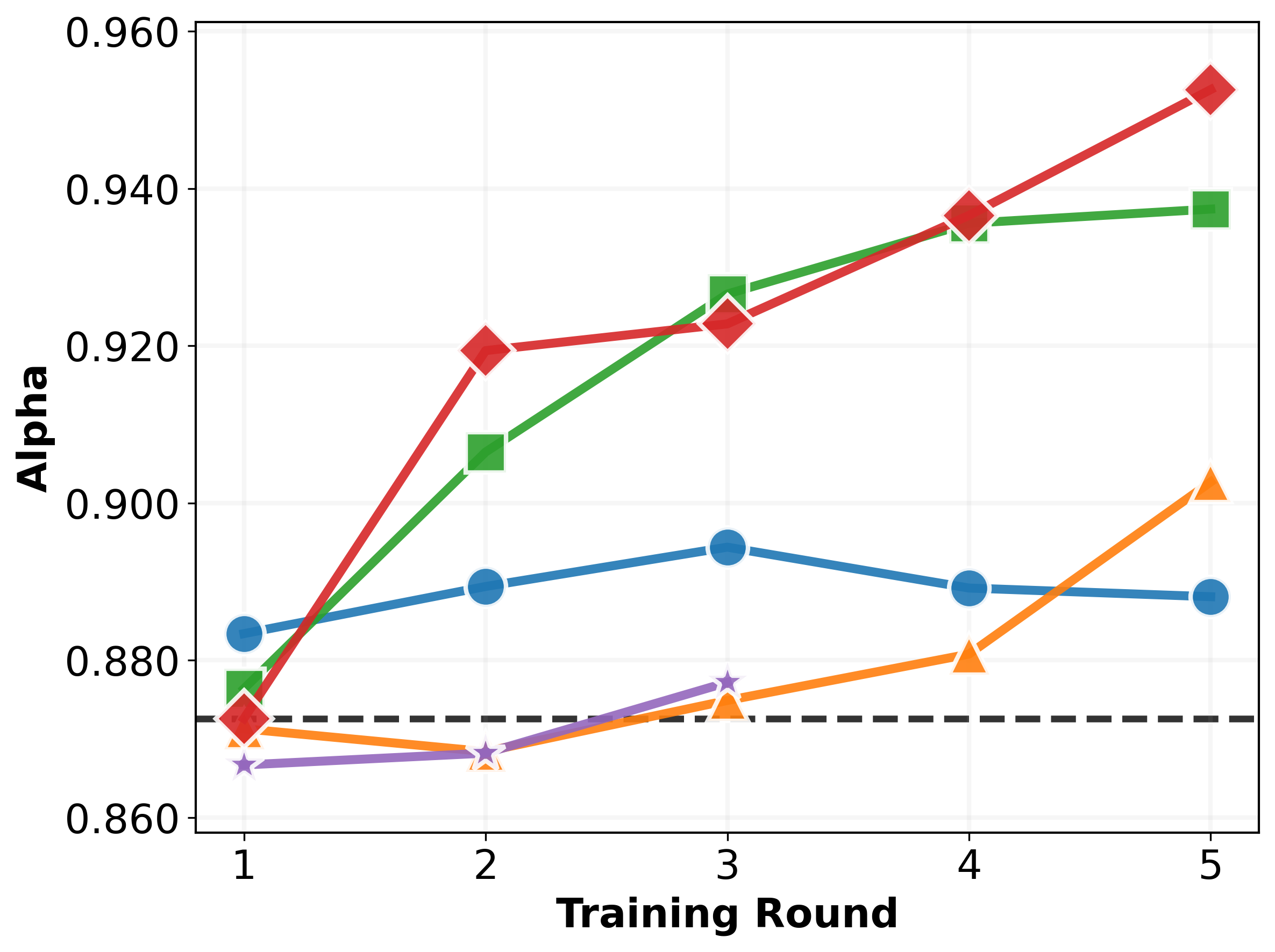}} &
        \subfloat[$\beta\uparrow$]{\includegraphics[width=0.23\textwidth]{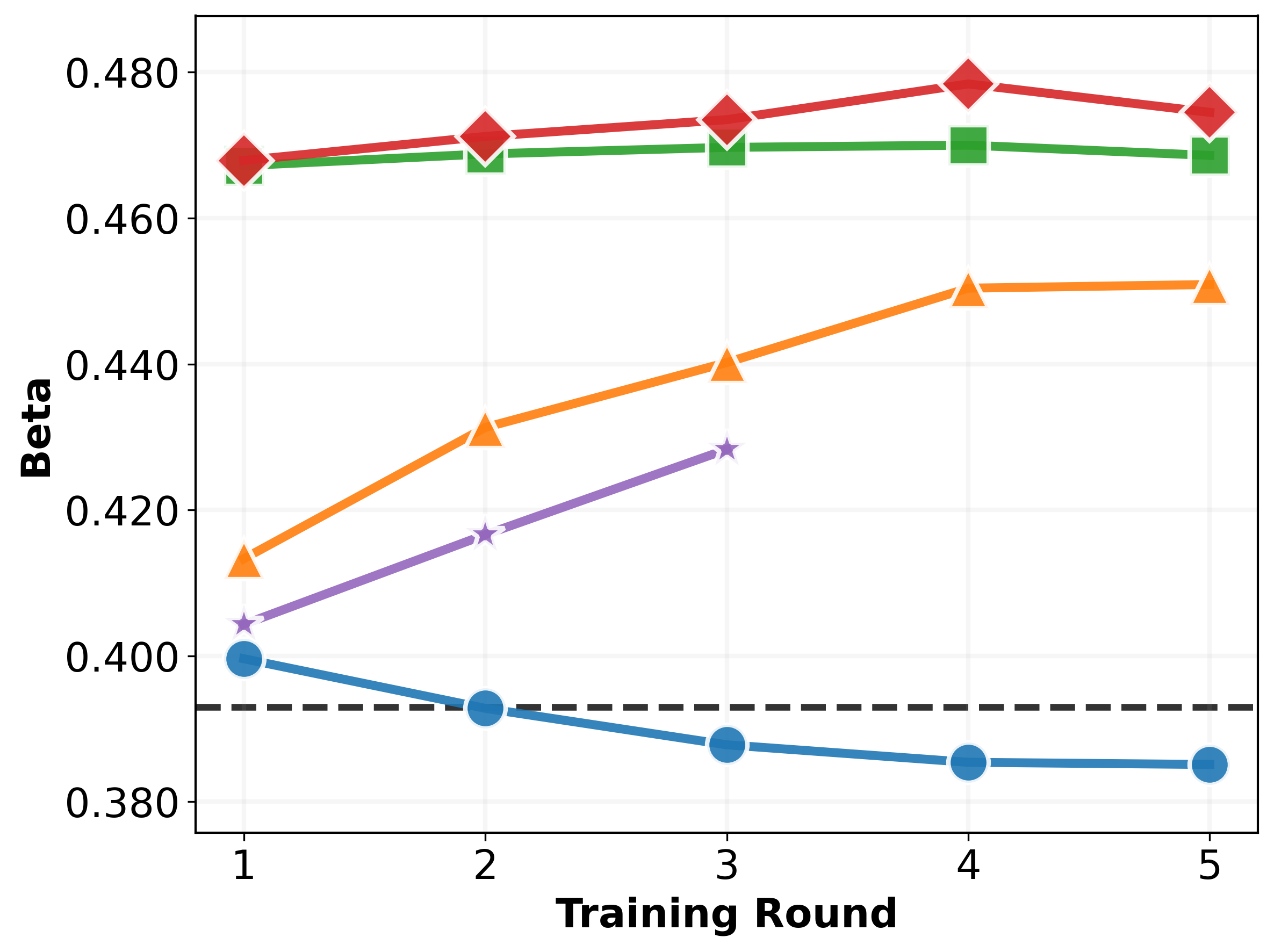}} \\ [1.5ex]
        \subfloat[C2ST$\uparrow$]{\includegraphics[width=0.23\textwidth]{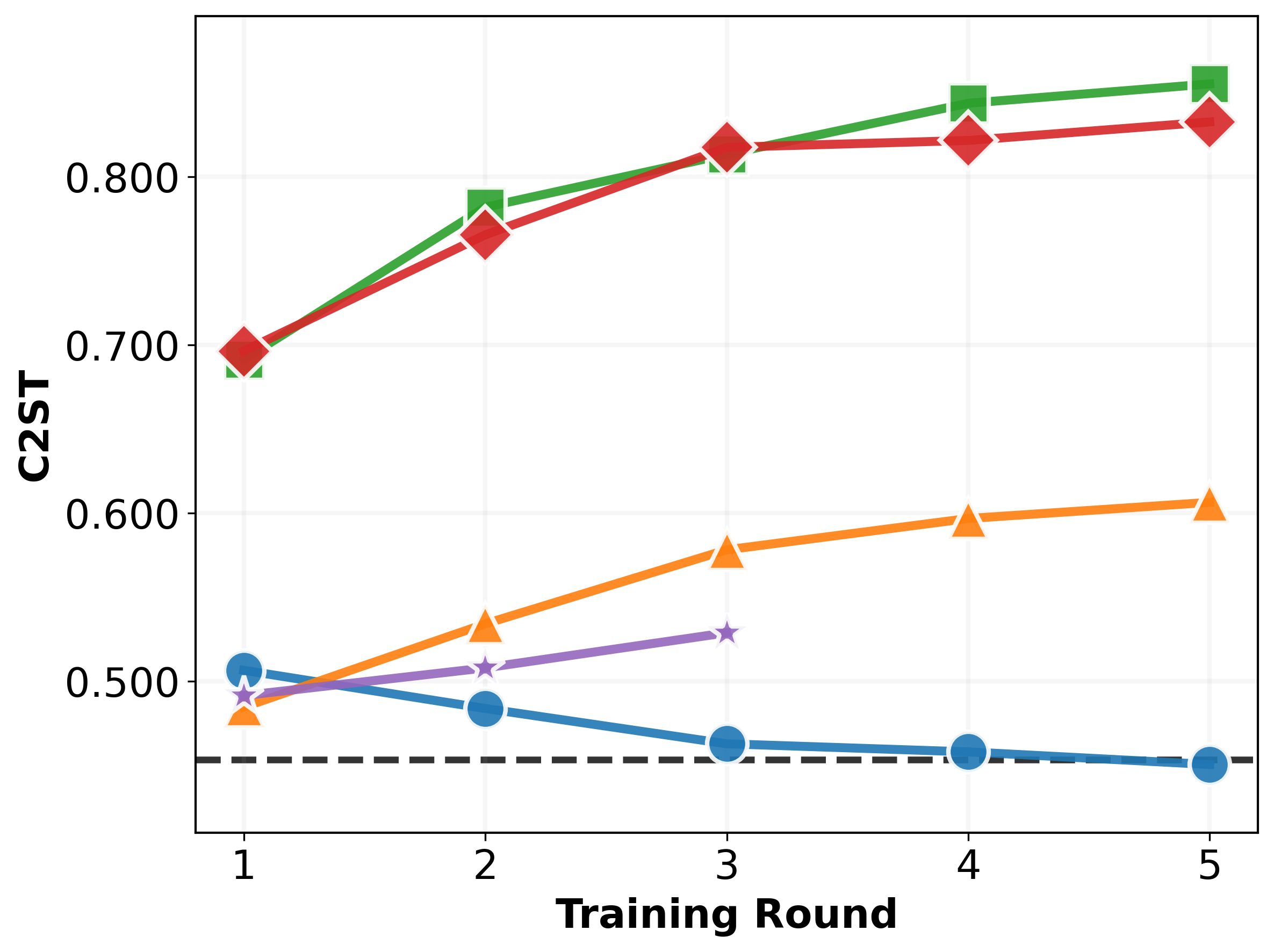}} &
        \subfloat[DA(AUC)$\to0.5$]{\includegraphics[width=0.23\textwidth]{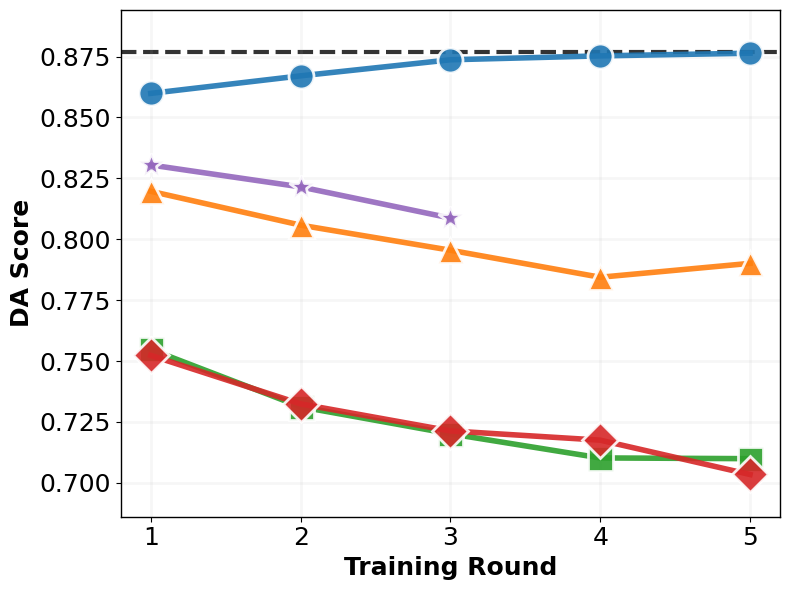}} &
        \subfloat[MLE$\uparrow$]{\includegraphics[width=0.23\textwidth]{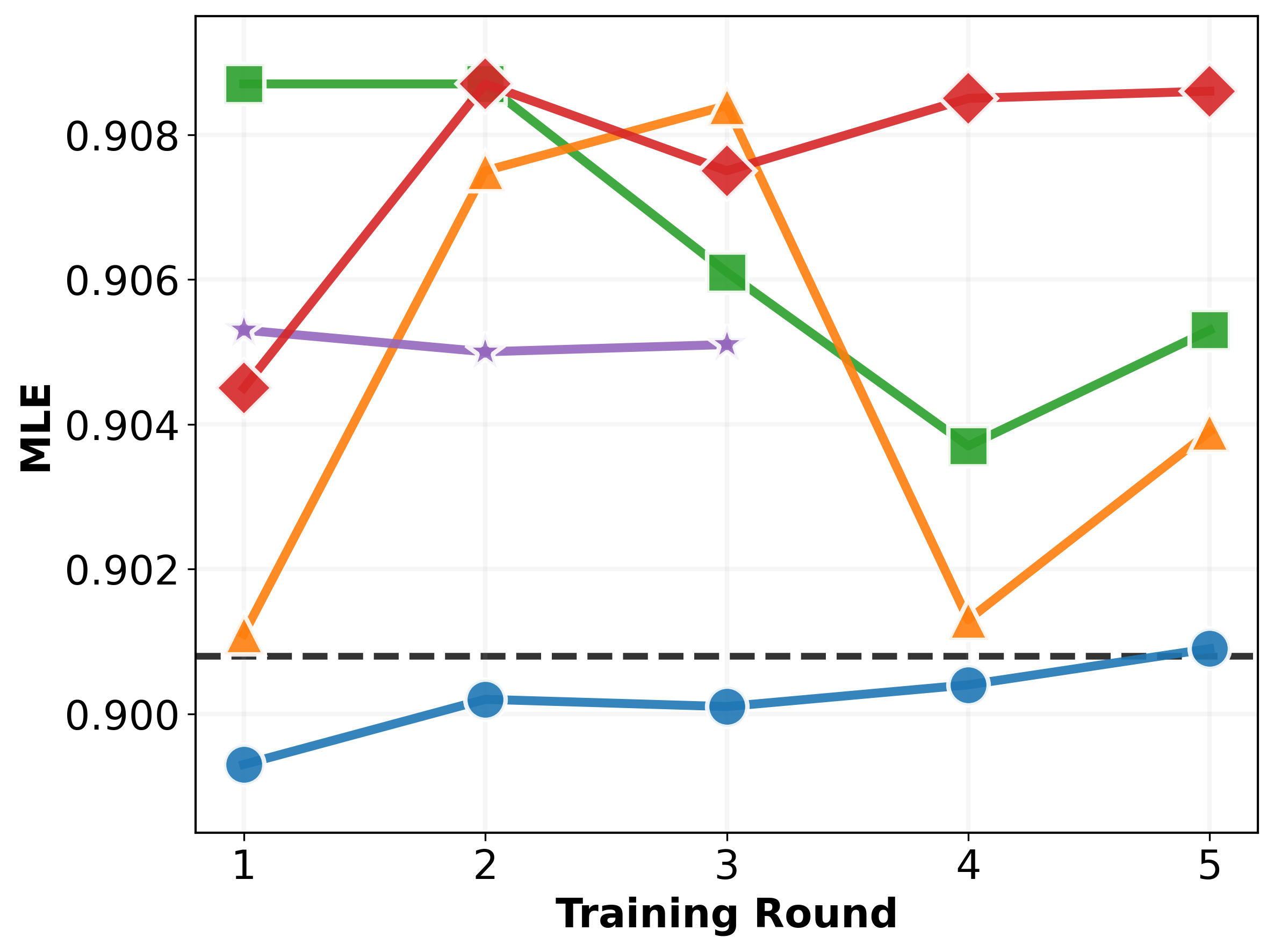}} &
        \\ [1.5ex]
    \end{tabular}
    \caption{Iterative performance progression across training rounds 1--5 on Magic.}
    \label{fig:iterative_magic}
\end{figure*}

\begin{figure*}[h!]
    \centering
    \footnotesize
    \begin{tabular}{@{}c@{\hspace{0.25em}}c@{\hspace{0.25em}}c@{\hspace{0.25em}}c@{}}
        \subfloat[CDE$\uparrow$]{\includegraphics[width=0.23\textwidth]{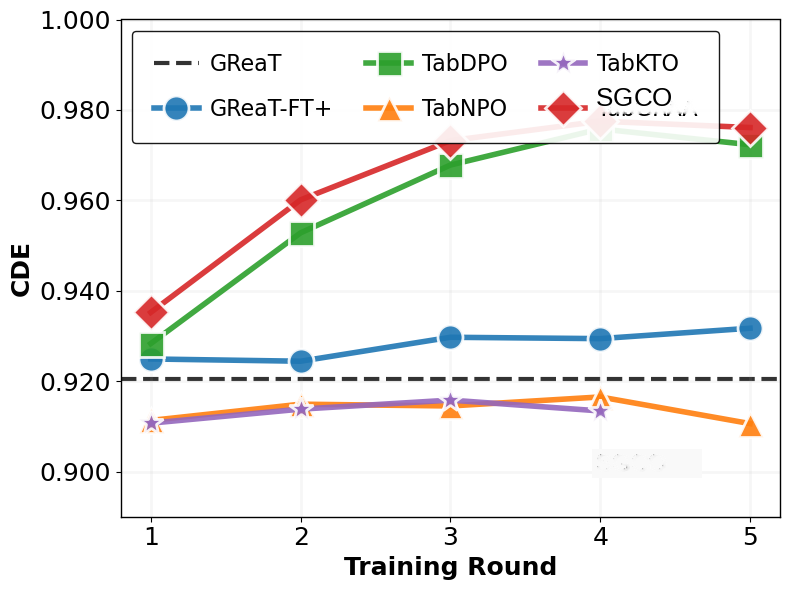}} & 
        \subfloat[PCC$\uparrow$]{\includegraphics[width=0.23\textwidth]{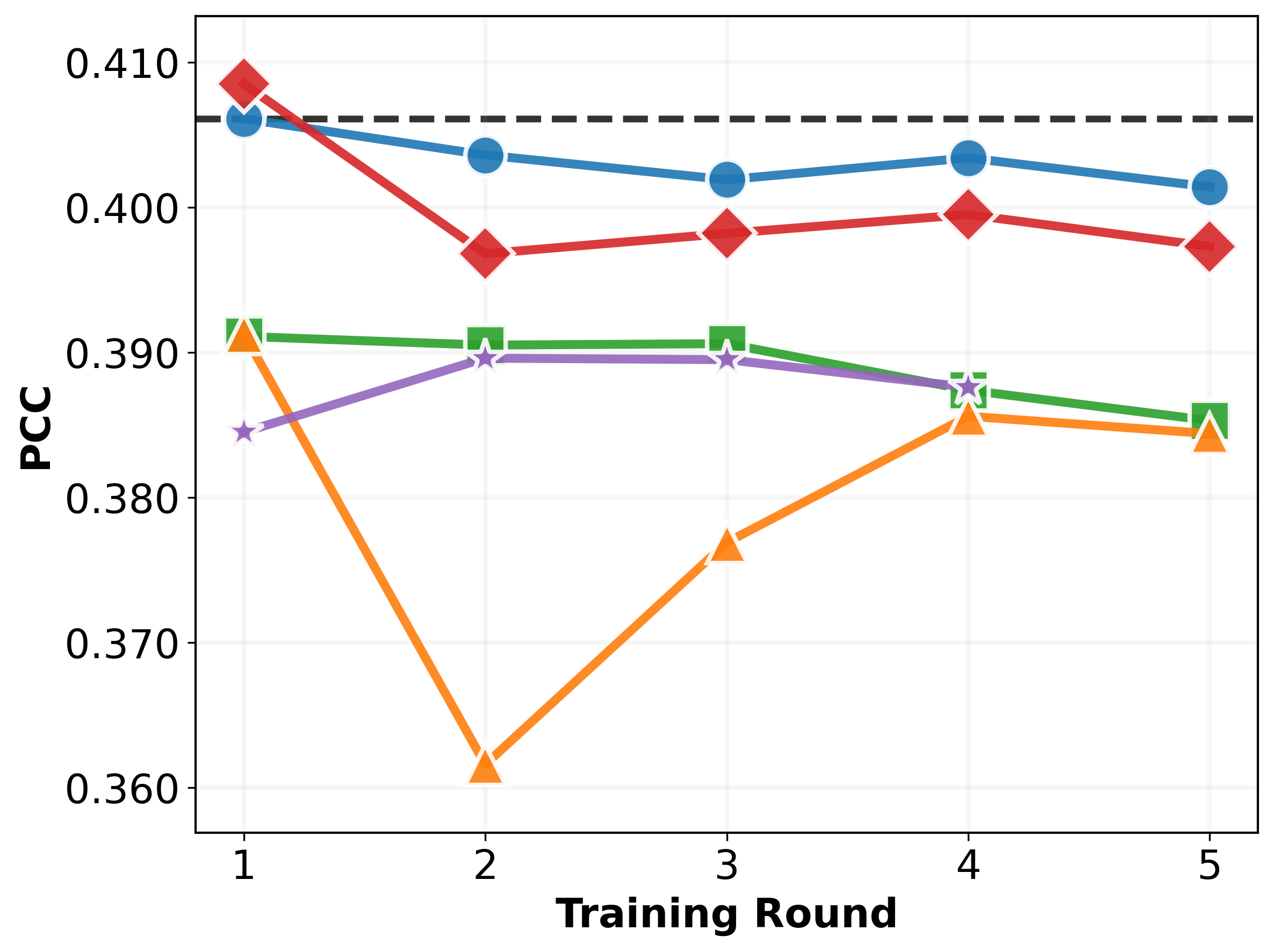}} & 
        \subfloat[$\alpha\uparrow$]{\includegraphics[width=0.23\textwidth]{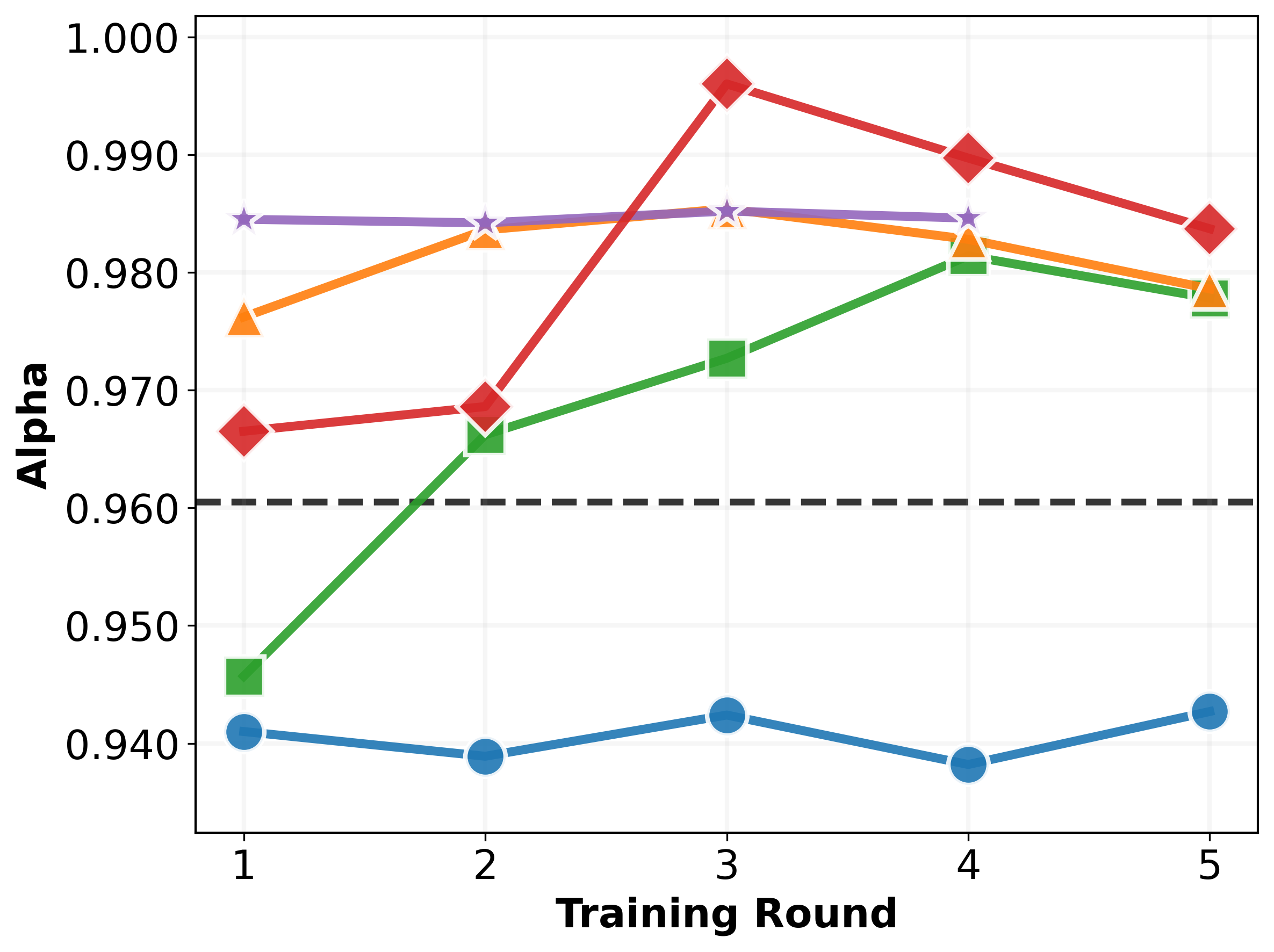}} &
        \subfloat[$\beta\uparrow$]{\includegraphics[width=0.23\textwidth]{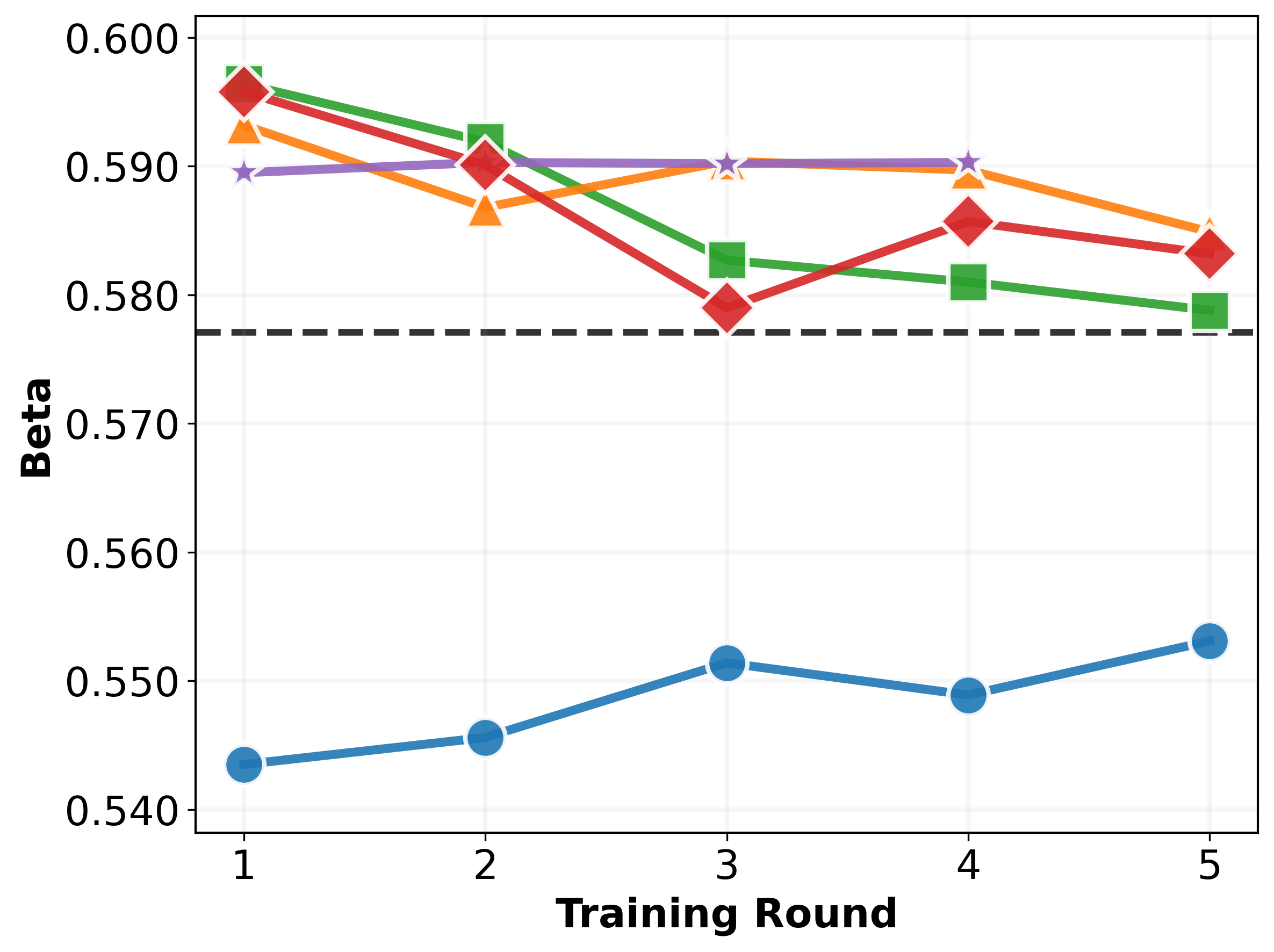}} \\ [1.5ex]
        \subfloat[C2ST$\uparrow$]{\includegraphics[width=0.23\textwidth]{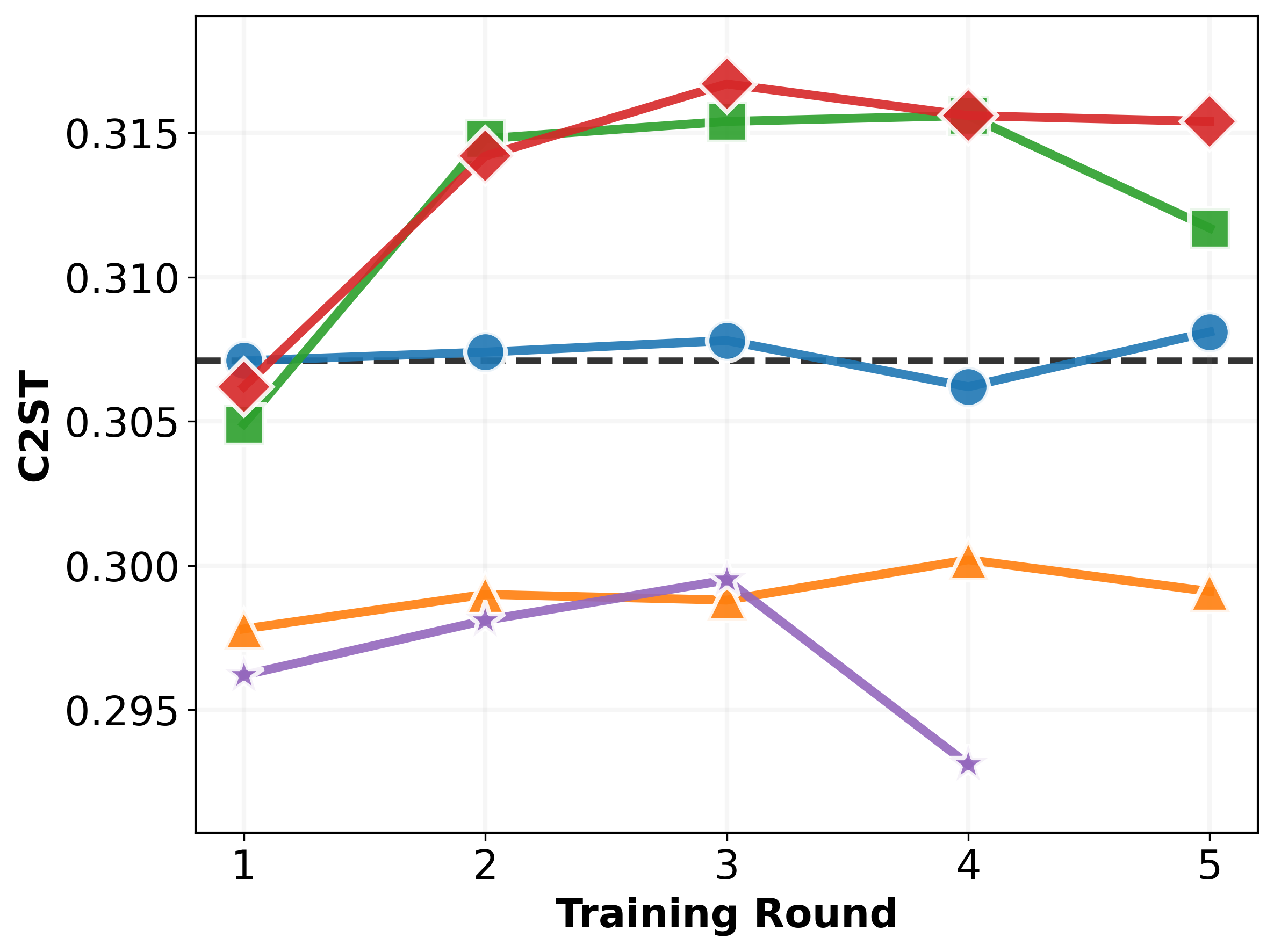}} &
        \subfloat[DA(AUC)$\to0.5$]{\includegraphics[width=0.23\textwidth]{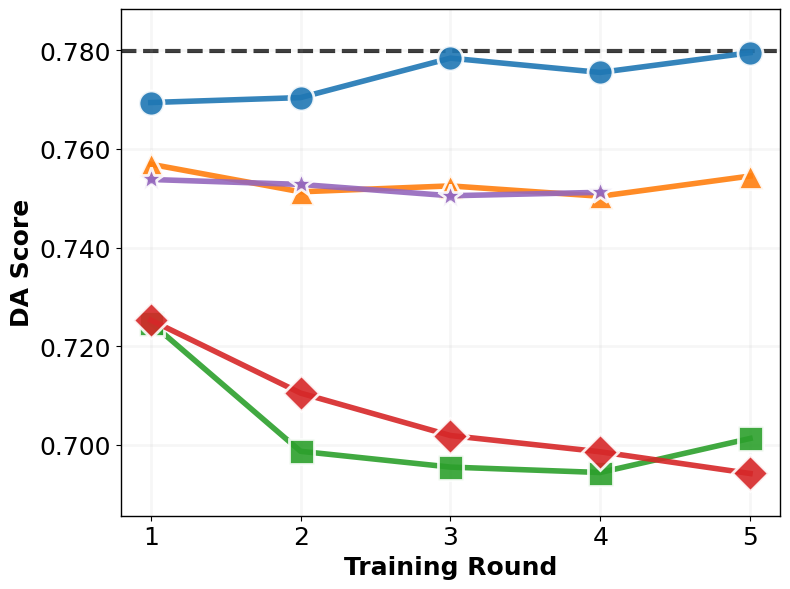}} &
        \subfloat[MLE$\downarrow$]{\includegraphics[width=0.23\textwidth]{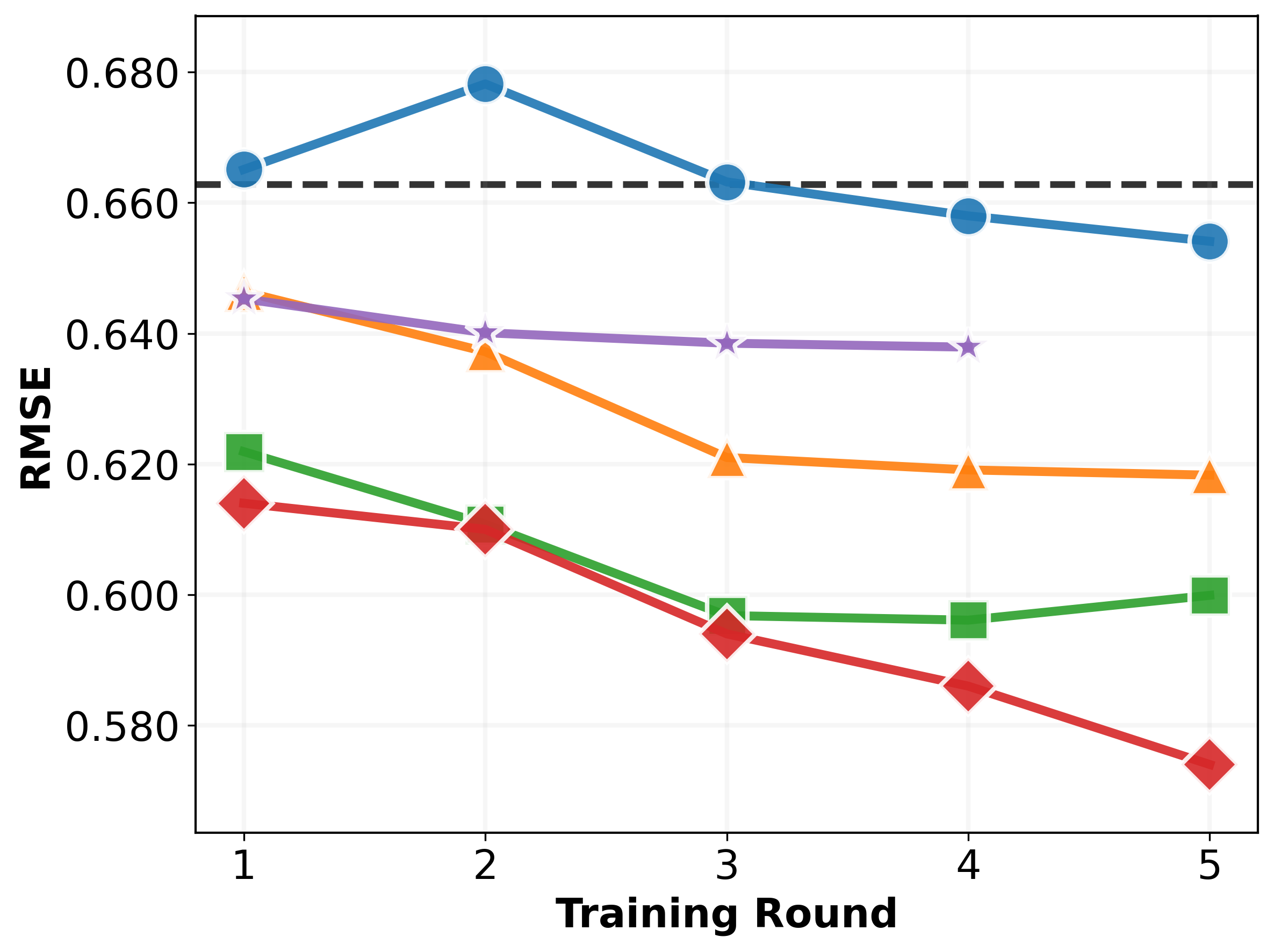}} &
        \\ [1.5ex]
    \end{tabular}
    \caption{Iterative performance progression across training rounds 1--5 on Beijing.}
    \label{fig:iterative_beijing}
\end{figure*}

\section{Comparison with non-LM generators}
\label{app:sota}

We compare SGCO with GAN-, VAE-, and diffusion-based tabular
generators. For SGCO, we report the best-performing loss
variant for each dataset and metric across ten rounds. Compared with
GReaT, SGCO achieves lower Shape and Trend errors and
higher $\alpha$-precision and $\beta$-recall across all
five datasets. Table~17 further reports improvements in
downstream utility, while Table~18 shows that SGCO moves
distinguishability AUC closer to 0.5. Diffusion-based models
remain stronger on several Shape and Trend comparisons,
while SGCO provides competitive fidelity and utility results.

\begin{table*}[t]
\centering
\caption{\textbf{Shape error comparison (10-round SGCO).} Performance comparison on Shape error rates (\%). Lower values are better.}
\label{tbl:exp-shape-updated}
\begin{threeparttable}
\scriptsize
\resizebox{\textwidth}{!}{
\begin{tabular}{lccccc}
\toprule
\textbf{Method} & \textbf{Adult} & \textbf{Default} & \textbf{Shoppers} & \textbf{Magic} & \textbf{Beijing} \\
\midrule
\multicolumn{6}{l}{\textit{GAN/VAE-based models}} \\
CTGAN & $16.84{\scriptstyle\pm0.03}$ & $16.83{\scriptstyle\pm0.04}$ & $21.15{\scriptstyle\pm0.10}$ & $9.81{\scriptstyle\pm0.08}$ & $21.39{\scriptstyle\pm0.05}$ \\
TVAE  & $14.22{\scriptstyle\pm0.08}$ & $10.17{\scriptstyle\pm0.05}$ & $24.51{\scriptstyle\pm0.06}$ & $8.25{\scriptstyle\pm0.06}$ & $19.16{\scriptstyle\pm0.06}$ \\
\midrule
\multicolumn{6}{l}{\textit{Diffusion-based models}} \\
TabDDPM & $1.75{\scriptstyle\pm0.03}$ & $1.57{\scriptstyle\pm0.08}$ & $2.72{\scriptstyle\pm0.13}$ & $1.01{\scriptstyle\pm0.09}$ & $1.30{\scriptstyle\pm0.03}$ \\
TabSyn  & $0.91{\scriptstyle\pm0.07}$ & $\mathbf{1.21}{\scriptstyle\pm0.09}$ & $1.51{\scriptstyle\pm0.05}$ & $1.09{\scriptstyle\pm0.09}$ & $1.45{\scriptstyle\pm0.04}$ \\ 
TabDiff & $\mathbf{0.63}{\scriptstyle\pm0.05}$ & $1.24{\scriptstyle\pm0.07}$ & $\mathbf{1.28}{\scriptstyle\pm0.09}$ & $\mathbf{0.78}{\scriptstyle\pm0.08}$ & $1.03{\scriptstyle\pm0.05}$ \\
\midrule
\multicolumn{6}{l}{\textit{LM-based models}} \\
GReaT   & $12.12{\scriptstyle\pm0.04}$ & $19.94{\scriptstyle\pm0.06}$ & $14.51{\scriptstyle\pm0.12}$ & $16.16{\scriptstyle\pm0.09}$ & $8.25{\scriptstyle\pm0.12}$ \\
SGCO & $0.87{\scriptstyle\pm0.11}$ & $5.36{\scriptstyle\pm0.14}$ & $9.63{\scriptstyle\pm0.21}$ & $4.42{\scriptstyle\pm0.09}$ & $\mathbf{1.02}{\scriptstyle\pm0.10}$ \\
\bottomrule
\end{tabular}}
\end{threeparttable}
\end{table*}

\begin{table*}[t]
\centering
\caption{\textbf{Trend error comparison (10-round SGCO).} Performance comparison on Trend error rates (\%). Lower values are better.}
\label{tbl:exp-trend-updated}
\begin{threeparttable}
\scriptsize
\resizebox{\textwidth}{!}{
\begin{tabular}{lccccc}
\toprule
\textbf{Method} & \textbf{Adult} & \textbf{Default} & \textbf{Shoppers} & \textbf{Magic} & \textbf{Beijing} \\
\midrule
\multicolumn{6}{l}{\textit{GAN/VAE-based models}} \\
CTGAN & $20.23{\scriptstyle\pm1.20}$ & $26.95{\scriptstyle\pm0.93}$ & $13.08{\scriptstyle\pm0.16}$ & $7.00{\scriptstyle\pm0.19}$ & $22.95{\scriptstyle\pm0.08}$ \\
TVAE  & $14.15{\scriptstyle\pm0.88}$ & $19.50{\scriptstyle\pm0.95}$ & $18.67{\scriptstyle\pm0.38}$ & $5.82{\scriptstyle\pm0.49}$ & $18.01{\scriptstyle\pm0.08}$ \\
\midrule
\multicolumn{6}{l}{\textit{Diffusion-based models}} \\
TabDDPM & $3.01{\scriptstyle\pm0.25}$ & $4.89{\scriptstyle\pm0.10}$ & $6.61{\scriptstyle\pm0.16}$ & $1.70{\scriptstyle\pm0.22}$ & $2.71{\scriptstyle\pm0.09}$ \\
TabSyn  & $1.93{\scriptstyle\pm0.07}$ & $2.81{\scriptstyle\pm0.48}$ & $2.13{\scriptstyle\pm0.10}$ & $0.88{\scriptstyle\pm0.18}$ & $3.13{\scriptstyle\pm0.34}$ \\ 
TabDiff & $\mathbf{1.49}{\scriptstyle\pm0.16}$ & $\mathbf{2.55}{\scriptstyle\pm0.75}$ & $\mathbf{1.74}{\scriptstyle\pm0.08}$ & $\mathbf{0.76}{\scriptstyle\pm0.12}$ & $\mathbf{2.59}{\scriptstyle\pm0.15}$ \\
\midrule
\multicolumn{6}{l}{\textit{LM-based models}} \\
GReaT   & $17.59{\scriptstyle\pm0.22}$ & $70.02{\scriptstyle\pm0.12}$ & $45.16{\scriptstyle\pm0.18}$ & $10.23{\scriptstyle\pm0.40}$ & $59.60{\scriptstyle\pm0.55}$ \\
SGCO & $9.38{\scriptstyle\pm0.15}$ & $69.09{\scriptstyle\pm0.21}$ & $44.79{\scriptstyle\pm0.31}$ & $7.88{\scriptstyle\pm0.45}$ & $59.25{\scriptstyle\pm0.49}$ \\
\bottomrule
\end{tabular}}
\end{threeparttable}
\end{table*}

\begin{table*}[h!]
\centering
\caption{\textbf{$\alpha$-Precision comparison across five datasets(10-round SGCO). For SGCO, we report the best result among the evaluated loss variants for each dataset. Higher values are better.} Higher values are better.}
\label{tbl:exp-alpha2-updated}
\begin{threeparttable}
\scriptsize
\resizebox{\textwidth}{!}{
\begin{tabular}{lccccc}
\toprule
\textbf{Method} & \textbf{Adult} & \textbf{Default} & \textbf{Shoppers} & \textbf{Magic} & \textbf{Beijing} \\
\midrule
\multicolumn{6}{l}{\textit{GAN/VAE-based models}} \\
CTGAN & $77.74{\scriptstyle\pm0.15}$ & $62.08{\scriptstyle\pm0.08}$ & $76.97{\scriptstyle\pm0.39}$ & $86.90{\scriptstyle\pm0.22}$ & $96.27{\scriptstyle\pm0.14}$ \\
TVAE  & $98.17{\scriptstyle\pm0.17}$ & $85.57{\scriptstyle\pm0.34}$ & $58.19{\scriptstyle\pm0.26}$ & $86.19{\scriptstyle\pm0.48}$ & $97.20{\scriptstyle\pm0.10}$ \\
\midrule
\multicolumn{6}{l}{\textit{Diffusion-based models}} \\
TabDDPM & $96.39{\scriptstyle\pm0.20}$ & $97.59{\scriptstyle\pm0.36}$ & $88.55{\scriptstyle\pm0.68}$ & $98.59{\scriptstyle\pm0.17}$ & $97.93{\scriptstyle\pm0.30}$ \\
TabSyn  & $99.39{\scriptstyle\pm0.18}$ & $\mathbf{98.65}{\scriptstyle\pm0.23}$ & $98.36{\scriptstyle\pm0.52}$ & $\mathbf{99.42}{\scriptstyle\pm0.28}$ & $87.51{\scriptstyle\pm0.24}$ \\
TabDiff & $99.02{\scriptstyle\pm0.20}$ & $98.49{\scriptstyle\pm0.28}$ & $99.11{\scriptstyle\pm0.34}$ & $\mathbf{99.42}{\scriptstyle\pm0.21}$ & $98.06{\scriptstyle\pm0.24}$ \\
\midrule
\multicolumn{6}{l}{\textit{LM-based models}} \\
GReaT   & $55.79{\scriptstyle\pm0.03}$ & $85.90{\scriptstyle\pm0.17}$ & $78.88{\scriptstyle\pm0.13}$ & $85.46{\scriptstyle\pm0.54}$ & $98.32{\scriptstyle\pm0.22}$ \\
SGCO & $\mathbf{99.44}{\scriptstyle\pm0.23}$ & $96.25{\scriptstyle\pm0.42}$ & $\mathbf{99.61}{\scriptstyle\pm0.31}$ & $95.26{\scriptstyle\pm0.17}$ & $\mathbf{99.51}{\scriptstyle\pm0.28}$ \\
\bottomrule
\end{tabular}}
\end{threeparttable}
\end{table*}

\begin{table*}[h!]
\centering
\caption{$\beta$-Recall comparison across five datasets(10-round SGCO). For SGCO, we report the best result among the evaluated loss variants for each dataset. Higher values are better.}
\label{tbl:exp-beta-updated}
\begin{threeparttable}
\scriptsize
\resizebox{\textwidth}{!}{
\begin{tabular}{lccccc}
\toprule
\textbf{Method} & \textbf{Adult} & \textbf{Default} & \textbf{Shoppers} & \textbf{Magic} & \textbf{Beijing} \\
\midrule
\multicolumn{6}{l}{\textit{GAN/VAE-based models}} \\
CTGAN & $30.80{\scriptstyle\pm0.20}$ & $18.22{\scriptstyle\pm0.17}$ & $31.80{\scriptstyle\pm0.35}$ & $11.75{\scriptstyle\pm0.20}$ & $34.80{\scriptstyle\pm0.10}$ \\
TVAE  & $38.87{\scriptstyle\pm0.31}$ & $23.13{\scriptstyle\pm0.11}$ & $19.78{\scriptstyle\pm0.10}$ & $32.44{\scriptstyle\pm0.35}$ & $28.45{\scriptstyle\pm0.08}$ \\
\midrule
\multicolumn{6}{l}{\textit{Diffusion-based models}} \\
TabDDPM & $47.05{\scriptstyle\pm0.25}$ & $47.83{\scriptstyle\pm0.35}$ & $47.79{\scriptstyle\pm0.25}$ & $\mathbf{48.46}{\scriptstyle\pm0.42}$ & $56.92{\scriptstyle\pm0.13}$ \\
TabSyn  & $47.92{\scriptstyle\pm0.23}$ & $46.45{\scriptstyle\pm0.35}$ & $49.10{\scriptstyle\pm0.60}$ & $48.03{\scriptstyle\pm0.50}$ & $59.15{\scriptstyle\pm0.22}$ \\ 
TabDiff & $51.64{\scriptstyle\pm0.20}$ & $\mathbf{51.09}{\scriptstyle\pm0.25}$ & $49.75{\scriptstyle\pm0.64}$ & $48.01{\scriptstyle\pm0.31}$ & $\mathbf{59.63}{\scriptstyle\pm0.23}$ \\
\midrule
\multicolumn{6}{l}{\textit{LM-based models}} \\
GReaT   & $49.12{\scriptstyle\pm0.18}$ & $42.04{\scriptstyle\pm0.19}$ & $44.90{\scriptstyle\pm0.17}$ & $34.91{\scriptstyle\pm0.28}$ & $43.34{\scriptstyle\pm0.31}$ \\
SGCO & $\mathbf{52.14}{\scriptstyle\pm0.22}$ & $46.47{\scriptstyle\pm0.32}$ & $\mathbf{50.22}{\scriptstyle\pm0.35}$ & $47.48{\scriptstyle\pm0.26}$ & $59.04{\scriptstyle\pm0.52}$ \\
\bottomrule
\end{tabular}}
\end{threeparttable}
\end{table*}

\begin{table*}[h!]
\centering
\caption{\textbf{Machine learning efficiency comparison (10-round SGCO).}
For classification datasets, higher AUC is better; for Beijing, lower RMSE is
better. The Real row is included as a reference, and bold values mark the best
synthetic generator in each column.}
\label{tbl:exp-mle-updated}
\begin{threeparttable}
\scriptsize
\resizebox{\textwidth}{!}{
\begin{tabular}{lccccc}
\toprule
\textbf{Method} & \textbf{Adult ($\uparrow$)} & \textbf{Default ($\uparrow$)} & \textbf{Shoppers ($\uparrow$)} & \textbf{Magic ($\uparrow$)} & \textbf{Beijing ($\downarrow$)} \\
\midrule
\multicolumn{6}{l}{\textit{Real data}} \\
Real & $0.927{\scriptstyle\pm0.000}$ & $0.770{\scriptstyle\pm0.005}$ & $0.926{\scriptstyle\pm0.001}$ & $0.946{\scriptstyle\pm0.001}$ & $0.423{\scriptstyle\pm0.003}$ \\
\midrule
\multicolumn{6}{l}{\textit{GAN/VAE-based models}} \\
CTGAN & $0.886{\scriptstyle\pm0.002}$ & $0.696{\scriptstyle\pm0.04}$ & $0.875{\scriptstyle\pm0.009}$ & $0.855{\scriptstyle\pm0.006}$ & $0.902{\scriptstyle\pm0.019}$ \\
TVAE  & $0.878{\scriptstyle\pm0.004}$ & $0.724{\scriptstyle\pm0.005}$ & $0.871{\scriptstyle\pm0.006}$ & $0.887{\scriptstyle\pm0.003}$ & $0.770{\scriptstyle\pm0.011}$ \\
\midrule
\multicolumn{6}{l}{\textit{Diffusion-based models}} \\
TabDDPM & $0.907{\scriptstyle\pm0.001}$ & $0.758{\scriptstyle\pm0.004}$ & $0.918{\scriptstyle\pm0.005}$ & $0.935{\scriptstyle\pm0.003}$ & $0.580{\scriptstyle\pm0.009}$ \\
TabSyn  & $0.909{\scriptstyle\pm0.001}$ & $0.763{\scriptstyle\pm0.005}$ & $0.914{\scriptstyle\pm0.004}$ & $\mathbf{0.937}{\scriptstyle\pm0.002}$ & $0.580{\scriptstyle\pm0.009}$ \\ 
TabDiff & $0.912{\scriptstyle\pm0.002}$ & $0.763{\scriptstyle\pm0.005}$ & $\mathbf{0.921}{\scriptstyle\pm0.004}$ & $0.936{\scriptstyle\pm0.003}$ & $\mathbf{0.555}{\scriptstyle\pm0.013}$ \\
\midrule
\multicolumn{6}{l}{\textit{LM-based models}} \\
GReaT   & $0.913{\scriptstyle\pm0.003}$ & $0.755{\scriptstyle\pm0.006}$ & $0.902{\scriptstyle\pm0.005}$ & $0.888{\scriptstyle\pm0.008}$ & $0.653{\scriptstyle\pm0.013}$ \\
SGCO & $\mathbf{0.921}{\scriptstyle\pm0.003}$ & $\mathbf{0.784}{\scriptstyle\pm0.005}$ & $0.911{\scriptstyle\pm0.007}$ & $0.903{\scriptstyle\pm0.001}$ & $0.574{\scriptstyle\pm0.003}$ \\
\bottomrule
\end{tabular}}
\end{threeparttable}
\end{table*}

\begin{table*}[t]
\centering
\caption{\textbf{Distinguishability Attack (DA) AUC comparison (10-round SGCO).} Values closer to $0.5$ indicate weaker real-vs-synthetic distinguishability
attacks. Bold values mark the closest value to $0.5$ in each column.}
\label{tbl:exp-DA-updated}
\begin{threeparttable}
\scriptsize
\resizebox{\textwidth}{!}{
\begin{tabular}{lccccc}
\toprule
\textbf{Method} & \textbf{Adult} & \textbf{Default} & \textbf{Shoppers} & \textbf{Magic} & \textbf{Beijing} \\
\midrule
\multicolumn{6}{l}{\textit{GAN/VAE-based models}} \\
CTGAN & $1.000{\scriptstyle\pm0.000}$ & $0.999{\scriptstyle\pm0.000}$ & $1.000{\scriptstyle\pm0.000}$ & $1.000{\scriptstyle\pm0.000}$ & $1.000{\scriptstyle\pm0.000}$ \\
TVAE  & $0.998{\scriptstyle\pm0.000}$ & $1.000{\scriptstyle\pm0.000}$ & $1.000{\scriptstyle\pm0.000}$ & $0.998{\scriptstyle\pm0.000}$ & $0.999{\scriptstyle\pm0.000}$ \\
\midrule
\multicolumn{6}{l}{\textit{Diffusion-based models}} \\
TabDDPM & $0.428{\scriptstyle\pm0.004}$ & $0.696{\scriptstyle\pm0.003}$ & $0.684{\scriptstyle\pm0.083}$ & $0.412{\scriptstyle\pm0.091}$ & $\mathbf{0.466}{\scriptstyle\pm0.006}$ \\
TabSyn  & $0.543{\scriptstyle\pm0.005}$ & $0.736{\scriptstyle\pm0.008}$ & $0.583{\scriptstyle\pm0.004}$ & $0.562{\scriptstyle\pm0.006}$ & $0.715{\scriptstyle\pm0.005}$ \\ 
TabDiff & $0.481{\scriptstyle\pm0.003}$ & $\mathbf{0.644}{\scriptstyle\pm0.006}$ & $\mathbf{0.545}{\scriptstyle\pm0.005}$ & $\mathbf{0.512}{\scriptstyle\pm0.001}$ & $0.658{\scriptstyle\pm0.003}$ \\
\midrule
\multicolumn{6}{l}{\textit{LM-based models}} \\
GReaT   & $0.720{\scriptstyle\pm0.001}$ & $0.863{\scriptstyle\pm0.004}$ & $0.832{\scriptstyle\pm0.003}$ & $0.859{\scriptstyle\pm0.005}$ & $0.748{\scriptstyle\pm0.002}$ \\
SGCO & $\mathbf{0.499}{\scriptstyle\pm0.003}$ & $0.760{\scriptstyle\pm0.004}$ & $0.727{\scriptstyle\pm0.006}$ & $0.7034{\scriptstyle\pm0.008}$ & $0.6942{\scriptstyle\pm0.03}$ \\
\bottomrule
\end{tabular}}
\end{threeparttable}
\end{table*}

\section{Additional component ablations}
\label{app:ablation}

\subsection{Classifier variants and retraining strategy}
\label{app:classifier_variants}


The default reward uses a Random Forest real-versus-synthetic discriminator
refitted against the current generator after every policy update. The
fixed-versus-refreshed comparison is presented in the main text
(Figure~\ref{fig:training}); this appendix reports the complementary classifier
choice analysis. Figure~\ref{fig:classifier_variants_comparison} compares
Random Forest and XGBoost-based feedback across datasets, showing that the
relative performance of the reward models depends on the dataset and metric.






\begin{figure*}[h!]
\centering
\begin{minipage}[t]{0.18\textwidth}
\centering
\includegraphics[width=\linewidth]{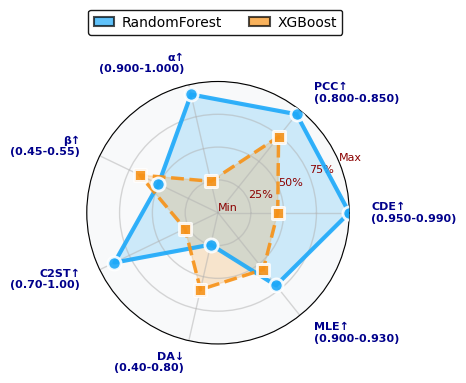}
\end{minipage}
\hfill
\begin{minipage}[t]{0.18\textwidth}
\centering
\includegraphics[width=\linewidth]{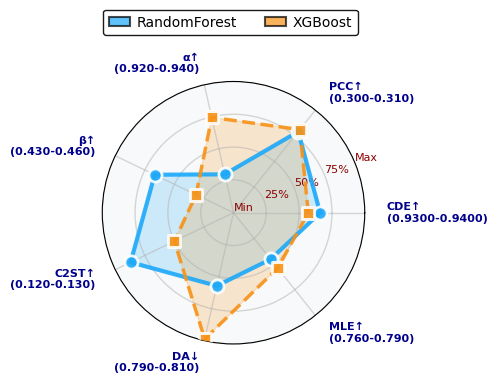}
\end{minipage}
\hfill
\begin{minipage}[t]{0.18\textwidth}
\centering
\includegraphics[width=\linewidth]{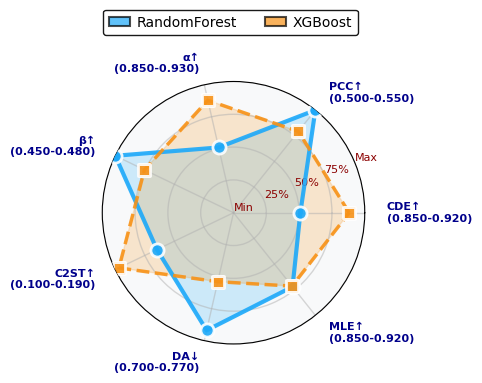}
\end{minipage}
\hfill
\begin{minipage}[t]{0.18\textwidth}
\centering
\includegraphics[width=\linewidth]{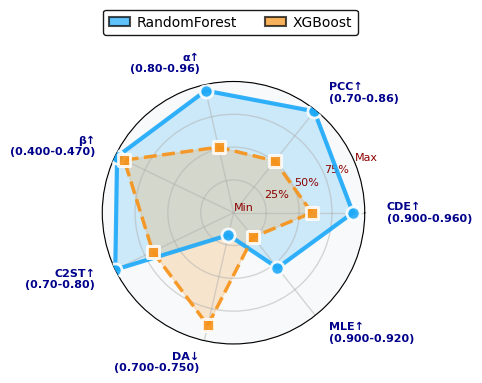}
\end{minipage}
\hfill
\begin{minipage}[t]{0.18\textwidth}
\centering
\includegraphics[width=\linewidth]{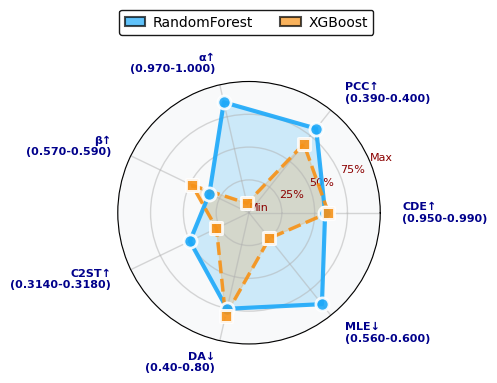}
\end{minipage}

\vspace{5pt}
(a) Adult \hfill (b) Default \hfill (c) Shoppers \hfill (d) Magic \hfill (e) Beijing

\caption{\textbf{Classifier reward variants across datasets.}
Comparison of classifier-based reward variants across five datasets. Values are scaled for visualization.}
\label{fig:classifier_variants_comparison}
\end{figure*}

\subsection{Beta sensitivity}
\label{app:beta}

\begin{figure*}[h]
    \centering
    \includegraphics[width=\linewidth]{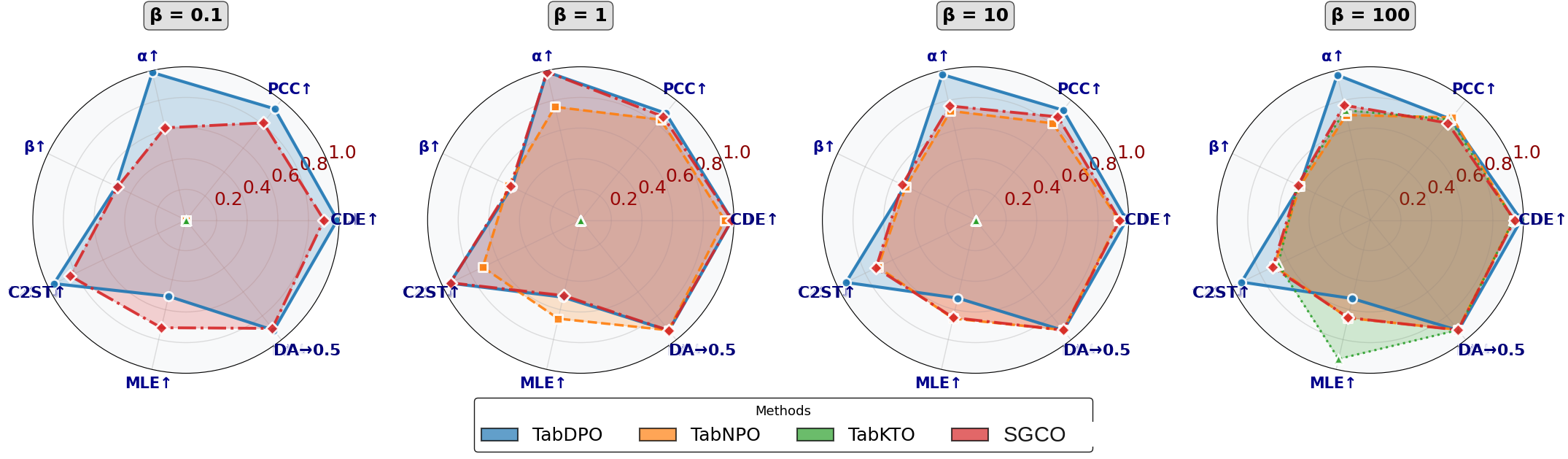}
    \caption{\textbf{Alignment-strength sensitivity on Adult.}
    Radar plot comparing tabular LM post-training methods across $\beta\in\{0.1,1,10,100\}$. Arrows indicate the preferred direction for each metric.}
    \label{fig:beta_comparison_radar_adult}
\end{figure*}

\begin{figure}[h]
    \centering
    \includegraphics[width=\linewidth]{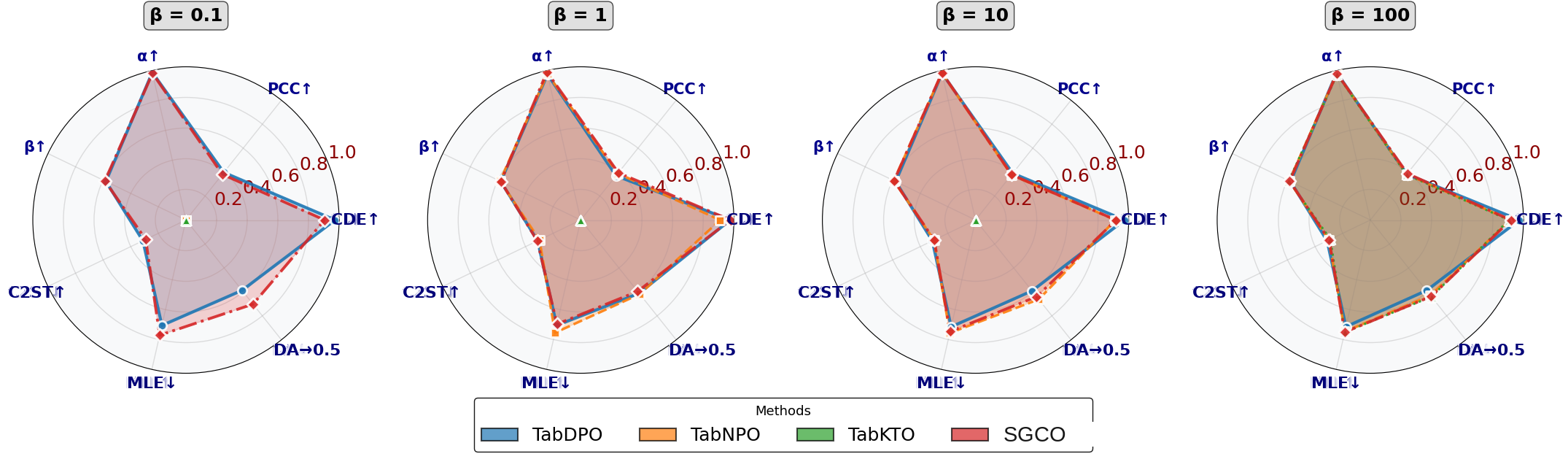}
    \caption{\textbf{Alignment-strength sensitivity on Beijing.}
    Radar plot comparing tabular LM post-training methods across $\beta\in\{0.1,1,10,100\}$. Arrows indicate the preferred direction for each metric.}
    \label{fig:beijing_radar}
\end{figure}

\begin{figure}[h]
    \centering
    \includegraphics[width=\linewidth]{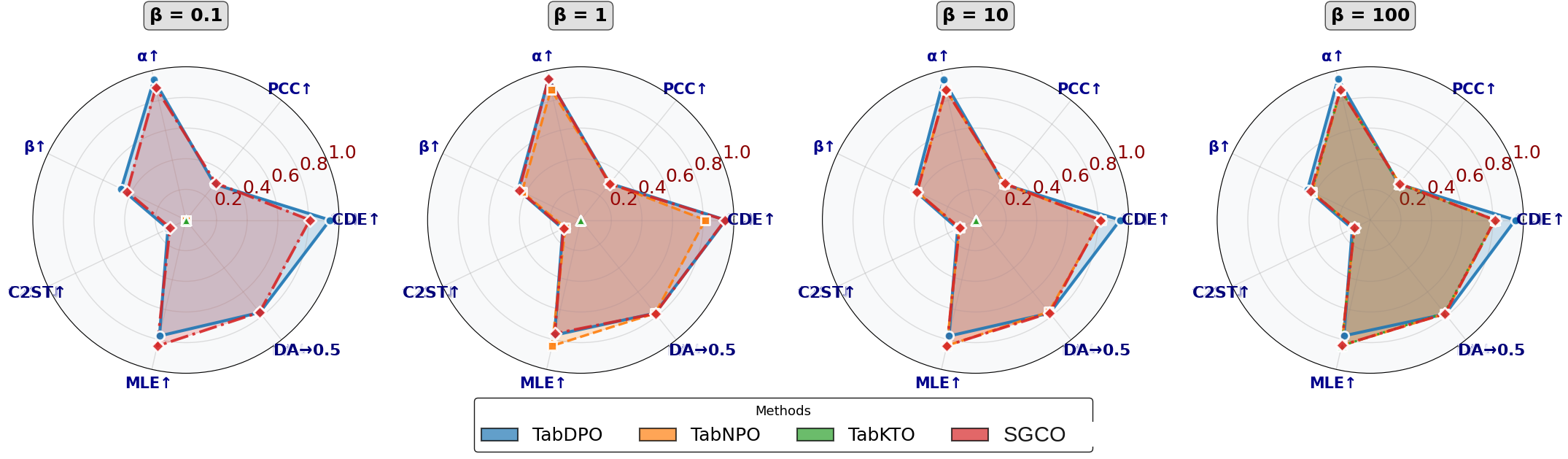}
    \caption{\textbf{Alignment-strength sensitivity on Default.}
    Radar plot comparing tabular LM post-training methods across $\beta\in\{0.1,1,10,100\}$. Arrows indicate the preferred direction for each metric.}
    \label{fig:default_radar}
\end{figure}

\begin{figure}[h]
    \centering
    \includegraphics[width=\linewidth]{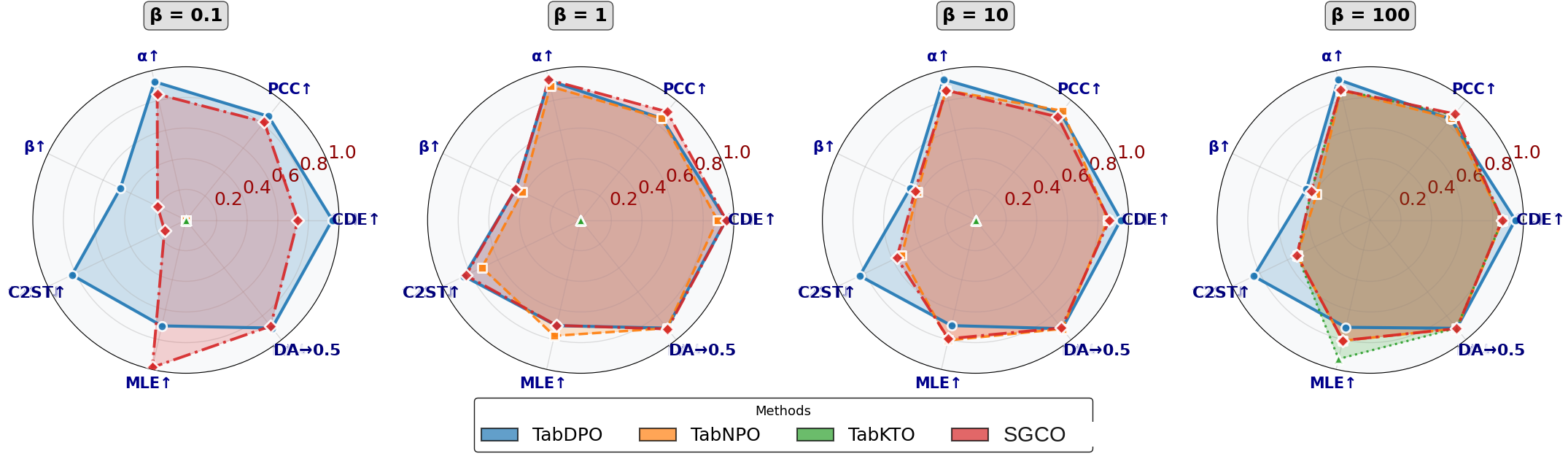}
    \caption{\textbf{Alignment-strength sensitivity on Magic.}
    Radar plot comparing tabular LM post-training methods across $\beta\in\{0.1,1,10,100\}$. Arrows indicate the preferred direction for each metric.}
    \label{fig:magic_radar}
\end{figure}

\begin{figure}[h]
    \centering
    \includegraphics[width=\linewidth]{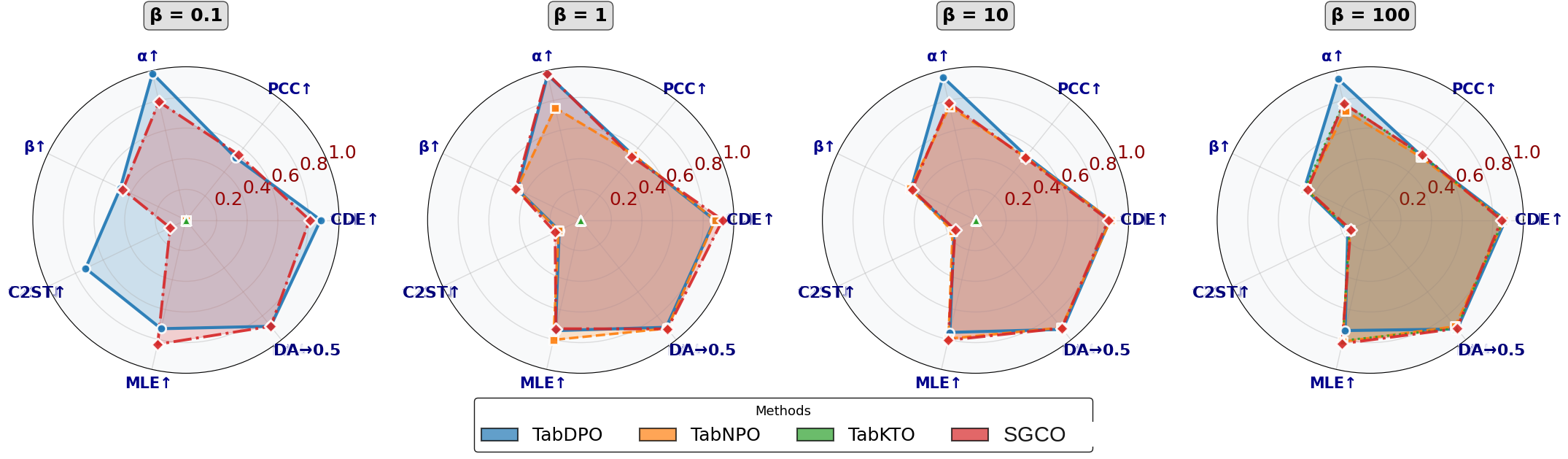}
    \caption{\textbf{Alignment-strength sensitivity on Shoppers.}
    Radar plot comparing tabular LM post-training methods across $\beta\in\{0.1,1,10,100\}$. Arrows indicate the preferred direction for each metric.}
    \label{fig:shoppers_radar}
\end{figure}

We evaluate alignment strength $\beta$ across $\{0.1,1,10,100\}$, with results shown in Figures~\ref{fig:beta_comparison_radar_adult}--\ref{fig:shoppers_radar}. SGCO is less sensitive to $\beta$ than negative-only or asymmetric alignment baselines. Extremely small $\beta$ gives weak updates, while overly large $\beta$ can destabilize post-training or overemphasize the reward contrast.

\subsection{Backbone architecture}
\label{app:architecture}

Table~\ref{tbl:architecture} evaluates SGCO with DistilGPT-2, GPT-2, and GPT-Neo backbones on Adult. SGCO remains effective across all three autoregressive LMs, showing that the method is not tied to a single backbone. The strongest metric differs by architecture: DistilGPT-2 performs best on $\alpha$-precision, GPT-Neo provides strong CDE/PCC/DA, and GPT-2 performs best on C2ST and slightly best on MLE. This supports the view that SGCO is a post-training framework for tabular language models rather than a backbone-specific tuning trick.

\begin{table*}[t!]
\centering
\caption{SGCO implemented across diverse language model architectures on Adult dataset. Best results are in \bf{bold}.}
\label{tbl:architecture}
\begin{threeparttable}
\resizebox{\textwidth}{!}{
\begin{tabular}{lccccccc}
\toprule[1pt]
\textbf{Model (Params)} & \textbf{CDE$\uparrow$} & \textbf{PCC$\uparrow$} & \textbf{$\alpha$$\uparrow$} & \textbf{$\beta$$\uparrow$} & \textbf{C2ST$\uparrow$} & \textbf{DA $\to 0.5$} & \textbf{MLE$\uparrow$} \\
\midrule

DistilGPT-2 (82M)  & $98.85{\scriptstyle\pm0.03}$ & $90.15{\scriptstyle\pm1.64}$ & $\bf{99.57}{\scriptstyle\pm0.12}$ & $50.25{\scriptstyle\pm0.33}$ & $95.76{\scriptstyle\pm0.41}$ & $0.5038{\scriptstyle\pm0.0043}$ & $0.9225{\scriptstyle\pm0.0015}$ \\

GPT-NEO (125M)  & $\bf{99.09}{\scriptstyle\pm0.05}$ & $\bf{90.53}{\scriptstyle\pm1.26}$ & $98.19{\scriptstyle\pm0.32}$ & $\bf{52.79}{\scriptstyle\pm0.35}$ & $95.48{\scriptstyle\pm0.27}$ & $\bf{0.4961}{\scriptstyle\pm0.0033}$ & $0.9188{\scriptstyle\pm0.0025}$ \\

GPT2 (124M)  & $99.03{\scriptstyle\pm0.04}$ & $85.53{\scriptstyle\pm1.57}$ & $98.98{\scriptstyle\pm0.41}$ & $50.21{\scriptstyle\pm0.22}$ & $\bf{96.09}{\scriptstyle\pm0.35}$ & $0.5179{\scriptstyle\pm0.0022}$ & $\bf{0.9228}{\scriptstyle\pm0.0032}$ \\

\bottomrule[1pt]
\end{tabular}}
\end{threeparttable}
\end{table*}

\subsection{Computational cost}
\label{app:compute}

We report wall-clock cost on a single NVIDIA RTX 4090. On Adult with group size $B=4$, one SGCO iteration takes approximately 12 minutes end-to-end: roughly 6 minutes for synthetic sampling, 20 seconds for reward scoring, 5 seconds for group construction, and 5 minutes for the alignment update. The autoregressive LM forward/backward pass therefore dominates runtime, while reward scoring is a small fraction of the total cost. The adapted alignment baselines have comparable per-iteration cost because they share the same sampling, scoring, and LM-update pipeline. Five SGCO iterations take approximately 60 minutes, compared with 4--6 hours for full GReaT supervised fine-tuning over 100 epochs, making iterative post-training substantially cheaper than retraining from scratch.

\subsection{Reward group construction}
\label{app:ablation_group}

We test whether SGCO requires explicit top-vs-bottom index matching when forming high- and low-reward groups. Table~\ref{tbl:ablation-group-membership} compares two group-selection strategies: random sampling within the top/bottom reward halves and explicit top-vs-bottom matching (T-vs-B). The two strategies perform similarly across datasets and metrics, indicating that SGCO mainly relies on separating high- and low-reward strata rather than on one-to-one pairwise correspondence.

\begin{table*}[t]
\centering
\caption{\textbf{Group-selection strategies for SGCO.} Random sampling within the top/bottom reward halves is compared with T-vs-B, where T-vs-B denotes explicit top-vs-bottom index matching. Results are shown after 5 iterations with $B=4$. Similar performance indicates that SGCO does not rely on pairwise correspondence. Lower Wasserstein/MMD/JSD is better.}
\label{tbl:ablation-group-membership}
\small
\setlength{\tabcolsep}{3.5pt}
\renewcommand{\arraystretch}{1.25}
\begin{tabular}{@{}lcccccc@{}}
\toprule
\textbf{Dataset}
& \multicolumn{2}{c}{\textbf{Wasserstein$\downarrow$}}
& \multicolumn{2}{c}{\textbf{MMD$\downarrow$}}
& \multicolumn{2}{c}{\textbf{JSD$\downarrow$}} \\
\cmidrule(lr){2-3}\cmidrule(lr){4-5}\cmidrule(l){6-7}
& \textbf{Random} & \textbf{T-vs-B}
& \textbf{Random} & \textbf{T-vs-B}
& \textbf{Random} & \textbf{T-vs-B} \\
\midrule
Adult
& $0.0316\pm 0.0062$ & $0.0278\pm 0.0004$
& $0.0014\pm 0.0002$ & $0.0012\pm 0.0002$
& $0.0033\pm 0.0003$ & $0.0034\pm 0.0000$ \\
\addlinespace[2pt]
Shoppers
& $0.2529\pm 0.0012$ & $0.2531\pm 0.0019$
& $0.1470\pm 0.0065$ & $0.1366\pm 0.0046$
& $0.1713\pm 0.0004$ & $0.1699\pm 0.0008$ \\
\addlinespace[2pt]
Beijing
& $0.3449\pm 0.0003$ & $0.3451\pm 0.0003$
& $0.0986\pm 0.0004$ & $0.0987\pm 0.0002$
& $0.2332\pm 0.0001$ & $0.2330\pm 0.0002$ \\
\addlinespace[2pt]
Default
& $0.4219\pm 0.0012$ & $0.4211\pm 0.0001$
& $0.3589\pm 0.0002$ & $0.3590\pm 0.0004$
& $0.2927\pm 0.0010$ & $0.2919\pm 0.0003$ \\
\addlinespace[2pt]
Magic
& $0.0329\pm 0.0090$ & $0.0328\pm 0.0099$
& $0.0041\pm 0.0005$ & $0.0041\pm 0.0005$
& $0.0170\pm 0.0018$ & $0.0148\pm 0.0027$ \\
\bottomrule
\end{tabular}
\end{table*}

\section{Additional comparison with GRPO}
\label{app:grpo-matched}

GRPO is a useful group-based baseline for TabRSI because both methods learn from several generated records at each update. We compare GReaT, GRPO, and SGCO on Adult, Default, Shoppers, Magic, and Beijing using matched initialization, reward construction, training rounds, and evaluation. The study used five training seeds and ten generated-output/evaluation repeats per method and dataset. The tables report the recorded output-evaluation means and standard deviations. The aggregate comparison in Table~\ref{tbl:grpo-main} summarizes the five dataset-level means.

\paragraph{Why group construction matters.} In tabular language modeling, generated rows usually share a common empty prompt. GRPO standardizes rewards within each sampled response group. As a result, the advantage of a row depends on the other rows sampled alongside it: a strong row can receive a negative advantage in a stronger group, while a weaker row can receive a positive advantage in a weaker group. SGCO instead ranks the full candidate population and samples high- and low-reward strata from that ranking. It then optimizes the mean policy/reference log-likelihood contrast between the strata. This uses a consistent population-level ranking within the round and does not require row-to-row pairing or within-group reward standardization.

\paragraph{Results.} GRPO's training reward increased during optimization, while its CDE and $\alpha$-precision showed smaller gains. SGCO achieves higher CDE and $\alpha$-precision on all five datasets. In the matched five-dataset summary, SGCO reaches 94.41 CDE and 96.99 $\alpha$-precision, compared with 87.44 and 84.96 for GRPO. Its results on other metrics vary across datasets: for example, GRPO achieves the highest $\beta$-recall on Adult and Beijing, while SGCO has better distributional distances on several datasets. Tables~\ref{tbl:grpo-comparison-full} present the full metric-level comparison.

\begin{table*}[t]
\centering
\caption{\textbf{GRPO comparison across five tabular datasets.} GReaT, GRPO, and SGCO are compared under matched training and evaluation settings. Values are mean $\pm$ standard deviation, multiplied by 100 and shown with two decimal places. Bold marks the best mean for each dataset and metric, including ties. MIA is ranked by closeness to 50.}
\label{tbl:grpo-comparison-full}
\begingroup
\fontsize{7.7}{9.2}\selectfont
\setlength{\tabcolsep}{1.6pt}
\renewcommand{\arraystretch}{1.35}
\resizebox{\textwidth}{!}{%
\begin{tabular}{@{}ll*{9}{c}@{}}
\toprule
\textbf{Dataset} & \textbf{Method} & \textbf{CDE$\uparrow$} & \textbf{PCC$\uparrow$} & \textbf{$\alpha\uparrow$} & \textbf{$\beta\uparrow$} & \textbf{Wass.$\downarrow$} & \textbf{MMD$\downarrow$} & \textbf{JSD$\downarrow$} & \textbf{MIA$\to50$} & \textbf{Utility$^{*}$} \\
\midrule
Adult & GReaT & $91.86\pm 0.33$ & $87.82\pm 1.54$ & $68.28\pm 0.15$ & $51.46\pm 0.23$ & $5.85\pm 0.00$ & $0.68\pm 0.00$ & $1.18\pm 0.00$ & $50.96\pm 0.00$ & $91.21\pm 0.17$ \\
 & GRPO & $93.93\pm 0.06$ & $\boldsymbol{88.53\pm 0.13}$ & $75.97\pm 0.19$ & $\boldsymbol{52.80\pm 0.24}$ & $4.55\pm 0.06$ & $0.36\pm 0.01$ & $0.72\pm 0.01$ & $\boldsymbol{50.83\pm 0.03}$ & $91.52\pm 0.24$ \\
 & SGCO & $\boldsymbol{96.28\pm 0.04}$ & $88.13\pm 5.14$ & $\boldsymbol{97.20\pm 0.22}$ & $52.67\pm 0.25$ & $\boldsymbol{2.78\pm 0.04}$ & $\boldsymbol{0.12\pm 0.02}$ & $\boldsymbol{0.34\pm 0.00}$ & $50.96\pm 0.13$ & $\boldsymbol{92.00\pm 0.10}$ \\
\addlinespace[2.5pt]
Default & GReaT & $80.46\pm 0.19$ & $30.80\pm 0.65$ & $86.32\pm 2.41$ & $41.73\pm 0.45$ & $42.13\pm 0.00$ & $\boldsymbol{35.90\pm 0.00}$ & $29.64\pm 0.00$ & $50.85\pm 0.00$ & $76.03\pm 0.34$ \\
 & GRPO & $79.50\pm 0.05$ & $\boldsymbol{30.92\pm 0.85}$ & $86.24\pm 0.09$ & $42.59\pm 0.29$ & $42.22\pm 0.01$ & $35.93\pm 0.01$ & $29.38\pm 0.12$ & $50.65\pm 0.04$ & $77.45\pm 0.47$ \\
 & SGCO & $\boldsymbol{93.78\pm 0.18}$ & $30.62\pm 0.45$ & $\boldsymbol{95.09\pm 0.16}$ & $\boldsymbol{45.97\pm 0.38}$ & $\boldsymbol{42.11\pm 0.01}$ & $\boldsymbol{35.90\pm 0.04}$ & $\boldsymbol{29.19\pm 0.03}$ & $\boldsymbol{50.59\pm 0.08}$ & $\boldsymbol{77.70\pm 0.40}$ \\
\addlinespace[2.5pt]
Shoppers & GReaT & $85.44\pm 0.16$ & $\boldsymbol{55.05\pm 0.24}$ & $79.10\pm 0.45$ & $45.29\pm 0.52$ & $25.43\pm 0.00$ & $14.10\pm 0.00$ & $17.60\pm 0.00$ & $\boldsymbol{48.87\pm 0.00}$ & $90.12\pm 0.68$ \\
 & GRPO & $85.85\pm 0.06$ & $54.38\pm 0.10$ & $76.98\pm 0.41$ & $45.69\pm 0.35$ & $25.33\pm 0.02$ & $14.65\pm 0.20$ & $17.38\pm 0.01$ & $48.47\pm 0.11$ & $90.12\pm 0.38$ \\
 & SGCO & $\boldsymbol{89.58\pm 0.11}$ & $54.13\pm 0.09$ & $\boldsymbol{98.20\pm 0.36}$ & $\boldsymbol{49.29\pm 0.69}$ & $\boldsymbol{25.31\pm 0.19}$ & $\boldsymbol{13.66\pm 0.46}$ & $\boldsymbol{16.99\pm 0.08}$ & $48.58\pm 0.18$ & $\boldsymbol{90.40\pm 0.30}$ \\
\addlinespace[2.5pt]
Magic & GReaT & $84.85\pm 0.10$ & $89.51\pm 0.54$ & $87.26\pm 0.61$ & $39.30\pm 0.47$ & $\boldsymbol{2.95\pm 0.00}$ & $0.53\pm 0.00$ & $3.29\pm 0.00$ & $49.87\pm 0.00$ & $90.08\pm 0.34$ \\
 & GRPO & $87.04\pm 0.08$ & $\boldsymbol{90.08\pm 1.21}$ & $86.99\pm 0.30$ & $43.58\pm 0.06$ & $4.05\pm 1.60$ & $0.56\pm 0.12$ & $2.18\pm 0.26$ & $49.92\pm 0.22$ & $\boldsymbol{90.90\pm 0.06}$ \\
 & SGCO & $\boldsymbol{95.28\pm 0.15}$ & $89.76\pm 1.54$ & $\boldsymbol{95.20\pm 0.50}$ & $\boldsymbol{46.94\pm 0.51}$ & $3.28\pm 0.99$ & $\boldsymbol{0.41\pm 0.05}$ & $\boldsymbol{1.48\pm 0.27}$ & $\boldsymbol{49.97\pm 0.09}$ & $90.40\pm 0.40$ \\
\addlinespace[2.5pt]
Beijing & GReaT & $90.92\pm 0.17$ & $40.61\pm 0.54$ & $96.05\pm 0.55$ & $57.71\pm 0.43$ & $35.82\pm 0.00$ & $10.11\pm 0.00$ & $23.71\pm 0.00$ & $51.64\pm 0.00$ & $66.28\pm 0.47$ \\
 & GRPO & $90.89\pm 0.02$ & $\boldsymbol{41.15\pm 3.80}$ & $98.60\pm 0.14$ & $\boldsymbol{58.84\pm 0.19}$ & $35.79\pm 0.02$ & $10.02\pm 0.04$ & $23.72\pm 0.03$ & $\boldsymbol{51.07\pm 0.08}$ & $65.89\pm 1.06$ \\
 & SGCO & $\boldsymbol{97.14\pm 0.08}$ & $40.85\pm 0.24$ & $\boldsymbol{99.24\pm 0.18}$ & $58.34\pm 0.14$ & $\boldsymbol{34.51\pm 0.03}$ & $\boldsymbol{9.87\pm 0.02}$ & $\boldsymbol{23.30\pm 0.02}$ & $51.31\pm 0.14$ & $\boldsymbol{59.20\pm 0.50}$ \\
\bottomrule
\end{tabular}}%
\endgroup
\vspace{1mm}
\par\raggedright\scriptsize $^{*}$Classification AUC$\uparrow$ for Adult, Default, Shoppers, and Magic; RMSE$\downarrow$ for Beijing. Wass. denotes Wasserstein distance. Values are shown to two decimal places, including SDs.
\end{table*}

\end{document}